%% file: neurips_2026.tex
\documentclass{article}

\usepackage[main, final]{neurips_2026}

\usepackage[utf8]{inputenc} 
\usepackage[T1]{fontenc}    
\usepackage{hyperref}       
\usepackage{url}            
\usepackage{booktabs}       
\usepackage{amsfonts}       
\usepackage{nicefrac}       
\usepackage{microtype}      
\usepackage{xcolor}         

\usepackage{float}
\usepackage{graphicx}
\usepackage{wrapfig}
\usepackage{subcaption}

\usepackage{amsmath}
\usepackage{amssymb}
\usepackage{mathtools}
\usepackage{amsthm}
\usepackage{enumitem}

\usepackage{algorithm}
\usepackage{algorithmic}
\usepackage{placeins}

\usepackage{bm}
\usepackage{multirow}
\usepackage{color, colortbl}
\definecolor{Gray}{gray}{0.9}

\newtheorem{proposition}{Proposition}
\newtheorem{assumption}{Assumption}
\title{Smart Picks in the Dark: Towards Efficient RLVR \\for Reasoning via Tracing Metacognitive Pivots}

\author{%
Guangcheng Zhu$^{12}\quad$ Shenzhi Yang$^{12}\quad$ Haobo Wang$^1$\thanks{Corresponding author.}$\quad$ Xing Zheng$^2\quad$ Yingfan MA$^2\quad$ \\ \textbf{Xuening Feng$^2\quad$ Zhongqi Chen$^2\quad$ Bowen Song$^2\quad$ Weiqiang Wang$^2\quad$ Gang Chen$^1\quad$} \\[4pt]
$^1$Zhejiang University \quad $^2$Ant Group \\[4pt]
\texttt{\{zhuguangcheng,wanghaobo\}@zju.edu.cn}
}

\begin{document}

\maketitle

\begin{abstract}
Reinforcement learning with verifiable rewards (RLVR) has greatly advanced large reasoning models (LRMs), but it requires timely training on a huge fully-annotated dataset. 
To this end, data-efficient RLVR methods have been widely studied from two perspectives: 
(i) data selection methods identify a small subset of ``golden'' samples that yield near-full-data performance, but they rely on a pre-existing pool of labeled data.
(ii) unsupervised RLVR methods train the model using its own internal supervision signals on large-scale unlabeled data, yet they exhibit suboptimal performance. 
Accordingly, we investigate the ``\textit{pick in the dark}'' setup for RLVR, which aims to select, without prior supervision, unlabeled samples that are most beneficial for training and worthy of annotation. 
Through systematic analysis, we demonstrate that smart picks hinge on a well-calibrated uncertainty estimator to enable strategic partitioning of data for adaptive training regimes.
Building on this insight, we propose \textbf{PivotTrace}, a three-way data triage framework that leverages attention dynamics to trace \textit{metacognitive pivots} during reasoning.
By precisely quantifying uncertainty through pivot density, PivotTrace achieves automated data routing to synergistically maximize both annotation and training efficiency.
Empirically, PivotTrace surpasses the fully supervised LRM with only \textbf{29.3\%} annotated samples and $\textbf{2.75}\times$ faster convergence.
Our code is available at \href{https://github.com/gczhu/PivotTrace}{https://github.com/gczhu/PivotTrace}.
\end{abstract}

\section{Introduction}
\begin{wrapfigure}[11]{r}{0.53\textwidth}
    \vspace{-12pt}
    \centering
    \includegraphics[width=0.51\textwidth]{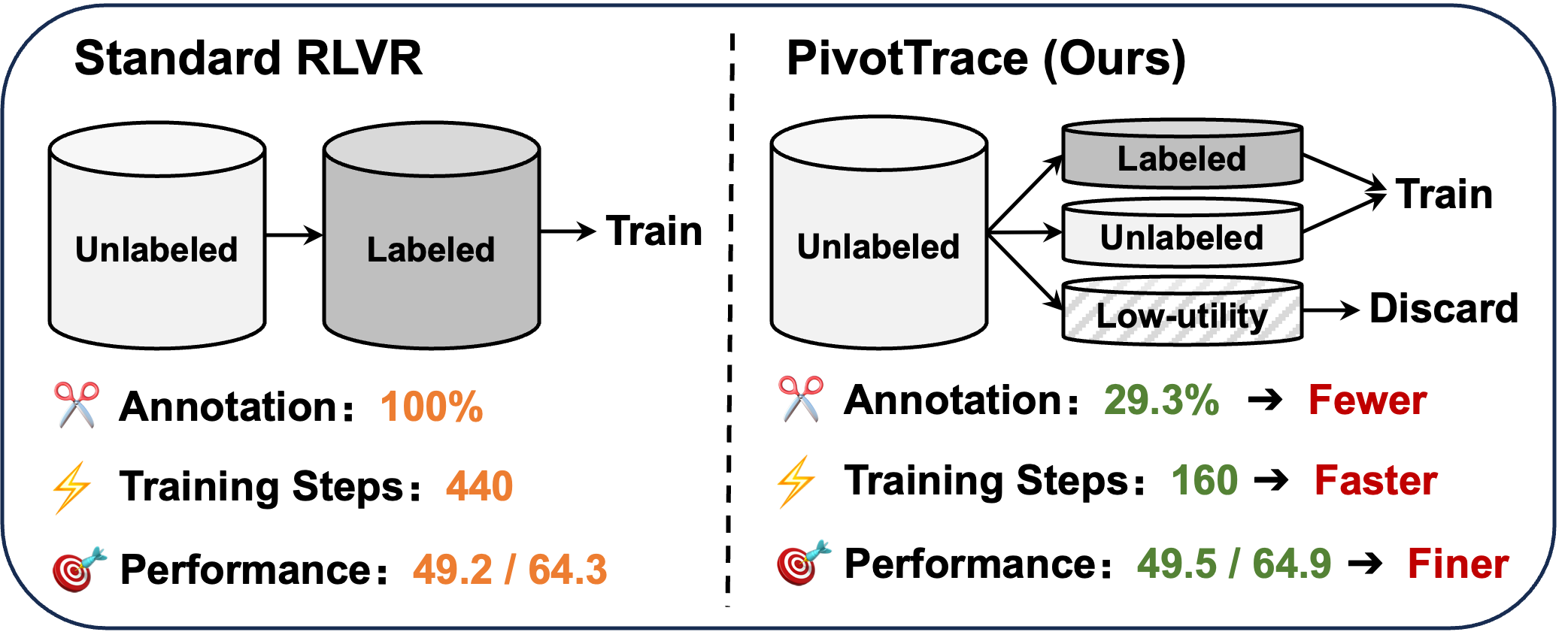}
    \vskip -0.05in
    \caption{\small Ours vs. standard RLVR: our method uses \textbf{fewer} annotations, trains \textbf{faster}, and achieves \textbf{finer} performance.}
    \label{fig:teaser}
\end{wrapfigure}
Large reasoning models (LRMs)~\citep{jaech2024openai,yang2025qwen3,guo2025deepseek} have recently demonstrated remarkable capabilities in complex problem-solving tasks. 
With the emergence of long chain-of-thought (long-CoT) reasoning~\citep{chen2025towards,yeo2025demystifying}, LRMs can decompose complex problems, perform self-reflection and correction, thereby significantly enhancing the accuracy and interpretability of their responses~\citep{renze2024self,kamoi2024can}. 
Notably, the success of these LRMs can be largely attributed to the \textit{Reinforcement Learning with Verifiable Rewards} (RLVR) paradigm~\citep{lambert2024tulu}, in which models are trained to prioritize reasoning trajectories that lead to correct answers.
By grounding rewards in verifiable outcomes, RLVR effectively elicits trustworthy reasoning~\citep{wen2025reinforcement,zhang2025survey}.

Despite its success, RLVR faces two major efficiency bottlenecks: high \textit{training} and \textit{annotation} costs.
First, RLVR training involves repeated trajectory sampling, reward computation, and policy updates over large-scale data, making it computationally expensive~\citep{zhu2025data,tang2025towards}.
Second, RLVR critically hinges on obtaining ground-truth answers for all training instances, which is notoriously time-consuming and labor-intensive. Even worse, it can be prohibitively expensive in practical scenarios requiring substantial expertise, such as medicine and finance~\citep{yu2025rlpr}.

In response, recent work focuses on improving the efficiency of RLVR.
To reduce training cost, methods like \citet{li2025limr} and \citet{yu2025dapo} filter out trivial or overly difficult samples, retaining only informative ones.
Yet, they require the entire dataset to be annotated beforehand, which defeats the purpose of reducing annotation cost.
In practice, the selection of training samples should be made before any ground-truth annotation is available, a setting we term ``\textit{pick in the dark}''.
To mitigate annotation costs, recent efforts~\citep{zuo2025ttrl,zhao2025learning,li2025confidence} explore unsupervised RLVR that eliminates external supervision entirely, deriving rewards from the model’s internal confidence signals.
However, relying solely on self-rewards can cause model collapse within a few training steps~\citep{zhang2025co}.
TraPO~\citep{yang2025trapo} alleviates this by randomly annotating a subset of samples and mixing them with unlabeled data for semi-supervised training.
Despite its promise, random selection wastes valuable annotation resources. 
Our experiments show that over $67\%$ of annotated samples already yield consistent and correct responses for reliable unsupervised training, rendering their annotation unnecessary.
We aim to jointly improve training and annotation efficiency in RLVR, which leads us to a key question:
\textit{How can we pick, in the dark, which samples are valuable to train on and which of those truly warrant annotation?}

To answer this, we focus on two crucial factors: \textit{learning utility} and \textit{internal reliability}. 
Low-utility samples can be justifiably excluded from training. 
Among the remaining high-utility samples, only those with unreliable self-supervision warrant human annotation; the rest can be used directly for unsupervised learning.
This data triage strategy, combined with the semi-supervised RLVR paradigm, paves the way for dual efficiency. 
To effectively quantify these two factors, we theoretically analyze the RLVR policy optimization process and show that both are determined by the model’s expected correctness on a given question: lower expected correctness implies higher learning utility but less reliable self-rewards. 
Unfortunately, it cannot be directly computed in our ``pick in the dark'' setting without access to ground-truth answers, necessitating a label-free proxy for expected correctness.
Essentially, we require a well-calibrated uncertainty estimator to bridge this gap, along with principled thresholds to facilitate precise data triage.

In this paper, we propose \textbf{PivotTrace} (shown in Figure~\ref{fig:framework}), a three-way data triage framework that traces metacognitive pivots during reasoning to quantify model uncertainty and guide adaptive data routing for efficient RLVR. 
Grounded in cognitive science~\citep{flavell1979metacognition}, we find that LRMs perform metacognitive monitoring and regulation like humans to detect and rectify reasoning failures. 
This process is driven by \textit{metacognitive pivots}, i.e., critical transition tokens where the model retracts prior inferences and initiates alternative reasoning paths. 
Our empirical analysis reveals that an excessive occurrence of these pivots reflects reasoning instability and correlates with erroneous outcomes. 
More importantly, these pivots consistently attract intense long-range attention, manifesting as sharp attention peaks.
Building on this, PivotTrace bridges individual pivot detection with system-wide data curation through two core mechanisms. 
First, it utilizes peak detection on attention dynamics to derive pivot counts as a robust uncertainty proxy. 
Second, it introduces an automated threshold calibration module, which dynamically determines optimal partitioning boundaries via minimal few-shot probing. 
This enables PivotTrace to adaptively route data into distinct training pipelines, thereby enhancing annotation and training efficiency in tandem.
Empirically, PivotTrace outperforms the strongest baseline by $\textbf{+1.6\%}$ in-domain (ID) and $\textbf{+2.4\%}$ out-of-domain (OOD) in average accuracy. 
Remarkably, it even \textbf{surpasses fully supervised training on the entire dataset} with only $\textbf{29.3\%}$ labeled samples and $\textbf{2.75}\times$ faster training, showing its effectiveness for efficient RLVR.

\section{Related Work}
\textbf{Active Learning} (AL) aims to select a small subset of data for annotation to achieve competitive performance over supervised learning on fully labeled data~\citep{cohn1994improving,roy2001toward,hacohen2022active,xiao2023freeal}. 
AL can be broadly categorized into two paradigms based on the training strategy: 
(i) \textit{Supervised Active Learning} (SAL), which trains solely on the selected labeled data in a purely supervised manner~\citep{sener2017active,yoo2019learning,xie2023towards}; and 
(ii) \textit{Semi-supervised Active Learning} (SSAL), which trains on the selected labeled data and the remaining unlabeled data in a semi-supervised manner~\citep{leng2013combining,wang2022unsupervised,rangnekar2023semantic}. 
Both paradigms can be applied to the RLVR task. 
However, \citet{yang2025trapo} show that, under limited annotation budgets, exploiting unlabeled data greatly boosts RLVR performance over purely supervised training. 
Thus, we adopt SSAL to maximize the utility of all available data.

\textbf{Training-efficient RLVR} accelerates learning by filtering out less informative training samples, focusing optimization on questions that yield high learning utility~\citep{li2025limr,yu2025dapo,wen2025light}. 
A primary criterion for sample selection in RLVR is question difficulty, with broad agreement that relatively difficult questions yield strongest learning signals~\citep{wang2025sota,zeng2025cures,bae2025online,li2025truth}. 
However, existing methods typically estimate difficulty via empirical answer accuracy, which faces two fundamental limitations: 
(i) it requires ground-truth answers, rendering it inapplicable in our more realistic ``pick in the dark'' setting; and 
(ii) it relies on repeated rollouts per question, incurring prohibitive computational overhead. 
Given this, we propose to estimate model-dependent question difficulty from a single rollout without any ground-truth supervision, by accurately quantifying the model’s uncertainty in its own reasoning.

\textbf{Annotation-efficient RLVR} seeks to mitigate reliance on ground-truth answers. Recent unsupervised methods~\citep{wei2025unsupervised,zhang2025right,zhao2025absolute} leverage internal model signals, such as majority voting~\citep{zuo2025ttrl}, entropy~\citep{agarwal2025unreasonable}, self-certainty~\citep{zhao2025learning}, and hidden states~\citep{zhang2025consistent}, to supervise learning. 
However, these methods often reinforce incorrect reasoning and suffer from model collapse~\citep{zhang2025no,zhang2025co}. 
To address this, TraPO~\citep{yang2025trapo} adopts semi-supervised RLVR with random annotation, which wastes valuable annotation resources and yields limited efficiency. 
We argue that actively selecting the most uncertain samples for annotation is essential to maximize the utility of scarce human labels.

\section{Preliminary}
\label{sec:preliminary}
\paragraph{RLVR Paradigm.}
RLVR is a reinforcement learning paradigm in which a rule-based verifier assigns a binary reward based on response correctness.
Formally, given a dataset $\mathcal{D}$ of question-answer pairs $(q,a)$, the policy $\pi_\theta$ generates a response $y \sim \pi_\theta(\cdot \mid q)$ for each question $q$. Let $\hat{a}$ denote the answer extracted from $y$. The reward is defined as $R(y,a)=\mathbb{I}[\hat{a}=a]$, 
where $\mathbb{I}[\cdot]$ is the indicator function.
Within this paradigm, various algorithms~\citep{schulman2017proximal,shao2024deepseekmath,hu2025reinforce++} have been proposed.
We adopt the widely used \textit{Group Relative Policy Optimization} (GRPO)~\citep{shao2024deepseekmath} as our base algorithm.
GRPO eliminates the value model and computes advantages from the rewards of multiple responses to the question. 
Formally, for each question $q$, we sample $G$ responses $\{y^{(i)}\}_{i=1}^G$ from the old policy $\pi_{\theta_{\text{old}}}$ and compute their rewards $R(y^{(i)}, a)$. The group-normalized advantage $\hat{A}_i$ is given by:

\vskip -0.1in
\begin{equation}
    \hat{A}_i = \frac{R(y^{(i)},a) - \text{mean}(\{R(y^{(i)},a)\}_{i=1}^G)}{\text{std}(\{R(y^{(i)}, a)\}_{i=1}^G)}.
\label{eq:advantage}
\end{equation}
Then, the GRPO objective is defined as:
\begin{equation}
\scalebox{0.94}{$
\begin{aligned}
    \mathcal{J}_{\text{GRPO}}
    (\theta; \mathcal{D})
    = \mathbb{E}[q \!\sim\! \mathcal{D}, \{y^{(i)}\}_{i=1}^G \!\sim\! \pi_{\theta_\text{old}}(\cdot \mid q)] \,\, \frac{1}{G} \! \sum_{i=1}^G \! \frac{1}{|y^{(i)}|} \! \sum_{t=1}^{|y^{(i)}|} \! \text{CLIP}\bigl(\gamma_{i,t}(\theta), \hat{A}_i, \epsilon\bigr) \! - \! \beta \! \cdot \! \mathbb{D}_{\text{KL}}[\pi_\theta \,\|\, \pi_{\text{ref}}]
\end{aligned}
$}
\label{eq:GRPO_loss}
\end{equation}
where ${\gamma_{i,t}(\theta) \!=\! {\pi_\theta(y^{(i)}_t | q, y^{(i)}_{<t})}/{\pi_{\theta_{\text{old}}}(y^{(i)}_{t} | q, y^{(i)}_{<t})}}$ is the importance weight, $\text{CLIP}(\gamma, A, \epsilon) = \min[\gamma \!\cdot\! A, \text{clip}(\gamma; 1\!-\!\epsilon, 1\!+\!\epsilon) \!\cdot\! A]$ is the clipped surrogate objective, and $\mathbb{D}_{\text{KL}}$ 
denotes the KL divergence.

\paragraph{Semi-supervised RLVR.} 
To balance annotation cost and training effectiveness, TraPO~\citep{yang2025trapo} proposes a semi-supervised RLVR paradigm that combines a small annotated set $\mathcal{D}_a = \{(q, a)\}$ with an unlabeled set $\mathcal{D}_u = \{q\}$.
It employs a hybrid reward function as follows: 
\begin{equation}
\label{eq:semi_reward}
    R_{\text{semi}}(y^{(i)}) =
    \begin{cases}
        R(y^{(i)}, a), & \text{if } (q, a) \in \mathcal{D}_a, \\
        R_u(y^{(i)}),    & \text{if } q \in \mathcal{D}_u.
    \end{cases}
\end{equation}
Here, $R_u(y^{(i)}) \!=\! \mathbb{I}[\hat{a}^{(i)} \!=\! a^{\!*}]$, with $a^{\!*} \!=\! \mathrm{MAJ}(\hat{a}^{(1)} \!,\! \dots \!,\! \hat{a}^{(G)})$ denoting the majority answer over $G$ responses.
The resulting rewards are used to compute the advantage estimates in Eq.~\eqref{eq:advantage}, which in turn formulate the GRPO loss in Eq. \eqref{eq:GRPO_loss}.

\begin{figure*}[t]
    \centering
    \includegraphics[width=\textwidth]{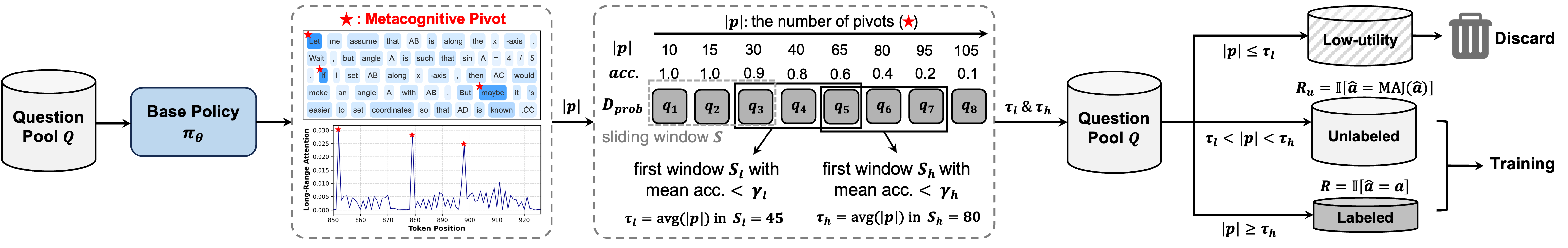}
    \vskip -0.01in
    \caption{\small Overview of the PivotTrace framework. It quantifies reasoning uncertainty by detecting attention peaks (pivot counts) from model-generated CoTs. By mapping pivot counts to accuracy via sliding windows on a small probing set, the framework automatically calibrates thresholds for three-way data triage, smartly selecting samples for labeling or filtering to enhance semi-supervised RLVR training and maximize data efficiency.}
    \label{fig:framework}
\end{figure*}

\section{How to Pick in the Dark for Dual Efficiency}
\label{sec:how_pick}
\subsection{Dual Efficiency: Objectives and Trade-offs}
We begin by formally defining dual efficiency objectives:
(i) \textit{Annotation efficiency} refers to maximizing the expected marginal utility per annotation, i.e., annotating only those samples for which self-supervision is unreliable. 
(ii) \textit{Training efficiency} refers to minimizing the number of training samples while maintaining model performance, i.e., discarding samples with low utility for learning.
Our primary goal is to jointly improve both efficiencies via active data selection before training.

\textbf{Remark.}
\textit{
Annotation and training efficiency are inherently in tension: maximizing annotation efficiency focuses labeling on the model’s knowledge frontier, yet training solely on these samples yields unstable optimization and suboptimal performance. 
Instead, directly adding unlabeled data maintains performance but degrades training efficiency due to low-utility samples.
This trade-off is empirically visualized in Figure~\ref{fig:trade_off}. See Appendix~\ref{app:train_label} for more detailed discussions.
}

To resolve this trade-off, we propose a \textit{three-way data triage} strategy within the semi-supervised RLVR paradigm, partitioning the question pool $\mathcal{Q}$ into three distinct subsets:

\begin{itemize}[left=0pt]
    \item \textbf{Annotation set} $\mathcal{D}_a$: high-utility samples with unreliable self-rewards, which are prioritized to be annotated for supervised training;
    \item \textbf{Unlabeled set} $\mathcal{D}_u$: high-utility samples with reliable self-rewards, used for unsupervised training;
    \item \textbf{Discard set} $\mathcal{D}_d$: low-utility samples, not used in training.
\end{itemize}

\subsection{Principles of Data Triage}
This triage is guided by two critical factors: \textit{learning utility} and \textit{internal reliability}, both of which are inherently tied to the model’s expected correctness $\mu_{\theta}(q)=\mathbb{E}_{y \sim \pi_{\theta}(\cdot \mid q)}[R(y, a)]$ on the sample $q$.

\textbf{Learning utility.} \,
From an optimization perspective, samples already mastered by the model (high $\mu_{\theta}(q)$) offer diminishing returns for further training. 
Conversely, samples where the model lacks proficiency (low $\mu_{\theta}(q)$) offer the greatest optimization potential.
We theoretically establish this relationship in Appendix~\ref{app:proof_prop} within the framework of trust-region policy optimization, confirming that learning utility is indeed \textbf{strictly decreasing} in $\mu_{\theta}(q)$.

\textbf{Internal reliability.} \,
On the other hand, from the standpoint of self-supervision, a lower $\mu_{\theta}(q)$ naturally implies that the model's self-generated rewards are less trustworthy. For such samples, relying solely on internal feedback risks reinforcing erroneous patterns, which can eventually result in model collapse. This renders external supervision essential to anchor the learning process.

This relationship motivates us to use $\mu_{\theta}(q)$ for our three-way data triage:
low values correspond to high-utility but unreliable samples (for annotation), 
moderate values to high-utility and reliable ones (for unsupervised use), 
and high values to low-utility samples (to discard).

Ideally, with access to ground-truth answers, $\mu_{\theta}(q)$ can be unbiasedly estimated by the proportion of correct responses in multiple rollouts from $\pi_{\theta}$.
Yet, in our ``pick in the dark'' setting, the correctness of model responses is unverifiable due to the absence of ground-truth answers, necessitating a faithful, label-free proxy for $\mu_{\theta}(q)$. 
In other words, we aim to develop a proxy that can accurately quantify the model's intrinsic reasoning uncertainty as a surrogate for correctness.

\begin{figure*}[t]
    \centering
    \begin{subfigure}[t]{0.352\textwidth}
        \centering
        \includegraphics[width=\textwidth]{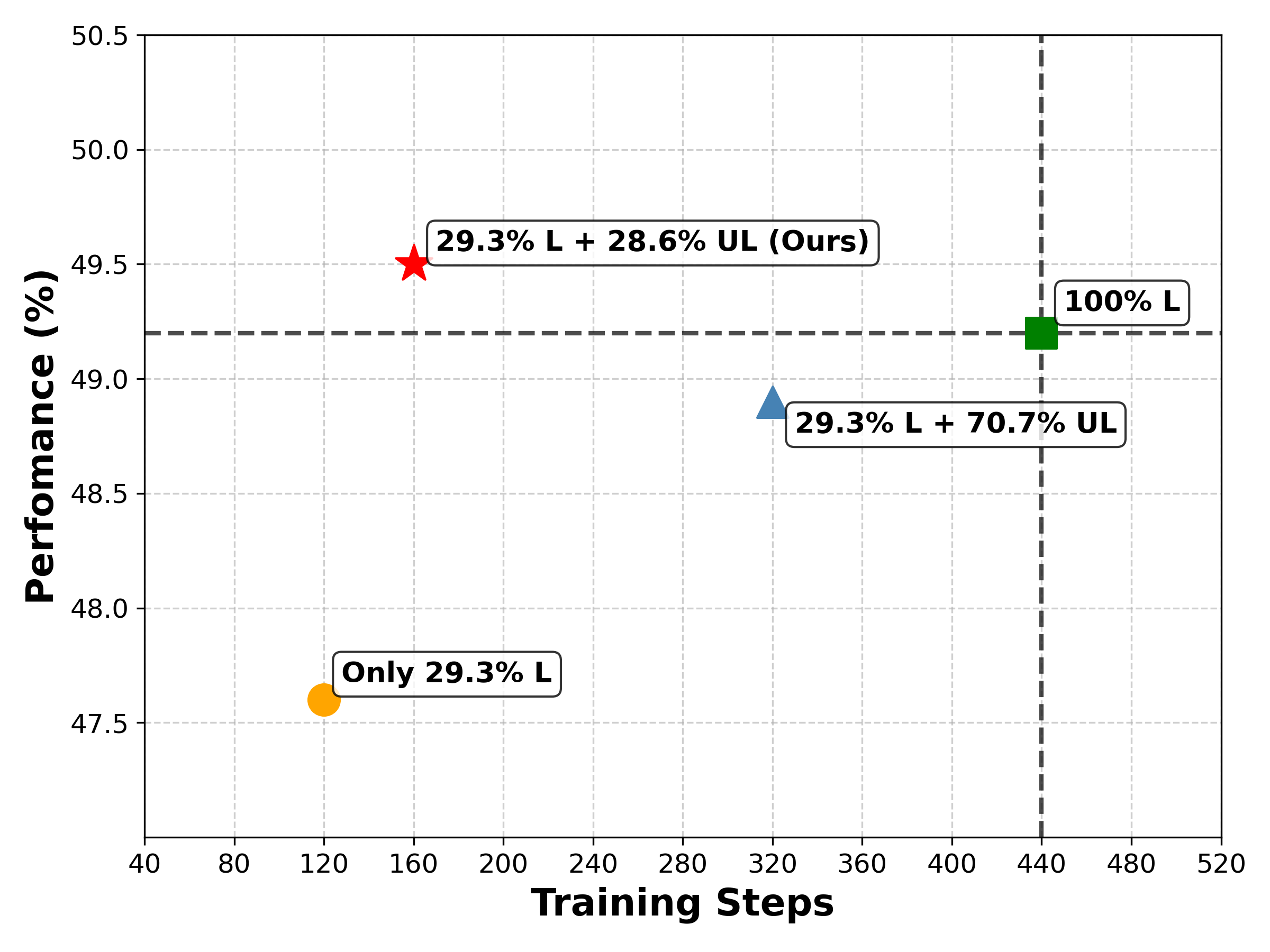}
        \caption{\small Dual efficiency trade-off}
        \label{fig:trade_off}
    \end{subfigure}
    \hfill
    \begin{subfigure}[t]{0.632\textwidth}
        \centering
        \includegraphics[width=\textwidth]{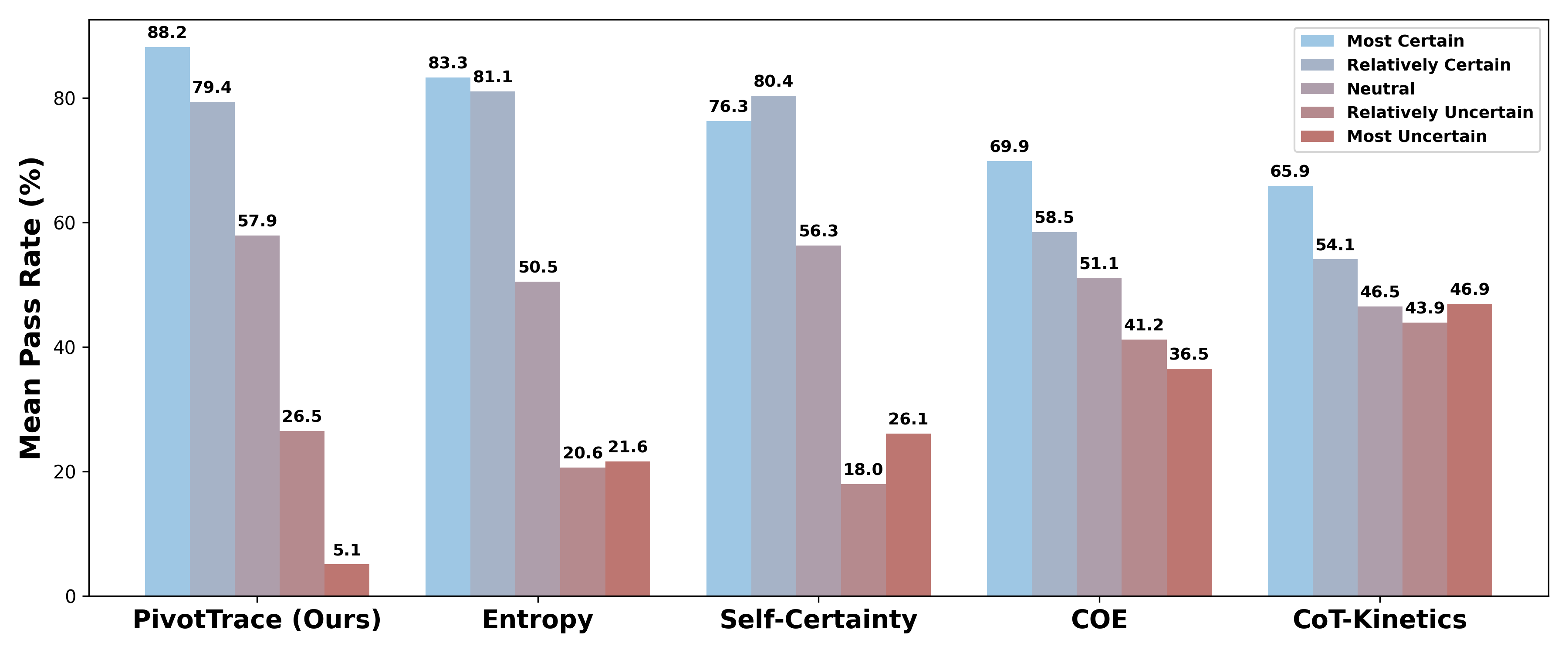}
        \caption{\small Comparison of uncertainty metrics}
        \label{fig:compare_baselines}
    \end{subfigure}
    
    \caption{\small (a) Using only labeled data or naively adding all unlabeled samples faces a trade-off between training speed and performance, whereas PivotTrace achieves dual efficiency via active selection.
    (b) Mean pass rate across uncertainty quantiles for different metrics; a larger gap indicates a superior proxy for expected correctness.}
    \label{fig:main}
\end{figure*}

\section{PivotTrace: Tracing Reasoning Uncertainty for Active Data Selection}

\subsection{Limitations of Standard Uncertainty Estimators}
\label{sec:standard_limit}
We first evaluate standard uncertainty estimation methods, as detailed in the baselines of Section~\ref{sec:setup}. 
For each method, an uncertainty score is computed for every sample, and samples are ranked into five quantiles based on that score, ranging from \textit{Most Certain} to \textit{Most Uncertain}.
We then compute the mean pass rate within each quantile to verify its expected negative correlation with uncertainty, a hallmark of well-calibrated uncertainty signals.
As illustrated in Figure~\ref{fig:compare_baselines}, representation-based methods (CoE and CoT-Kinetics) exhibit limited certainty-uncertainty separation, with pass rates of about $60\%$ and $40\%$, respectively.
In contrast, probability-based methods (entropy and self-certainty) exhibit better separation, yet still fail to maintain monotonicity at finer granularity: the \textit{Most Uncertain} quantile achieves a higher pass rate than the \textit{Relatively Uncertain} one.
These results suggest that existing uncertainty proxies unreliably reflect expected correctness.

\subsection{Metacognitive Pivots in Uncertain Reasoning}
\label{sec:nature_reason}
To address this, we delve into the behavioral signatures of model reasoning exhibited in incorrect responses.
Empirically, we find that the model displays unstable reasoning patterns, i.e., frequently retracting earlier claims and redirecting its logic, which resembles a metacognitive process of self-monitoring and regulation.
Crucially, the frequency of these events correlates strongly with lower accuracy; as shown in Figure~\ref{fig:compare_baselines}, we partition samples by this frequency and reveal a distinct accuracy gap across uncertainty quantiles.
Motivated by this insight, we seek an intrinsic signal to automatically capture these events.
We hypothesize that attention patterns, which reflect how the model dynamically weights past reasoning steps, can serve as such a signal. 
In particular, when the model revises its logic, it creates a pivot point that subsequent tokens frequently revisit, inducing sustained, long-range attention.
Our visualization in Figure~\ref{fig:attn_vis} confirms this hypothesis, revealing that tokens marking these pivot points indeed attract strong, sustained attention.
Thus, we term these tokens \emph{metacognitive pivots} and use their count as a scalable, automatic proxy for expected correctness.

\subsection{Attention-Driven Pivot Detection}
\label{sec:pivot_detect}
Based on the above observations, we propose to identify metacognitive pivots through peak detection on long-range attention dynamics.
Given a question $q$ and its model-generated CoT response $y$ of length $T$, we extract attention maps $\{\mathbf{A}^{(h)} \!\in\! \mathbb{R}^{T \times T} \}_{h=1}^H$ from all $H$ attention heads, where $\mathbf{A}^{(h)}_{t,s}$ denotes the attention weight of head $h$ that token $t$ assigns to token $s$.
For each token $y_t$, we define its long-range attention $\alpha^{(h)}_t$ under head $h$ as the average attention from a future window $\mathcal{W}_t$:
\begin{equation}
\label{eq:long_attn}
\begin{aligned}
    \alpha^{(h)}_t = \frac{1}{|\mathcal{W}_t|} \sum_{s \in \mathcal{W}_t} \mathbf{A}^{(h)}_{s,t}, \,\,
    \mathcal{W}_t = \{ s \mid t \!+\! d_{\min} \!\leq\! s \!\leq\! \min(t \!+\! d_{\max}, T) \}.
\end{aligned}
\end{equation}
Here, $d_{\min}$ and $d_{\max}$ denote the minimum and maximum distances from token $y_t$ to future tokens included in the window, respectively.
We exclude nearby tokens ($s \!<\! t \!+\! d_{\min}$) to avoid conflating long-range signals with local syntactic or lexical patterns, and cap the window at $t \!+\! d_{\max}$ to ensure consistent attention normalization across positions, as early tokens would otherwise suffer from diluted attention over an excessively long future horizon.

Notably, we find that attention heads exhibit heterogeneity: some consistently yield near-zero $\alpha^{(h)}_t$, i.e., focusing on local context rather than global reasoning (Figure~\ref{fig:attn_head}). 
Including such heads dilutes the long-range attention signal, which can obscure genuine metacognitive pivots.
Given this, we measure the long-range capability of each head $h$ by averaging $\alpha^{(h)}_t$ over a small set of CoT trajectories:
\begin{equation}
\label{eq:head_attn}
    \bar{\alpha}^{(h)} = \frac{1}{M} \sum_{i=1}^M \left( \frac{1}{T_i}\sum_{t=1}^{T_i} \alpha^{(h)}_t \right),
\end{equation}
where $M$ is the number of CoT trajectories and $T_i$ is the length of the $i$-th trajectory. Let $\mathcal{T}_k$ denote the top-$k$ heads with the highest $\bar{\alpha}^{(h)}$. The final long-range attention for $y_t$ is then given by:
\begin{equation}
\label{eq:final_attn}
    \alpha_t = \frac{1}{k} \sum_{h \in \mathcal{T}_k} \alpha^{(h)}_t.
\end{equation}
The resulting sequence $\{\alpha_t\}_{t=1}^T$ captures the long-range attention dynamics of $y$. 
We detect metacognitive pivots as prominent local maxima in this sequence. 
Specifically, let $\bm{p} \!=\! \{p_i\}$ denote the peak positions, where each $p_i$ satisfies:
\begin{equation}
\label{eq:peak_condition}
    \alpha_{p_i} \geq \zeta, \,\, 
    \mathcal{P}(p_i) \geq \psi, \,\, 
    |p_i - p_j| \geq \Delta \, (\forall i \neq j).
\end{equation}
Here, $\mathcal{P}(p_i)$ denotes the prominence at $p_i$, \textit{i.e.}, the vertical distance from $\alpha_{p_i}$ to the lowest point between two higher neighbors. The thresholds $\zeta$, $\psi$, and $\Delta$ control minimum peak height, prominence, and inter-peak distance, respectively.
Thus, $\bm{p}$ gives the locations of pivot tokens. 

\subsection{Adaptive Three-Way Data Triage}
\label{sec:active_select}
For each question $q$, we treat the pivot count $|\bm{p}|$ in its model-generated CoT response $y$ as an uncertainty proxy: higher $|\bm{p}|$ correlates with lower expected correctness $\mu_{\theta}(q)$. 
As suggested in Section~\ref{sec:how_pick}, the question pool $\mathcal{Q}$ is sorted by increasing $|\bm{p}|$ (decreasing $\mu_{\theta}(q)$) and partitioned into: 
(1) annotation set $\mathcal{D}_a = \{ q \mid |\bm{p}| \!\geq\! \tau_h \}$;
(2) unlabeled set $\mathcal{D}_u = \{ q \mid \tau_l \!<\! |\bm{p}| \!<\! \tau_h \}$;
(3) discard set $\mathcal{D}_d = \{ q \mid |\bm{p}| \!\leq\! \tau_l \}$. 
We define the annotation rate as $\rho_a \!=\! |\mathcal{D}_a|/|\mathcal{Q}|$ and the training data retention rate as $\rho_t \!=\! |\mathcal{D}_a \!\cup\! \mathcal{D}_u|/|\mathcal{Q}|$, where lower values signify improved annotation and training efficiency.

Crucially, effective data triage hinges on the precise setting of $\tau_l$ and $\tau_h$. Since $|\bm{p}|$ varies significantly across models, manual tuning is unscalable. To this end, we design an automated calibration procedure.
First, we uniformly sample $N$ questions from the sorted pool $\mathcal{Q}$ to form a probing set $\mathcal{P} \!=\! \{q_j\}_{j=1}^N$. 
For each $q_j \in \mathcal{P}$, we annotate its ground-truth answer $a_j$, generate $G$ model responses, and compute the empirical accuracy $\hat{\mu}_j$. 
A sliding window of size $K$ is then applied over $\mathcal{P}$ to compute the average empirical accuracy within each window.
Let the $i$-th window be $\mathcal{S}_i \!=\! \{ q_j \!\in\! \mathcal{P} \!\mid\! j \!\in\! [i, i\!+\!K\!-\!1]\}$ for $i = 1, \dots, N-K+1$, its average accuracy is defined as $\bar{\mu}_i = \frac{1}{K} \sum_{q_j \in \mathcal{S}_i} \hat{\mu}_j$.
The thresholds $\tau_l$ and $\tau_h$ are then set to the average pivot count in the first windows where $\bar{\mu}_i$ drops below fixed accuracy levels $\gamma_l$ and $\gamma_h$, respectively:
\begin{align}
    \tau_l &= \frac{1}{K} \!\sum_{q_j \in \mathcal{S}_{i_l}}\!\! |\bm{p}_j|,  \text{where } i_l \!=\! \min \{ i \mid \bar{\mu}_i < \gamma_l \}, \\
    \tau_h &= \frac{1}{K} \!\sum_{q_j \in \mathcal{S}_{i_h}}\!\! |\bm{p}_j|,  \text{where } i_h \!=\! \min \{ i \mid \bar{\mu}_i < \gamma_h \},
\end{align}
Using dynamically computed thresholds $\tau_l$ and $\tau_h$, we split the data into $\mathcal{D}_a$, $\mathcal{D}_u$ and $\mathcal{D}_d$. 
Finally, we discard $\mathcal{D}_d$, annotate $\mathcal{D}_a$, and perform semi-supervised RLVR on $\mathcal{D}_a \cup \mathcal{D}_u$, as described in Section~\ref{sec:preliminary}.

\textbf{Remark.}
\textit{
By combining attention-driven pivot detection with automated threshold calibration, PivotTrace provides a principled, label-free framework to pre-identify both low-utility samples for discarding and high-priority samples for annotation. 
This realization of ``smart picks in the dark'' eliminates the need for expensive ground-truth verification during data selection, maximizing dual efficiency by ensuring that every training step and every annotation credit is spent where it provides the highest marginal gain for model improvement.
}

\begin{table*}[!t]
\centering
\caption{\small In-domain (ID) and out-of-domain (OOD) performance using Qwen3-4B-Base. Results are reported under two settings: (i) each method annotates $29.3\%$ samples from the full dataset and trains on all data; (ii) each method annotates $29.3\%$ samples and trains on the selected $57.9\%$ high-utility samples. \textbf{Bold} denotes the best results and $\dagger$ denotes methods requiring multiple stochastic inferences.}
\label{tab:main_results}
\setlength{\tabcolsep}{3pt}  
\renewcommand{\arraystretch}{1.1} 
\resizebox{\textwidth}{!}{%
\begin{tabular}{lcccccc|cccc}
\toprule
\multirow{2}{*}{\textbf{Methods}} & \multicolumn{6}{c}{\textbf{In-Domain Performance}} & \multicolumn{4}{c}{\textbf{Out-of-Domain Performance}} \\
\cmidrule(lr){2-7} \cmidrule(lr){8-11}
 & \textbf{AIME 24/25} & \textbf{AMC} & \textbf{MATH-500} & \textbf{Minerva} & \textbf{Olympiad} & \textbf{Avg.} & \textbf{ARC-c} & \textbf{GPQA}$^{*}$ & \textbf{MMLU-Pro} & \textbf{Avg.} \\
\midrule
\multicolumn{11}{c}{Semi-Supervised Training on the Full Dataset ($\rho_t = 100\%$) with $\rho_a \approx 29.3\%$ Annotated} \\
\midrule
Random & 21.6/23.2 & 57.6 & 85.3 & 42.1 & 47.7 & 46.3 & 87.6 & 32.1 & 61.5 & 60.4 \\
Consistency$^\dagger$ & 24.7/23.9 & 57.4 & 85.4 & 43.0 & 47.6 & 47.0 & 86.5 & 31.1 & 61.5 & 59.7 \\
CoE & 24.9/21.8 & 59.0 & 86.6 & 43.5 & 48.1 & 47.3 & 92.6 & 31.7 & 61.4 & 61.9 \\
CoT-Kinetics & 23.1/23.3 & 55.9 & 86.5 & 42.8 & 47.4 & 46.5 & 86.9 & 35.5 & 62.5 & 61.6 \\
Entropy & 24.9/23.8 & 59.5 & 86.1 & \textbf{43.6} & 47.9 & 47.6 & 91.0 & 34.6 & 62.1 & 62.6 \\
Self-Certainty & 24.3/22.0 & 59.1 & 86.5 & 43.5 & 48.8 & 47.4 & 92.6 & 31.8 & 62.2 & 62.2 \\
\rowcolor{Gray}\textbf{PivotTrace (ours)} & \textbf{25.7}/\textbf{24.1} & \textbf{63.1} & \textbf{87.5} & 43.2 & \textbf{49.9} & \textbf{48.9} & \textbf{92.8} & \textbf{36.6} & \textbf{62.7} & \textbf{64.0} \\
\midrule
\multicolumn{11}{c}{Semi-Supervised Training on the Selected Subset (\textit{$\rho_t\approx57.9\%$}) with \textit{$\rho_a\approx29.3\%$} Annotated} \\
\midrule
Random & 23.1/23.4 & 56.6 & 85.2 & 42.7 & 47.5 & 46.4 & 88.2 & 33.0 & 62.4 & 61.2 \\
Consistency$^\dagger$ & 25.9/22.1 & 59.9 & 86.2 & 43.5 & 50.0 & 47.9 & 86.1 & 33.2 & 62.7 & 60.7 \\
CoE & 26.5/23.9 & 57.0 & 87.0 & 43.3 & 49.7 & 47.9 & 92.8 & 32.8 & 61.6 & 62.4 \\
CoT-Kinetics & 21.6/24.8 & 57.3 & 86.0 & 43.7 & 47.6 & 46.8 & 88.7 & 33.7 & 61.8 & 61.4 \\
Entropy & 26.1/22.1 & 59.2 & 86.1 & 43.7 & 48.4 & 47.6 & 83.3 & 37.8 & 62.5 & 61.2 \\
Self-Certainty & 24.2/23.6 & 58.8 & 86.3 & 43.8 & 48.5 & 47.5 & 90.1 & 35.7 & 61.7 & 62.5 \\
\rowcolor{Gray}\textbf{PivotTrace (ours)} & \textbf{27.2}/\textbf{25.2} & \textbf{62.4} & \textbf{87.3} & \textbf{44.6} & \textbf{50.1} & \textbf{49.5} & \textbf{93.0} & \textbf{38.3} & \textbf{63.5} & \textbf{64.9} \\
\midrule
\textcolor{gray}{{Fully Supervised}} & \textcolor{gray}{26.6/23.8} & \textcolor{gray}{60.7} & \textcolor{gray}{87.4} & \textcolor{gray}{43.8} & \textcolor{gray}{52.9} & \textcolor{gray}{49.2} & \textcolor{gray}{93.1} & \textcolor{gray}{36.5} & \textcolor{gray}{63.4} & \textcolor{gray}{64.3} \\
\bottomrule
\end{tabular}%
}
\end{table*}

\section{Experiments}
In this section, we present the main results and a detailed analysis showing that our method effectively quantifies the model’s reasoning uncertainty, thereby achieving higher annotation and training efficiency. More experimental details and results are provided in Appendix~\ref{app:exp_setup} and \ref{app:exp_results}, respectively.
\subsection{Setup}
\label{sec:setup}
\paragraph{Implementation Details.}
For data selection, we set the future window bounds in Eq.~\eqref{eq:long_attn} to $d_{\text{min}} \!=\! 20$ and $d_{\text{max}} \!=\! 100$. 
We use $M \!=\! 8$ trajectories to compute Eq.~\eqref{eq:head_attn} and select the top $k \!=\! 20\%$ of attention heads for Eq.~\eqref{eq:final_attn}. 
The thresholds in Eq.~\eqref{eq:peak_condition} are specified as follows: $\zeta$ to the $95$th percentile of the attention signal, $\psi$ to $5\%$ of its dynamic range, and $\Delta \!=\! 10$. 
The size of the probing set is set to $N \!=\! 100$ and that of the sliding window to $K \!=\! 20$.
The accuracy thresholds are fixed at $\gamma_l= 0.7$ and $\gamma_h= 0.3$.
For model training, we build upon the \textit{verl} framework~\citep{sheng2025hybridflow}, use a training batch size of $128$, a micro-batch size of $32$, and a learning rate of $1e^{-6}$.
Qwen3-4B-Base serves as the default model, with more evaluations using different models in Appendix~\ref{app:more_model}.
All models are trained on DAPO-Math-14k~\citep{yu2025dapo} using $8 \times$ A100 GPUs.
See Appendix~\ref{app:exp_setup} for more details.

\paragraph{Evaluation.}
Following prior work~\citep{yan2025learning,yang2025trapo}, we evaluate on benchmarks spanning mathematical and general reasoning. 
For mathematical reasoning, we test on six competition-level datasets: AIME 2024, AIME 2025, MATH-500~\citep{hendrycks2021measuring}, Minerva~\citep{lewkowycz2022solving}, AMC~\citep{li2024numinamath}, and OlympiadBench~\citep{he2024olympiadbench}. 
To evaluate out-of-distribution generalization, we further include three general reasoning benchmarks: ARC-c~\citep{clark2018think}, GPQA-diamond~\citep{rein2024gpqa} (denoted GPQA$^*$), and MMLU-Pro~\citep{wang2024mmlu}. 
Due to varying test-set sizes, we report $\text{avg}@32$ for AIME 2024/2025 and AMC, $\text{pass}@1$ for MMLU-Pro, and $\text{avg}@4$ for all other datasets.  
All evaluations use temperature $0.6$ and top-$p$ $1.0$.

\paragraph{Baselines.}
We consider six label-free selection strategies for fair comparison: 
(1) \textbf{Random}.
(2) \textbf{Consistency}~\citep{zuo2025ttrl}: lower agreement across $G$ sampled answers implies higher uncertainty.
(3) \textbf{Entropy}~\citep{huang2023look}: higher average entropy over output tokens implies higher uncertainty.
(4) \textbf{Self-Certainty}~\citep{kang2025scalable}: lower average KL divergence between output token distributions and the uniform distribution implies higher uncertainty.
(5) \textbf{CoE}~\citep{wang2024latent}: smaller magnitude variation and larger angular variation of hidden states across layers imply higher uncertainty.
(6) \textbf{CoT-Kinetics}~\citep{bi2025cot}: lower semantic momentum and curvature energy of hidden states across layers imply higher uncertainty. 
Samples are ranked at random for (1) and by descending estimated uncertainty for (2)--(6).
Using the $|\mathcal{D}_a|$ and $|\mathcal{D}_d|$ from Section~\ref{sec:active_select}, we annotate the top-$|\mathcal{D}_a|$ samples and discard the bottom-$|\mathcal{D}_d|$ samples, as ranked by each baseline.

\subsection{Main Results}
As shown in Table~\ref{tab:main_results}, we conduct two sets of experiments. 
First, we retain the full dataset without discarding any samples and annotate $29.3\%$ high-uncertainty samples selected by each method for semi-supervised RLVR.
PivotTrace outperforms all baselines, achieving average accuracy gains of $1.3\%$ (in-domain, ID) and $1.4\%$ (out-of-domain, OOD) over the strongest baseline. 
This indicates that our method more effectively identifies samples with unreliable self-supervision that benefit most from annotation, thereby directing human labeling to the model’s knowledge frontier.

Second, under each strategy, we annotate $29.3\%$ of the full dataset, exclude $42.1\%$ of low-utility samples, and train on the remaining $57.9\%$. 
In this experimental setting, most methods show gains over their corresponding full-data baselines, yet PivotTrace remains SOTA, outperforming the best baseline by $1.6\%$ (ID) and $2.4\%$ (OOD) in average accuracy. 
Notably, PivotTrace even surpasses fully supervised training on the full dataset. 
This validates our theoretical claim that low-uncertainty samples exhibit limited learning utility and can even degrade performance, as they typically yield zero policy gradients, shrinking the magnitude and increasing the noise sensitivity of batch updates.
Collectively, these results confirm that PivotTrace achieves superior RLVR with improved training and annotation efficiency (see Appendix~\ref{app:more_model} for additional model scales and architectures).

\begin{figure*}[t]
    \centering
    \begin{subfigure}[t]{0.32\textwidth}
        \centering
        \includegraphics[width=\textwidth]{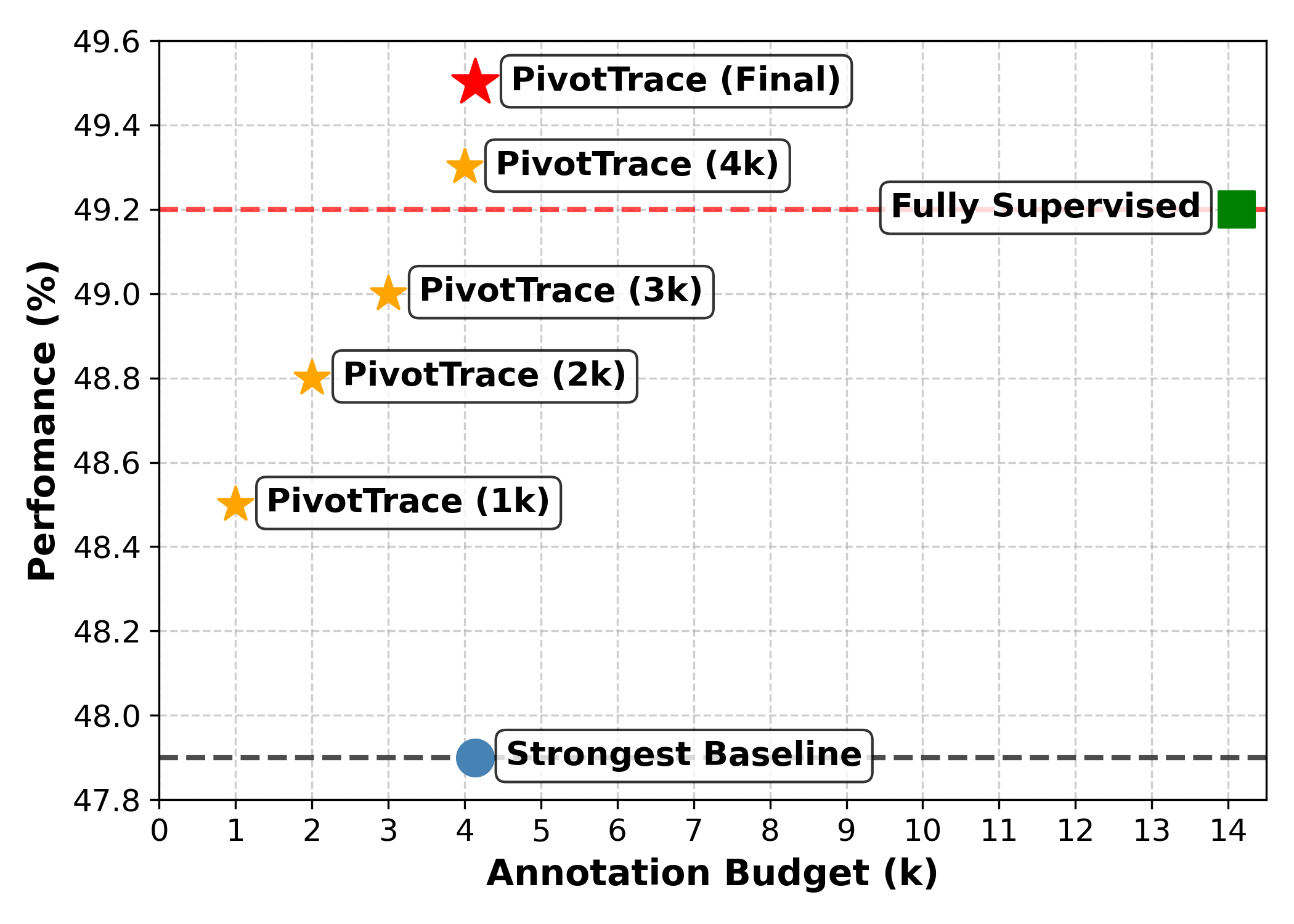}
        \caption{\small Performance vs. annotation size}
        \label{fig:diff_annotate}
    \end{subfigure}
    \hfill
    \begin{subfigure}[t]{0.38\textwidth}
        \centering
        \includegraphics[width=\textwidth]{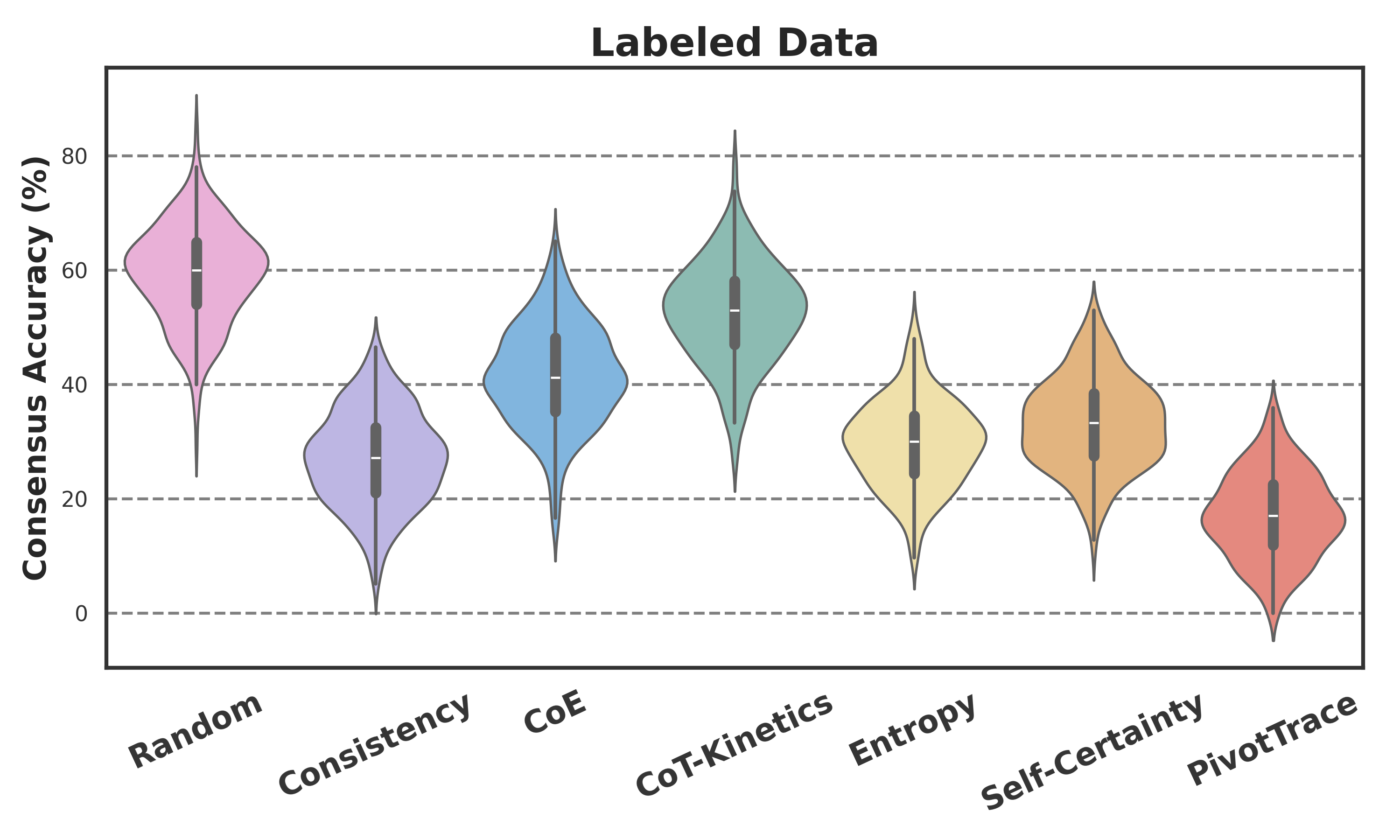}
        \caption{\small Consensus accuracy of labeled data}
        \label{fig:cons_acc_l}
    \end{subfigure}
    \hfill
    \begin{subfigure}[t]{0.27\textwidth}
        \centering
        \includegraphics[width=\textwidth]{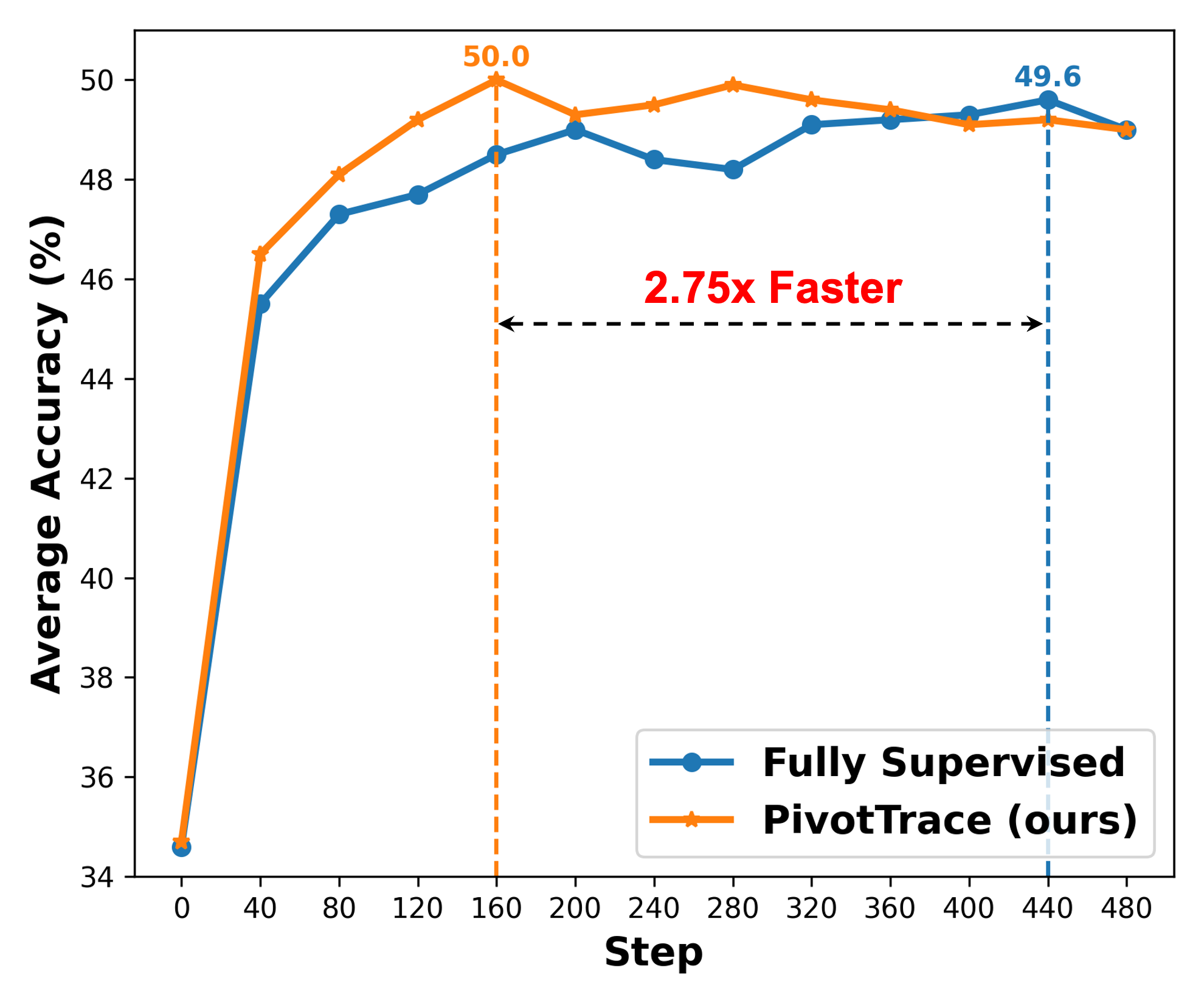}
        \caption{\small Validation performance}
        \label{fig:train_eff}
    \end{subfigure}
    
    \caption{\small (a) Model performance scales positively with annotation budget.
    (b) PivotTrace prioritizes samples with highest uncertainty for annotation. 
    (c) PivotTrace achieves 2.75$\times$ speedup over fully supervised baseline.}
    \vskip -0.1in
\end{figure*}

\subsection{Further Analysis}
\paragraph{PivotTrace under fixed annotation budget.}
While PivotTrace automatically determines the annotation set $\mathcal{D}_a$ based on a threshold $\tau_h$ that is computed on-the-fly, practical scenarios often impose a fixed annotation budget $B < |\mathcal{D}_a|$. 
To assess PivotTrace under such constraints, we uniformly sample $B \in \{1\text{k}, 2\text{k}, 3\text{k}, 4\text{k}\}$ examples from $\mathcal{D}_a$ for annotation. 
The sampled labeled subset and the unlabeled set $\mathcal{D}_u$ are used for semi-supervised RLVR.
As shown in Figure~\ref{fig:diff_annotate}, model performance consistently improves as $B$ increases. 
Notably, PivotTrace with only $1$k annotations outperforms the strongest baseline by a large margin, which uses over $4$k annotations.
Besides, at $B = 4$k (nearly the final dynamically determined $|\mathcal{D}_a| = 4{,}136$), PivotTrace surpasses the fully supervised baseline that uses over $14$k annotated samples. 
This demonstrates that PivotTrace achieves remarkable effectiveness even under limited annotation budgets.

\paragraph{Higher annotation efficiency with PivotTrace.} 
To further validate the annotation efficiency of PivotTrace, we first visualize the consensus accuracy distributions on samples selected for labeling. 
As shown in Figure~\ref{fig:cons_acc_l}, baseline methods select samples where consensus accuracy clusters within the $0.3$--$0.6$ range, while PivotTrace targets instances with substantially lower accuracy (mostly $<20\%$). 
This indicates that our method bypasses well-mastered samples and selectively identifies samples on which the model is prone to incorrect consensus, maximizing the marginal utility of each annotation.

Besides, we conduct an ablation study to quantify the efficiency gap between different strategies.
Our method finally annotates the $4{,}136$ samples with the highest pivot count (Top-4k). 
For comparison, we annotate $8{,}272$ samples using random selection (Random-8k) and selection by lowest pivot count (Bottom-8k).
Table~\ref{tab:ablation_annotation} shows that with twice as many labels, Random-8k lags Top-4k by $1.2\%$ (ID) and $1.3\%$ (OOD), and Bottom-8k performs even worse, falling behind by $2.4\%$ (ID) and $2.7\%$ (OOD).
This highlights that poorly allocated annotations can severely degrade model performance under semi-supervised RLVR, and PivotTrace achieves superior annotation efficiency with far fewer labels.

\begin{table*}[!t]
\centering
\caption{\small Ablation study on annotation efficiency: top-$k$ vs. bottom-$k$ vs. random selection.}
\label{tab:ablation_annotation}
\setlength{\tabcolsep}{5pt}  
\renewcommand{\arraystretch}{1.1} 
\resizebox{\textwidth}{!}{%
\begin{tabular}{lcccccc|cccc}
\toprule
\multirow{2}{*}{\textbf{Methods}} & \multicolumn{6}{c}{\textbf{In-Domain Performance}} & \multicolumn{4}{c}{\textbf{Out-of-Domain Performance}} \\
\cmidrule(lr){2-7} \cmidrule(lr){8-11}
 & \textbf{AIME 24/25} & \textbf{AMC} & \textbf{MATH-500} & \textbf{Minerva} & \textbf{Olympiad} & \textbf{Avg.} & \textbf{ARC-c} & \textbf{GPQA}$^{*}$ & \textbf{MMLU-Pro} & \textbf{Avg.} \\
\midrule
\rowcolor{Gray}\textbf{Top-4k} & \textbf{25.7}/\textbf{24.1} & \textbf{63.1} & \textbf{87.5} & \textbf{43.2} & \textbf{49.9} & \textbf{48.9} & \textbf{92.8} & \textbf{36.6} & \textbf{62.7} & \textbf{64.0} \\
Bottom-8k & 23.8/22.3 & 57.2 & 85.0 & 42.5 & 48.2 & 46.5 & 92.0 & 31.6 & 60.2 & 61.3 \\
Random-8k & 24.3/23.4 & 59.8 & 86.3 & 43.0 & 49.6 & 47.7 & 92.5 & 33.5 & 62.1 & 62.7 \\
\bottomrule
\end{tabular}%
}
\end{table*}

\begin{figure*}[t]
    \vskip -0.1in
    \centering
    \begin{subfigure}[t]{0.367\textwidth}
        \centering
        \includegraphics[width=\textwidth]{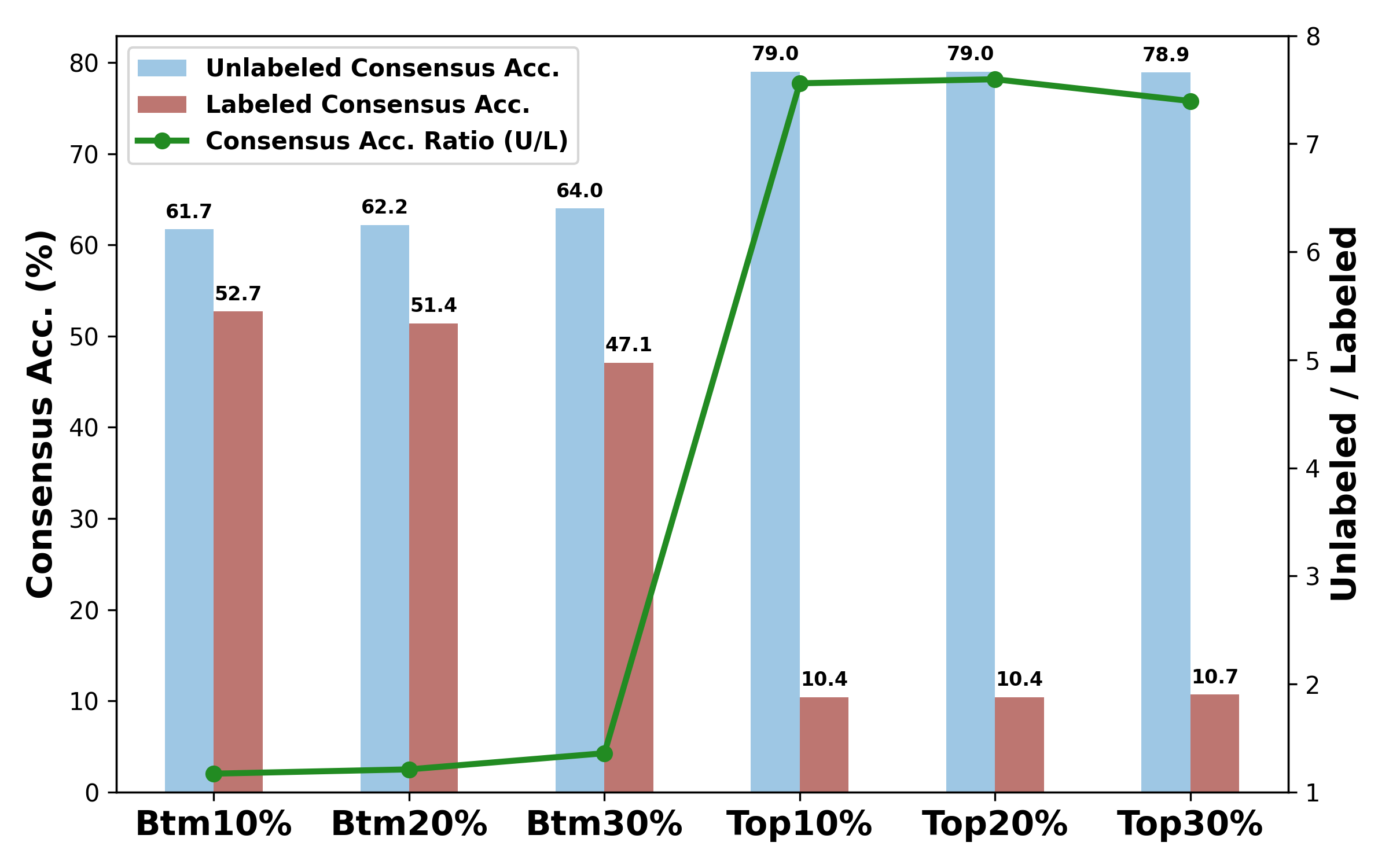}
        \caption{\small Long- vs. short-range head ablation}
        \label{fig:ablation_head}
    \end{subfigure}
    \hfill
    \begin{subfigure}[t]{0.3\textwidth}
        \centering
        \includegraphics[width=\textwidth]{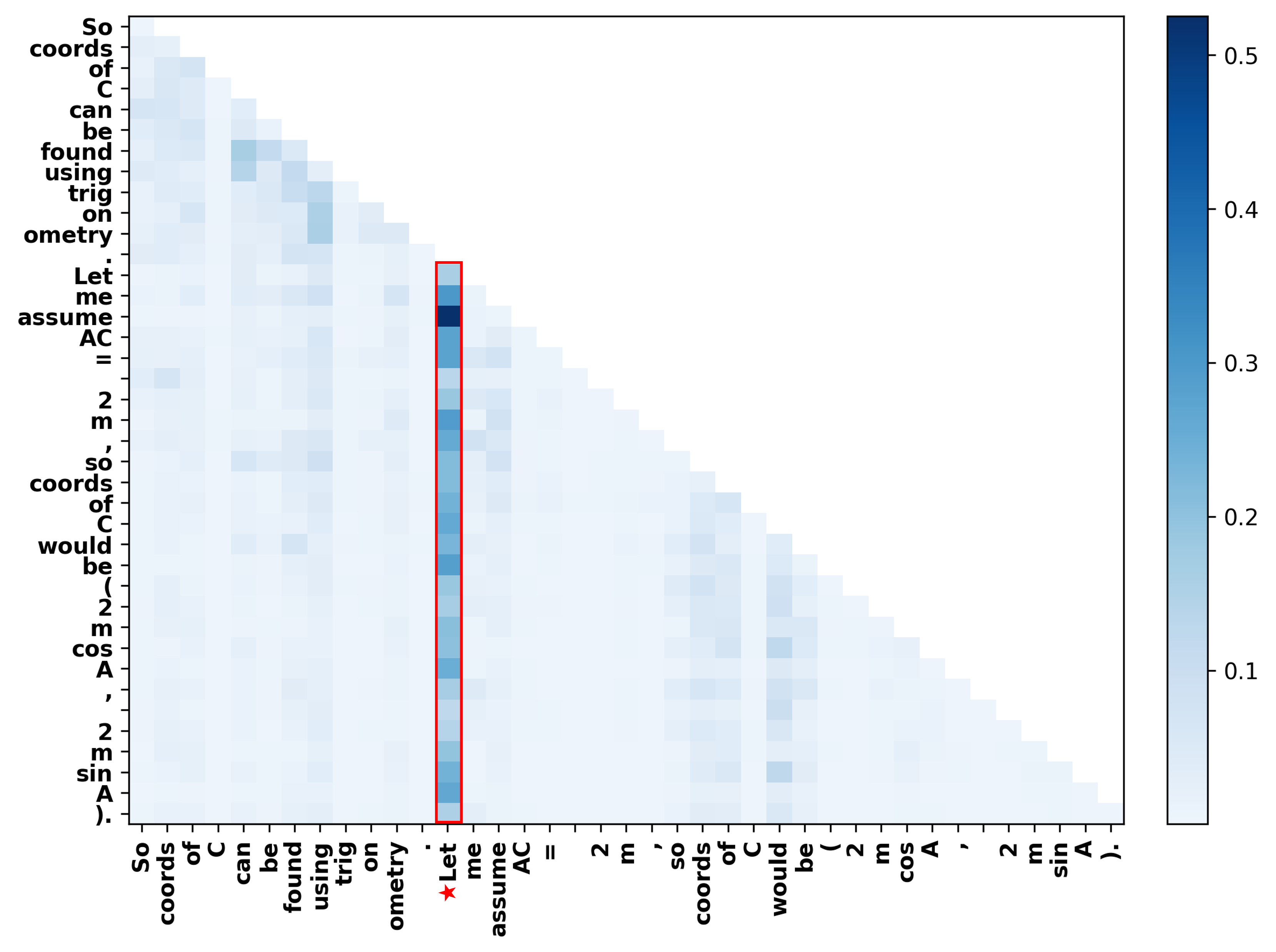}
        \caption{\small Long-range head attention}
        \label{fig:attn_map}
    \end{subfigure}
    \begin{subfigure}[t]{0.3\textwidth}
        \centering
        \includegraphics[width=\textwidth]{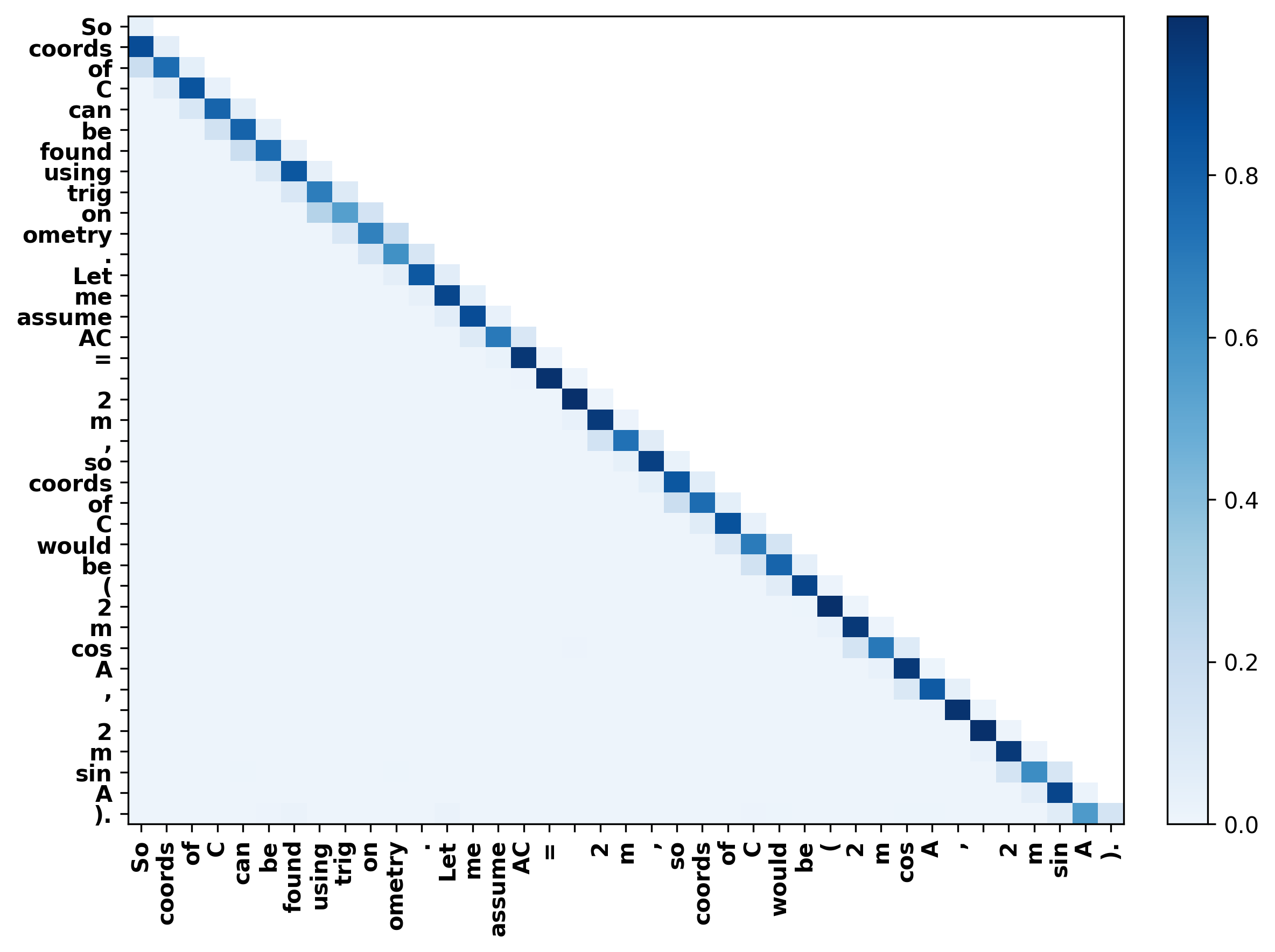}
        \caption{\small Short-range head attention}
        \label{fig:attn_map_l}
    \end{subfigure}
    \hfill
    
    \caption{\small (a) Long-range heads (top~$10$--$30\%$) effectively identify pivot-rich, high-uncertainty questions to annotate, while short-range heads (bottom~$10$--$30\%$) lack selectivity.
    (b) Attention map of long-range heads focus on global reasoning.
    (c) Attention map of short-range heads exhibit strong local concentration.}
    \vskip -0.1in
    \label{fig:attn_head}
\end{figure*}

\paragraph{Higher training efficiency with PivotTrace.} 
To evaluate the training efficiency enabled by PivotTrace, we compare semi-supervised training on the samples it selects against fully supervised training on the full dataset.
As illustrated in Figure~\ref{fig:train_eff}, which plots performance over the first $480$ training steps, our method not only achieves higher performance but also converges much faster. 
Specifically, it reaches a performance level above $49.2\%$ in just $160$ steps, compared to $440$ steps required by the fully supervised baseline, yielding a $2.75\times$ speedup in convergence.
This underscores that PivotTrace effectively identifies high-utility samples, accelerating learning and improving sample efficiency.

\paragraph{Functional analysis of attention heads.}
As stated in Section~\ref{sec:pivot_detect}, attention heads fall into two categories: short-range heads, which focus on local token neighborhoods (Figure~\ref{fig:attn_map_l}), and long-range heads, which capture global reasoning patterns (Figure~\ref{fig:attn_map}).
We propose excluding short-range heads to mitigate local bias.
To justify this choice, we select the top and bottom $10$--$30\%$ of heads using the metric in Eq.~\eqref{eq:head_attn}.
Using each head subset, we detect pivot tokens, select samples for annotation, and compute the consensus accuracy on unlabeled and labeled samples; 
their ratio (unlabeled/labeled) reflects how well the selected samples align with the model’s uncertainty.
As shown in Figure~\ref{fig:ablation_head}, short-range heads show minimal selectivity, with unlabeled-to-labeled consensus accuracy ratios near $1.2$. 
Conversely, long-range heads achieve ratios exceeding $7$, demonstrating superior pivot detection and uncertainty estimation. 
Consequently, we retain only the top $20\%$ of heads for Eq.~\eqref{eq:final_attn}.

\section{Conclusion}
In this work, we propose PivotTrace, an active data selection framework designed to jointly alleviate the high training and annotation costs in RLVR. 
We focus on a more realistic and challenging setting termed ``pick in the dark'': before annotation, one must identify which samples are low-utility, learnable using self-generated rewards, or require external supervision for stable training. 
The key insight of PivotTrace is to leverage the dynamics of long-range attention during reasoning to detect metacognitive pivots, which mark moments of internal reasoning redirection.
These pivots serve as fine-grained indicators of reasoning uncertainty, thereby informing active data selection. 
Excitingly, we show that by carefully curating data, PivotTrace even exceeds fully supervised baseline, with far less annotation and faster training. 
We hope that our work can draw attention back to data curation for RLVR and pave the way toward efficient, scalable post-training in the era of data-centric AI. 

\begin{ack}
This paper is supported by the NSFC under Grants (No. 62402424) and Ant Group through CCF-Ant Research Fund (No. CCF-AFSG RF20240406).
Haobo Wang is also supported by the Ningbo Yongjiang Talent Introduction Programme (No. 2023A-399-G).
\end{ack}

\bibliography{neurips_2026}
\bibliographystyle{neurips_2026}


\appendix
\newpage
\section{Theoretical Proof}
\label{sec:theory_proof}
\subsection{Derivation of Instantaneous Learning Utility}
\begin{proposition}
\label{prop:instantaneous_opt_potential}
Let $\mu_{\theta_{\mathrm{old}}}(q)$ denote the expected binary reward of the policy $\pi_{\theta_{\mathrm{old}}}$ on question $q$. Under a trust-region policy update with KL constraint $\delta$, the maximal achievable improvement in the surrogate objective satisfies
\begin{equation}
|\Delta \mathcal{J}(q)| \leq \sqrt{ 2\delta \, \mu_{\theta_{\mathrm{old}}}(q) \big(1 - \mu_{\theta_{\mathrm{old}}}(q)\big) }.
\end{equation}
\end{proposition}
\begin{proof}
Following the framework of traditional trust region methods~\citep{schulman2015trust,schulman2017proximal}, we maximize a surrogate objective subject to a constraint that limits the magnitude of the policy update. 
Specifically, for a given question $q$, we formulate the optimization problem as:
\begin{equation}
\min_{\theta} \mathcal{J}(\theta; q) = \min_{\theta} -\mathbb{E}_{y \sim \pi_{\theta}(\cdot \mid q)} \left[ A_{\theta_{\mathrm{old}}}(q, y) \right],
\quad \text{subject to} \,\,
\mathbb{D}_{\mathrm{KL}} \big( \pi_{\theta_{\mathrm{old}}}(\cdot \mid q) \|\ \pi_{\theta}(\cdot \mid q) \big) \leq \delta
\end{equation}
where $A_{\theta_{\mathrm{old}}}(q,y)$ denotes the advantage of generating response $y$ to question $q$ under the old policy $\pi_{\theta_{\mathrm{old}}}$ and $\delta > 0$ controls the size of the trust region.
To enable local analysis, we reparameterize the policy parameters as $\theta = \theta_{\mathrm{old}} + d$, where $d$ represents a small policy update. Substituting into the constrained problem and applying the method of Lagrange multipliers yields the following Lagrangian formulation:
\begin{equation}
d^* = \operatorname*{argmin}_{d} \mathcal{J}(\theta_{\mathrm{old}} + d; q) + \lambda \left( \mathbb{D}_{\mathrm{KL}} \big( \pi_{\theta_{\mathrm{old}}}(\cdot \mid q) \|\ \pi_{\theta_{\mathrm{old}} + d}(\cdot \mid q) \big) - \delta \right).
\end{equation}
where $\lambda \geq 0$  is the Lagrange multiplier. Using a second-order Taylor expansion of the Lagrangian around $d = 0$ (i.e., $\theta = \theta_{\mathrm{old}}$), we obtain:
\begin{align}
\label{eq:lagrangian}
d^* 
&= \arg\min_{d} \; \mathcal{J}(\theta_{\mathrm{old}} + d; q) + \lambda \left( \mathbb{D}_{\mathrm{KL}} \big( \pi_{\theta_{\mathrm{old}}}(\cdot \mid q) \,\|\, \pi_{\theta_{\mathrm{old}} + d}(\cdot \mid q) \big) - \delta \right) \nonumber \\
&\approx \arg\min_{d} \; \underbrace{\mathcal{J}(\theta_{\mathrm{old}}; q) - \lambda \delta}_{\text{constant w.r.t. } d}
+ \nabla_\theta \mathcal{J}(\theta; q)^\top d \Big|_{\theta = \theta_{\mathrm{old}}}
+ \frac{\lambda}{2} \, d^\top \nabla_\theta^2 \mathbb{D}_{\mathrm{KL}} \big( \pi_{\theta_{\mathrm{old}}}(\cdot \mid q) \,\|\, \pi_{\theta}(\cdot \mid q) \big) \Big|_{\theta = \theta_{\mathrm{old}}} d.
\end{align}
To proceed, we compute the first- and second-order derivatives of the KL divergence at $\theta = \theta_{\mathrm{old}}$. 
The first derivative vanishes because the KL divergence achieves its minimum at $\theta = \theta_{\mathrm{old}}$, or equivalently, because the expectation of the score function under its own distribution is zero:
\begin{equation}
\nabla_\theta \mathbb{D}_{\mathrm{KL}} \big( \pi_{\theta_{\mathrm{old}}}(\cdot \mid q) \,\|\, \pi_{\theta}(\cdot \mid q) \big) \Big|_{\theta = \theta_{\mathrm{old}}} = 0.
\end{equation}
The Hessian, however, is nontrivial and coincides with the Fisher information matrix. Specifically,
\begin{equation}
\begin{aligned}
\nabla_\theta^2 \mathbb{D}_{\mathrm{KL}} \big( \pi_{\theta_{\mathrm{old}}}(\cdot \mid q) \,\|\, \pi_{\theta}(\cdot \mid q) \big) \Big|_{\theta = \theta_{\mathrm{old}}}
&= - \nabla_\theta^2 \mathbb{E}_{y \sim \pi_{\theta_{\mathrm{old}}}(\cdot \mid q)} \big[ \log \pi_{\theta}(y \mid q) \big] \Big|_{\theta = \theta_{\mathrm{old}}} \\
&= - \mathbb{E}_{y \sim \pi_{\theta_{\mathrm{old}}}(\cdot \mid q)} \big[ \nabla_\theta^2 \log \pi_{\theta}(y \mid q) \big] \Big|_{\theta = \theta_{\mathrm{old}}} \\
&= \mathbb{E}_{y \sim \pi_{\theta_{\mathrm{old}}}(\cdot \mid q)} \left[ \nabla_\theta \log \pi_{\theta}(y \mid q) \, \nabla_\theta \log \pi_{\theta}(y \mid q)^\top \right] \Big|_{\theta = \theta_{\mathrm{old}}} \\
&= F(q; \theta_{\mathrm{old}})
\end{aligned}
\end{equation}
where $F(q; \theta) = \mathbb{E}_{y \sim \pi_{\theta}(\cdot \mid q)} \big[ \nabla_\theta \log \pi_\theta(y \mid q) \, \nabla_\theta \log \pi_\theta(y \mid q)^\top \big]$ denotes the Fisher information matrix. 
Here, the third equality follows from the identity 
$\mathbb{E}_{\pi_\theta}[\nabla_\theta^2 \log \pi_\theta] = -\mathbb{E}_{\pi_\theta}[(\nabla_\theta \log \pi_\theta)(\nabla_\theta \log \pi_\theta)^\top]$, which holds due to the normalization of $\pi_\theta$.
Using these local approximations, the constrained optimization problem in Eq.~\eqref{eq:lagrangian} reduces to the quadratic surrogate:
\begin{equation}
\label{eq:approx_lagrangian}
d^* = \arg\min_{d} \; g^\top d + \frac{\lambda}{2} d^\top F(q; \theta_{\mathrm{old}}) d
\end{equation}
where $g = \nabla_\theta \mathcal{J}(\theta; q) \big|_{\theta = \theta_{\mathrm{old}}}$ is the policy gradient evaluated at $\theta_{\mathrm{old}}$.
Differentiating the surrogate objective in Eq.~\eqref{eq:approx_lagrangian} with respect to $d$ and setting the gradient to zero yields the optimality condition,
\begin{equation}
\label{eq:optimal_d}
g + \lambda F(q; \theta_{\mathrm{old}}) d = 0
\quad \Rightarrow \quad
d = -\frac{1}{\lambda} F(q; \theta_{\mathrm{old}})^{-1} g.
\end{equation}
To determine the Lagrange multiplier $\lambda$, we enforce the trust-region constraint by approximating the KL divergence with its second-order Taylor expansion around $\theta_{\mathrm{old}}$:
\begin{equation}
\label{eq:KL_constraint}
\mathbb{D}_{\mathrm{KL}} \big( \pi_{\theta_{\mathrm{old}}}(\cdot \mid q) \,\|\, \pi_{\theta_{\mathrm{old}} + d}(\cdot \mid q) \big)
\approx \frac{1}{2} d^\top F(q; \theta_{\mathrm{old}}) d = \delta,
\end{equation}
Substituting $d$ in Eq.~\eqref{eq:optimal_d} into the KL constraint in Eq.~\eqref{eq:KL_constraint} yields
\begin{equation}
\frac{1}{2\lambda^2} \, g^\top F(q; \theta_{\mathrm{old}})^{-1} g = \delta,
\end{equation}
which gives the closed-form solution
\begin{equation}
\lambda = \sqrt{ \frac{ g^\top F(q; \theta_{\mathrm{old}})^{-1} g }{2\delta} }.
\end{equation}
Consequently, the approximate improvement in the surrogate objective is
\begin{equation}
\Delta \mathcal{J}(q) 
:= \mathcal{J}(\theta_{\mathrm{old}} + d^*; q) - \mathcal{J}(\theta_{\mathrm{old}}; q)
\approx g^\top d^*
= - \sqrt{ 2\delta \, g^\top F(q; \theta_{\mathrm{old}})^{-1} g }.
\end{equation}
The policy gradient $g$ is fully determined by the task reward. 
Let $R(y, a) \in \{0,1\}$ denote the binary reward for generating response $y$ to question $q$, as defined in Section~\ref{sec:preliminary}, where $a$ is the ground-truth answer for $q$. 
The expected reward (\textit{i.e.}, the probability of sampling a correct answer) is
\begin{equation}
\mu_\theta(q) := \mathbb{E}_{y \sim \pi_\theta(\cdot \mid q)}[R(y, a)].
\end{equation}
Under an unbiased advantage estimator with a state-dependent baseline, the gradient simplifies to
\begin{equation}
g = -\nabla_\theta \mu_\theta(q) \big|_{\theta = \theta_{\mathrm{old}}}.
\end{equation}
Since $R(y, a)$ is an unbiased estimator of $\mu_\theta(q)$ and follows a Bernoulli distribution under $\pi_{\theta_{\mathrm{old}}}(\cdot \mid q)$, its variance is
\begin{equation}
\mathbb{V}_{\theta_{\mathrm{old}}}[R(y, a)] = \mu_{\theta_{\mathrm{old}}}(q) \big(1 - \mu_{\theta_{\mathrm{old}}}(q)\big).
\end{equation}
Moreover, under standard regularity conditions, the Cram\'er--Rao inequality implies that the variance of any unbiased estimator of $\mu_\theta(q)$ upper-bounds the squared natural gradient norm:
\begin{equation}
g^\top F(q; \theta_{\mathrm{old}})^{-1} g \leq \mathbb{V}_{\theta_{\mathrm{old}}}[R(y, a)] = \mu_{\theta_{\mathrm{old}}}(q) \big(1 - \mu_{\theta_{\mathrm{old}}}(q)\big).
\end{equation}
It follows that the magnitude of the achievable improvement for question $q$ is bounded as
\begin{equation}
|\Delta \mathcal{J}(q)| = \sqrt{ 2\delta \, g^\top F(q; \theta_{\mathrm{old}})^{-1} g }
\leq \sqrt{ 2\delta \, \mu_{\theta_{\mathrm{old}}}(q) \big(1 - \mu_{\theta_{\mathrm{old}}}(q)\big) }.
\end{equation}
\end{proof}

\subsection{Derivation of Total Learning Utility}
\label{app:proof_prop}
In trust-region policy optimization algorithms such as TRPO~\citep{schulman2015trust}, updates are constrained to ensure non-decreasing performance, which is quantified by \(\mu_{\theta}(q)\) in RLVR. Empirically, learning often exhibits rapid early progress followed by diminishing returns as performance saturates near its optimum. To capture this behavior in a tractable form, we make the following modeling assumption:
\begin{assumption}
\label{asm:exp_conv}
The accuracy trajectory \(\mu(t) := \mu_{\theta_t}(q)\) evolves continuously in training time \(t\) according to
\[
\frac{d\mu}{dt} = k(1 - \mu), \quad k > 0,
\]
implying monotonic convergence to perfect performance (\(\mu \to 1\)) with a rate proportional to the remaining error \(1 - \mu\).
\end{assumption}
This dynamics aligns with the monotonic improvement guarantee of TRPO and reflects the empirically observed “exponential approach” to saturation in many learning systems~\citep{shamir2019exponential,zhou2025bridging,yang2025bounded,yang2026can}.

\begin{proposition}
Let $\mu_{\theta_0}(q)=\mathbb{E}_{y \sim \pi_{\theta_0}(\cdot \mid q)}[R(y, a)]$ denote the expected binary correctness reward of the initial policy $\pi_{\theta_0}$ on question $q$. 
Under trust-region updates with a fixed KL budget $\delta > 0$ and convergence dynamics $\frac{d\mu}{dt} = k(1 - \mu)$ for some constant $k > 0$, the total learning utility accumulated from $\mu_{\theta_0}(q)$ to convergence is given by
\begin{equation}
\mathcal{G}(q) \!=\! \frac{\sqrt{2\delta}}{k} \! \left( \! \frac{\pi}{2} \! - \! \arcsin \! \big(\sqrt{\mu_{\theta_0}(q)}\big) \! + \! \sqrt{\mu_{\theta_0}(q)\big(1 \! - \! \mu_{\theta_0}(q)\big)} \right)
\end{equation}
which is strictly decreasing in $\mu_{\theta_0}(q)$.
\end{proposition}

\begin{proof}
By Proposition~\ref{prop:instantaneous_opt_potential}, the maximal instantaneous learning utility at performance level $\mu$ is $\sqrt{2\delta\, \mu(1 - \mu)}$. 
To compute the total expected learning utility over the training trajectory, we integrate this instantaneous gain over time, using the convergence dynamics specified in Assumption~\ref{asm:exp_conv}. 
Changing the integration variable from time $t$ to performance $\mu$, and noting that
\begin{equation}
dt = \frac{d\mu}{d\mu/dt} = \frac{d\mu}{k(1 - \mu)},
\end{equation}
we obtain the total expected learning utility:
\begin{equation}
\mathcal{G}(q) = \int_{0}^{\infty} \sqrt{2\delta\, \mu(t)(1 - \mu(t))} \, dt
= \int_{\mu_{\theta_0}(q)}^{1} \frac{\sqrt{2\delta\, \mu(1 - \mu)}}{k(1 - \mu)} \, d\mu
= \frac{\sqrt{2\delta}}{k} \int_{\mu_{\theta_0}(q)}^{1} \frac{\sqrt{\mu}}{\sqrt{1 - \mu}} \, d\mu.
\end{equation}
To evaluate the integral, we first substitute $u = \sqrt{\mu}$, so that $\mu = u^2$ and $d\mu = 2u\,du$. The limits become $u \in [\sqrt{\mu_{\theta_0}(q)}, 1]$, yielding
\begin{align}
\int_{\mu_{\theta_0}(q)}^{1} \frac{\sqrt{\mu}}{\sqrt{1 - \mu}} \, d\mu
&= \int_{\sqrt{\mu_{\theta_0}(q)}}^{1} \frac{u}{\sqrt{1 - u^2}} \cdot 2u \, du
= 2 \int_{\sqrt{\mu_{\theta_0}(q)}}^{1} \frac{u^2}{\sqrt{1 - u^2}} \, du.
\end{align}
Next, let $u = \sin\theta$, so that $du = \cos\theta\,d\theta$ and $\sqrt{1 - u^2} = \cos\theta$. When $u = \sqrt{\mu_{\theta_0}(q)}$, we have $\theta = \arcsin(\sqrt{\mu_{\theta_0}(q)})$; when $u = 1$, $\theta = \pi/2$. The integral simplifies to
\begin{align}
2 \int_{\arcsin(\sqrt{\mu_{\theta_0}(q)})}^{\pi/2} \frac{\sin^2\theta}{\cos\theta} \cdot \cos\theta \, d\theta
&= 2 \int_{\arcsin(\sqrt{\mu_{\theta_0}(q)})}^{\pi/2} \sin^2\theta \, d\theta \\
&= \int_{\arcsin(\sqrt{\mu_{\theta_0}(q)})}^{\pi/2} (1 - \cos 2\theta) \, d\theta \\
&= \Big[ \theta - \tfrac{1}{2}\sin 2\theta \Big]_{\arcsin(\sqrt{\mu_{\theta_0}(q)})}^{\pi/2}.
\end{align}
Evaluating the boundary terms:
\[
\left. \theta - \tfrac{1}{2}\sin 2\theta \right|_{\theta = \pi/2} = \frac{\pi}{2}, \quad
\left. \theta - \tfrac{1}{2}\sin 2\theta \right|_{\theta = \arcsin(\sqrt{\mu_{\theta_0}(q)})} = \arcsin(\sqrt{\mu_{\theta_0}(q)}) - \sqrt{\mu_{\theta_0}(q)(1 - \mu_{\theta_0}(q))},
\]
where we used the identity $\sin(2\arcsin x) = 2x\sqrt{1 - x^2}$. Therefore,
\begin{equation}
\int_{\mu_{\theta_0}(q)}^{1} \frac{\sqrt{\mu}}{\sqrt{1 - \mu}} \, d\mu
= \frac{\pi}{2} - \arcsin\!\big(\sqrt{\mu_{\theta_0}(q)}\big) + \sqrt{\mu_{\theta_0}(q)\big(1 - \mu_{\theta_0}(q)\big)}.
\end{equation}
Substituting back, we obtain the closed-form expression:
\begin{equation}
\mathcal{G}(q) = \frac{\sqrt{2\delta}}{k} \left( \frac{\pi}{2} - \arcsin\!\big(\sqrt{\mu_{\theta_0}(q)}\big) + \sqrt{\mu_{\theta_0}(q)\big(1 - \mu_{\theta_0}(q)\big)} \right).
\end{equation}
To verify monotonicity, let $\mu_0 := \mu_{\theta_0}(q) \in (0,1)$ and differentiate:
\begin{align}
\frac{d\mathcal{G}}{d\mu_0}
&= \frac{\sqrt{2\delta}}{k} \left( 
- \frac{1}{\sqrt{1 - \mu_0}} \cdot \frac{1}{2\sqrt{\mu_0}} 
+ \frac{1 - 2\mu_0}{2\sqrt{\mu_0(1 - \mu_0)}}
\right) \nonumber \\
&= \frac{\sqrt{2\delta}}{k} \cdot \frac{-1 + (1 - 2\mu_0)}{2\sqrt{\mu_0(1 - \mu_0)}} \nonumber \\
&= -\frac{\sqrt{2\delta}}{k} \cdot \frac{\mu_0}{\sqrt{\mu_0(1 - \mu_0)}} \nonumber \\
&= -\frac{\sqrt{2\delta}}{k} \cdot \frac{\sqrt{\mu_0}}{\sqrt{1 - \mu_0}} < 0.
\end{align}
Thus, $\mathcal{G}(q)$ is strictly decreasing on $(0,1)$. Moreover, $\mathcal{G}(q)$ is continuous on the closed interval $[0,1]$, as can be verified by direct evaluation of the limits:
\[
\lim_{\mu_0 \to 0^+} \mathcal{G}(q) = \frac{\sqrt{2\delta}}{k} \cdot \frac{\pi}{2}, \quad
\lim_{\mu_0 \to 1^-} \mathcal{G}(q) = 0,
\]
and the expression remains finite at both endpoints. Since a continuous function that is strictly decreasing on $(0,1)$ must also be strictly decreasing on $[0,1]$ (in the sense that $\mu_0^{(1)} < \mu_0^{(2)} \implies \mathcal{G}(\mu_0^{(1)}) > \mathcal{G}(\mu_0^{(2)})$ for all $\mu_0^{(1)}, \mu_0^{(2)} \in [0,1]$), we conclude that $\mathcal{G}(q)$ is strictly decreasing on the entire interval $[0,1]$.
\end{proof}

\section{Additional Experimental Setups}
\label{app:exp_setup}
In addition to the selection and training setups described in Section~\ref{sec:setup}, we provide the following supplementary details.

\paragraph{More Selection Details}
For base models (e.g., Qwen3-4B-Base), which initially exhibit poor instruction-following and formatting capabilities, we introduce a warmup training phase to enable effective use of model-generated responses for data selection. Specifically, we construct a warmup dataset comprising 512 question–answer pairs and perform RLVR training for approximately 20 steps to endow the model with basic formatting competence. In contrast, instruct-tuned models (e.g., Deepseek-R1-Distill-Llama-8B and Deepseek-R1-Distill-Qwen-1.5B) do not require this warmup stage due to their stronger out-of-the-box alignment.
Starting from either the warmed-up base model or the instruct model, we generate one response per question in the training set (eight responses for the \textit{Consistency} baseline). 
We then perform a forward pass through the model using each (question, response) pair to collect the necessary information required by each baseline for computing its corresponding uncertainty score:
\begin{itemize}[leftmargin=*]
\item For \textit{Entropy} and \textit{Self-Certainty}, we record the predicted token-level probability distributions.
\item For \textit{CoE} and \textit{CoT-Kinetics}, we record the hidden states from all layers at each token position.
\item For \textit{Consistency}, we directly compute the proportion of the majority answer among the eight generated responses.
\item For our proposed method \textit{PivotTrace}, we retain the attention maps from all attention heads.
\end{itemize}
Once the uncertainty scores are obtained, we proceed with data annotation and filtering as described in Section~\ref{sec:setup}.

\paragraph{More Training Details}
The maximum prompt length is set to 1,024 tokens, and the maximum response length is capped at 4,096 tokens. 
We employ \texttt{Math-Verify}\footnote{\url{https://github.com/huggingface/Math-Verify}} as the sole reward function, without incorporating any auxiliary rewards for format or response length. 
Inference is conducted using the \texttt{vLLM}\footnote{\url{https://github.com/vllm-project/vllm}} framework, with 8 rollouts per prompt, temperature fixed at 1.0, and top-$p$ sampling also set to 1.0. 
All prompts during rollouts are prepended with the system prompt: ``\texttt{Let's think step by step and output the final answer within \textbackslash\textbackslash boxed\{\}.}''. 
All semi-supervised baselines are trained for a maximum of 400 steps, whereas fully supervised training on the complete dataset is extended to 600 steps to guarantee sufficient convergence. 
Model checkpoints are saved every 40 training steps for validation.

\section{Additional Experimental Results}
\label{app:exp_results}

\begin{table*}[!t]
\centering
\caption{\small In-domain (ID) and out-of-domain (OOD) performance using Deepseek-R1-Distill-Llama-8B. \textbf{Bold} denotes the best results and $\dagger$ denotes methods requiring multiple stochastic inferences.}
\label{tab:results_llama}
\setlength{\tabcolsep}{3pt}  
\renewcommand{\arraystretch}{1.2} 
\resizebox{\textwidth}{!}{%
\begin{tabular}{lcccccc|cccc}
\toprule
\multirow{2}{*}{\textbf{Methods}} & \multicolumn{6}{c}{\textbf{In-Domain Performance}} & \multicolumn{4}{c}{\textbf{Out-of-Domain Performance}} \\
\cmidrule(lr){2-7} \cmidrule(lr){8-11}
 & \textbf{AIME 24/25} & \textbf{AMC} & \textbf{MATH-500} & \textbf{Minerva} & \textbf{Olympiad} & \textbf{Avg.} & \textbf{ARC-c} & \textbf{GPQA}$^{*}$ & \textbf{MMLU-Pro} & \textbf{Avg.} \\
\midrule
Random & 27.0/14.5 & 34.0 & 22.4 & 14.1 & 21.9 & 22.3 & 25.5 & 22.2 & 43.2 & 30.3 \\
Consistency$^\dagger$ & 28.3/14.6 & 35.9 & 27.4 & 11.9 & 23.4 & 23.6 & 23.3 & 20.3 & 40.7 & 28.1 \\
CoE & 25.0/14.8 & 35.4 & 25.1 & 13.0 & 25.8 & 23.2 & 8.7 & 18.4 & 46.8 & 24.6 \\
CoT-Kinetics & 31.9/19.4 & 41.7 & 28.6 & 11.9 & 26.7 & 26.7 & 15.6 & 18.9 & 42.1 & 25.5 \\
Entropy & 34.8/20.7 & 47.3 & 33.9 & 17.2 & 32.7 & 31.1 & 16.5 & 23.1 & \textbf{47.5} & 29.0 \\
Self-Certainty & 30.9/11.9 & 35.7 & 26.1 & 14.4 & 23.4 & 23.7 & 15.6 & 20.2 & 42.9 & 26.2 \\
\rowcolor{Gray}\textbf{PivotTrace (ours)} & \textbf{35.9}/\textbf{21.9} & \textbf{55.5} & \textbf{40.2} & \textbf{17.7} & \textbf{34.0} & \textbf{34.2} & \textbf{26.9} & \textbf{23.4} & 46.3 & \textbf{32.2} \\
\bottomrule
\end{tabular}%
}
\end{table*}

\subsection{Extend PivotTrace to More Models}
\label{app:more_model}
We further investigate the applicability of PivotTrace across diverse model architectures and scales by selecting data and training with Deepseek-R1-Distill-Qwen-1.5B and Deepseek-R1-Distill-Llama-8B, which represent distinct architectural families (Qwen~\citep{yang2024qwen2} vs. Llama~\citep{grattafiori2024llama}) and parameter scales (1.5B vs. 8B).

The data selection and training protocols for these two models closely follow those described in Section~\ref{sec:setup} and Appendix~\ref{app:exp_setup}, with two minor modifications. 
First, we observe that both models generate substantially longer responses compared to Qwen3-4B-Base; accordingly, we increase the maximum response length during both training and inference from 4,096 to 6,144 tokens.
Second, due to computational resource constraints, we randomly sample a subset of 4,000 examples from the DAPO-Math-14k dataset for all experiments involving these models.

The experimental results for Deepseek-R1-Distill-Llama-8B are presented in Table~\ref{tab:results_llama}. 
PivotTrace outperforms the strongest baseline by $3.1\%$ in average ID accuracy and $1.9\%$ in average OOD accuracy. 
For Deepseek-R1-Distill-Qwen-1.5B, results are shown in Table~\ref{tab:results_qwen2}. 
Since this model has already been distilled with strong reasoning capabilities, the performance gap between different methods is relatively narrow.
Nonetheless, PivotTrace still achieves consistent gains, improving upon the best baseline by $0.3\%$ on ID and $1.0\%$ on OOD benchmarks. 
Notably, the larger relative improvement on OOD data suggests that PivotTrace enhances generalization even in settings where absolute performance is saturated.
These results show that PivotTrace consistently improves reasoning performance across diverse model architectures and scales by enabling more effective data selection.

\begin{figure*}[t]
    \centering
    \begin{subfigure}[t]{0.29\textwidth}
        \centering
        \includegraphics[width=\textwidth]{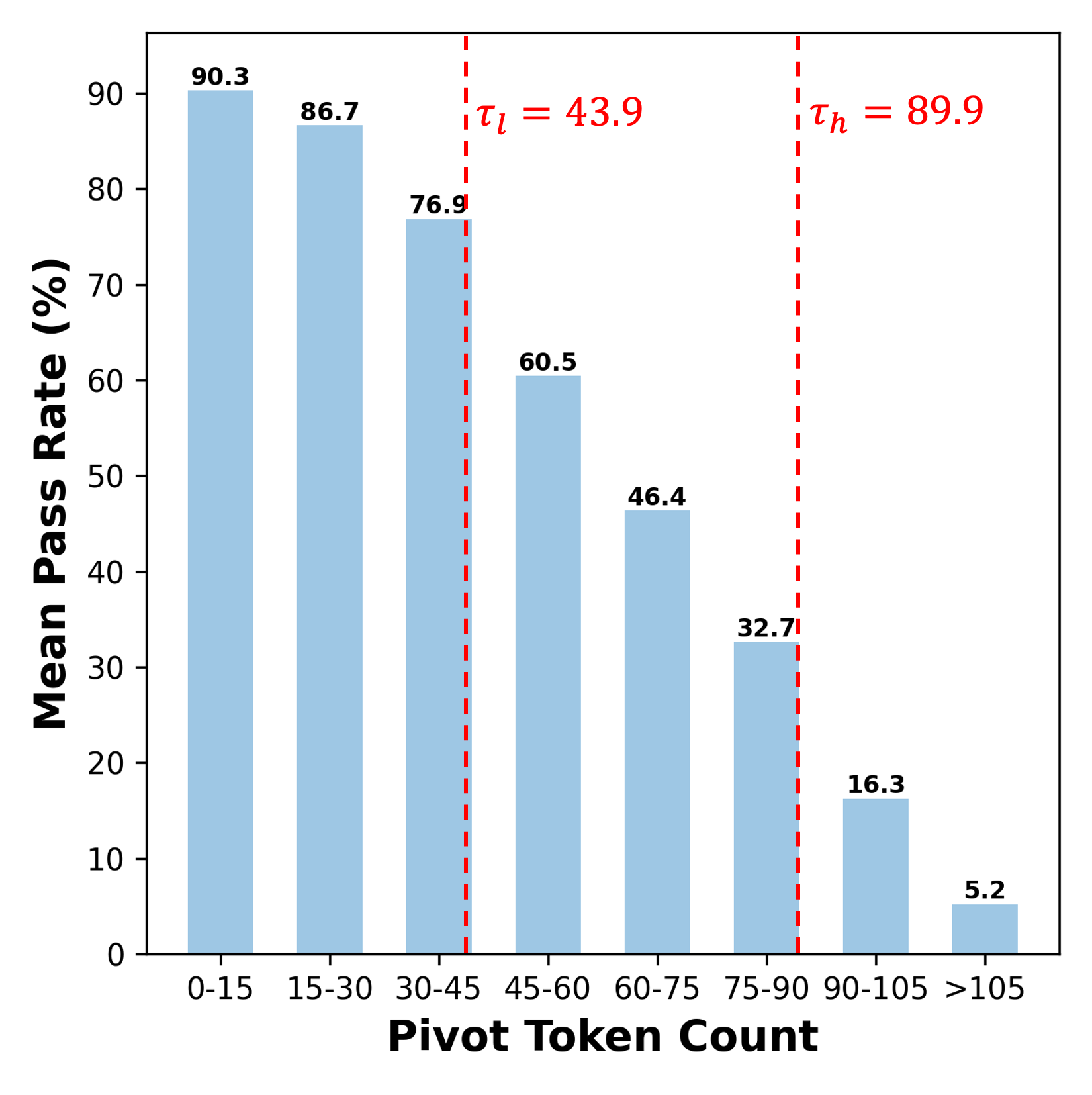}
        \caption{\small Performance vs. pivot count}
        \label{fig:pivot_acc}
    \end{subfigure}
    \hfill
    \begin{subfigure}[t]{0.348\textwidth}
        \centering
        \includegraphics[width=\textwidth]{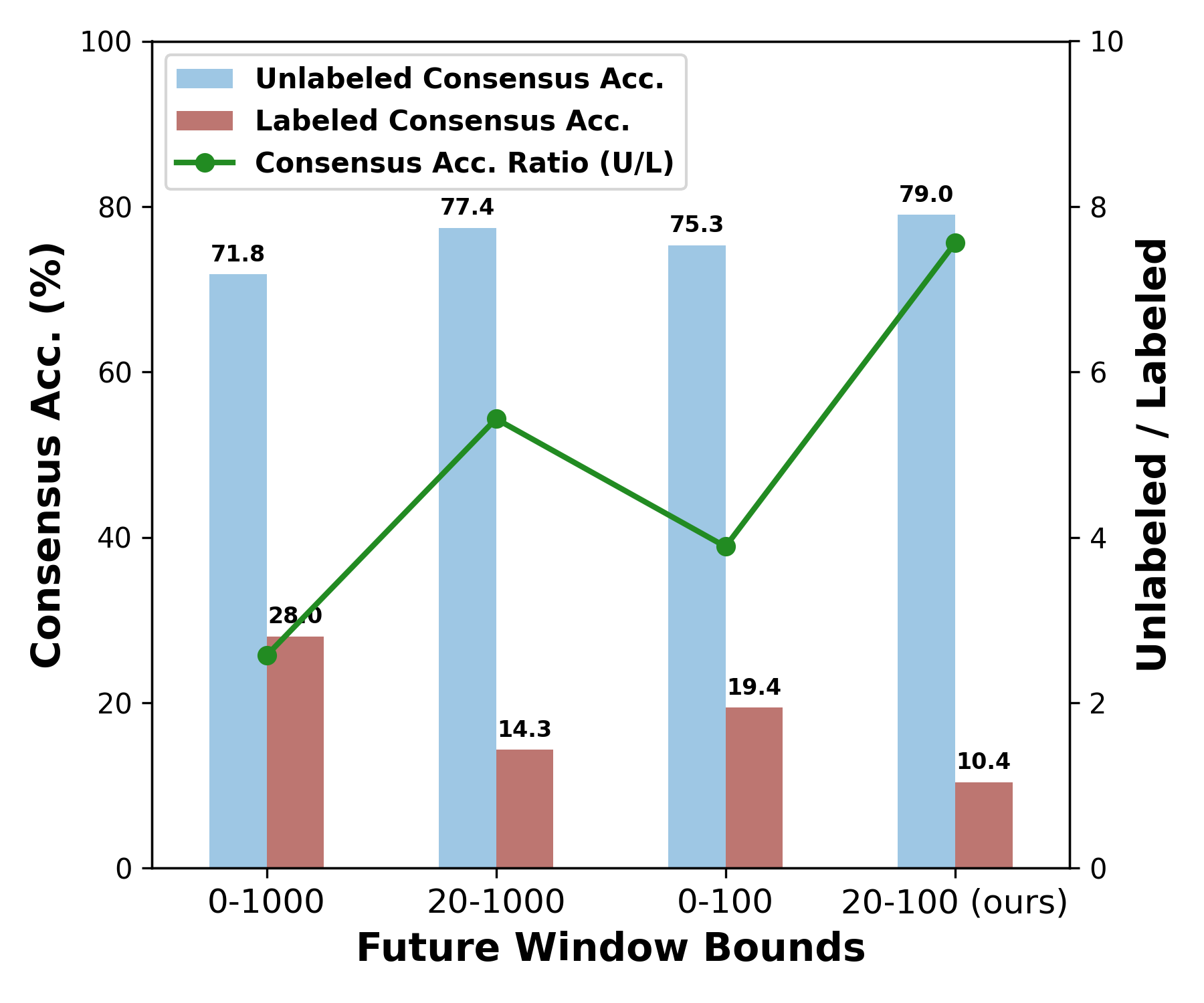}
        \caption{\small Ablation on future window bounds}
        \label{fig:future_window}
    \end{subfigure}
    \hfill
    \begin{subfigure}[t]{0.345\textwidth}
        \centering
        \includegraphics[width=\textwidth]{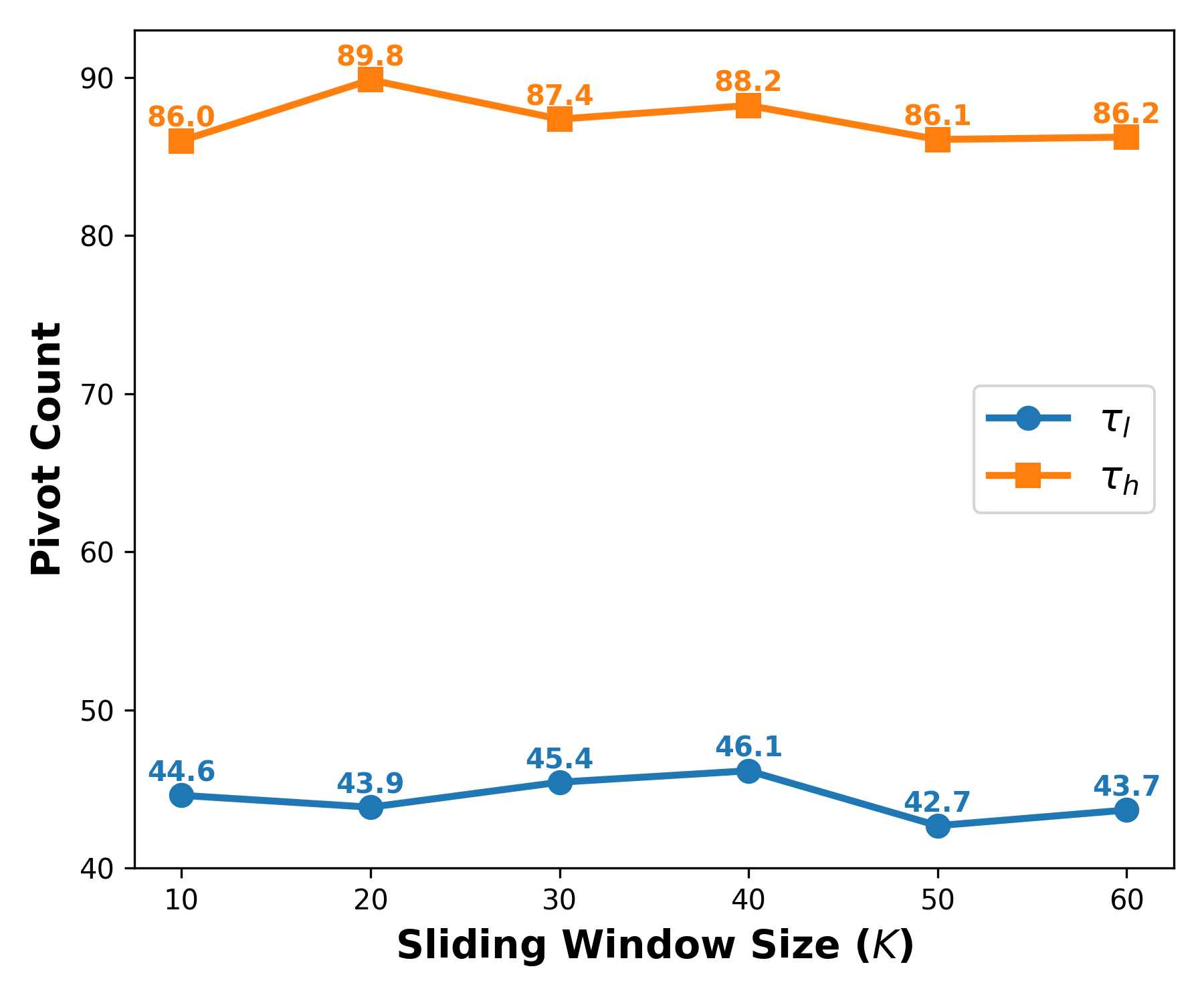}
        \caption{\small Ablation on sliding window size}
        \label{fig:slide_ablation}
    \end{subfigure}
    
    \caption{\small (a) Model performance consistently declines as the number of metacognitive pivots increases.
    (b) The effectiveness of pivot-based selection degrades when the future window is not properly constrained.
    (c) The thresholds $\tau_l$ and $\tau_h$ remains relatively stable across different sliding window sizes.}
\end{figure*}

\subsection{Pivot Count Inversely Signals Expected Correctness.}
To further probe the relationship between metacognitive pivots and expected correctness, we generate $8$ reasoning trajectories per question in DAPO-Math-14k and compute the pass rate alongside the average pivot count.
As shown in Figure~\ref{fig:pivot_acc}, higher pivot count consistently correlates with lower pass rates, demonstrating that pivot count quantifies reasoning uncertainty and reliably reflects expected correctness. 
Moreover, samples with pivot count $|\bm{p}|<\tau_l$ exhibit pass rates above $75\%$, suggesting they are well-mastered; 
while those with $|\bm{p}|>\!\tau_h$ fall below $25\%$, reflecting high error rates and substantial noise risk under self-supervision.
These results validate that the dynamically derived thresholds $\tau_l$ and $\tau_h$ effectively support our three-way data triage.

\begin{table*}[!t]
\centering
\caption{\small In-domain (ID) and out-of-domain (OOD) performance using Deepseek-R1-Distill-Qwen-1.5B. \textbf{Bold} denotes the best results and $\dagger$ denotes methods requiring multiple stochastic inferences.}
\label{tab:results_qwen2}
\setlength{\tabcolsep}{3pt}  
\renewcommand{\arraystretch}{1.2} 
\resizebox{\textwidth}{!}{%
\begin{tabular}{lcccccc|cccc}
\toprule
\multirow{2}{*}{\textbf{Methods}} & \multicolumn{6}{c}{\textbf{In-Domain Performance}} & \multicolumn{4}{c}{\textbf{Out-of-Domain Performance}} \\
\cmidrule(lr){2-7} \cmidrule(lr){8-11}
 & \textbf{AIME 24/25} & \textbf{AMC} & \textbf{MATH-500} & \textbf{Minerva} & \textbf{Olympiad} & \textbf{Avg.} & \textbf{ARC-c} & \textbf{GPQA}$^{*}$ & \textbf{MMLU-Pro} & \textbf{Avg.} \\
\midrule
Random & 26.1/22.1 & 61.3 & 83.0 & 32.0 & 46.7 & 45.2 & 36.1 & 27.7 & 29.2 & 31.0 \\
Consistency$^\dagger$ & 29.0/22.3 & 63.4 & \textbf{84.9} & \textbf{33.8} & 45.2 & 46.4 & 40.8 & 27.8 & \textbf{31.8} & 33.5 \\
CoE & 28.2/23.9 & 65.0 & 84.8 & 32.3 & 46.8 & 46.8 & 36.7 & 28.8 & 29.6 & 31.7 \\
CoT-Kinetics & 26.7/22.7 & 64.8 & 82.6 & 29.7 & 45.7 & 45.3 & 32.3 & 27.7 & 28.9 & 29.6 \\
Entropy & 28.3/22.2 & 64.8 & 84.1 & 33.3 & 46.3 & 46.5 & 41.7 & \textbf{29.4} & 31.2 & 34.1 \\
Self-Certainty & 28.9/23.0 & 62.6 & 82.8 & 29.3 & 45.8 & 45.4 & 32.7 & 26.1 & 28.0 & 28.9 \\
\rowcolor{Gray}\textbf{PivotTrace (ours)} & \textbf{29.5}/\textbf{24.5} & \textbf{65.1} & 83.8 & 32.4 & \textbf{47.0} & \textbf{47.1} & \textbf{44.7} & 29.0 & 31.6 & \textbf{35.1} \\
\bottomrule
\end{tabular}%
}
\end{table*}

\subsection{Hyperparameter Sensitivity Analysis}
We conduct a comprehensive sensitivity analysis to evaluate how key design choices and hyperparameters in PivotTrace affect pivot token detection, data selection quality, and model performance.

\paragraph{Future Window Bounds ($d_{\text{min}}$, $d_{\text{max}}$).}
We ablate four configurations of the future window: ($0$, $1000$), ($20$, $1000$), ($0$, $100$), and ($20$, $100$), which respectively correspond to no bounds, a lower bound only, an upper bound only, and both bounds applied. 
As shown in Figure~\ref{fig:future_window}, the ratio of unlabeled to labeled consensus accuracy is highly sensitive to these choices. 
The unbounded setting ($0$, $1000$) achieves the lowest ratio of $2.57$. 
Imposing only a lower bound ($20$, $1000$) or only an upper bound ($0$, $100$) yields moderate improvements with ratios of $5.43$ and $3.89$, respectively. 
In contrast, the jointly bounded window ($20$, $100$) attains a substantially higher ratio of $7.56$. 

This demonstrates that effective pivot identification requires simultaneous control over both ends of the future window: without a minimum offset, the attention mechanism captures predominantly short-range patterns rather than long-range signals; without a maximum horizon, attention is diluted over irrelevant future tokens. 
Only the balanced interval ($20$, $100$) enables robust identification of pivot tokens for active data selection.

\begin{figure*}[t]
    \centering
    \begin{subfigure}[t]{0.32\textwidth}
        \centering
        \includegraphics[width=\textwidth]{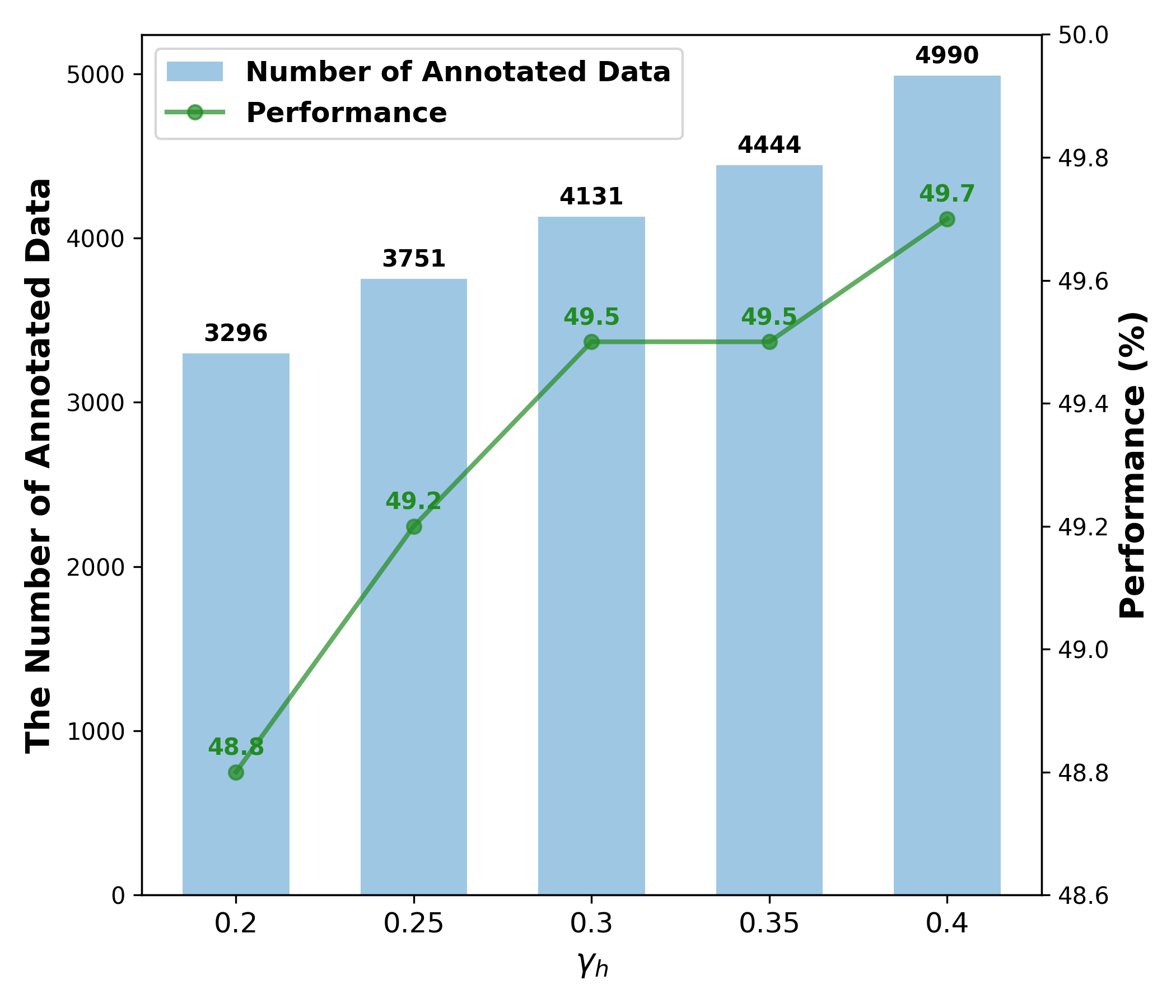}
        \caption{\small Ablation on threshold $\gamma_h$}
        \label{fig:gamma_h_ablation}
    \end{subfigure}
    \hfill
    \begin{subfigure}[t]{0.32\textwidth}
        \centering
        \includegraphics[width=\textwidth]{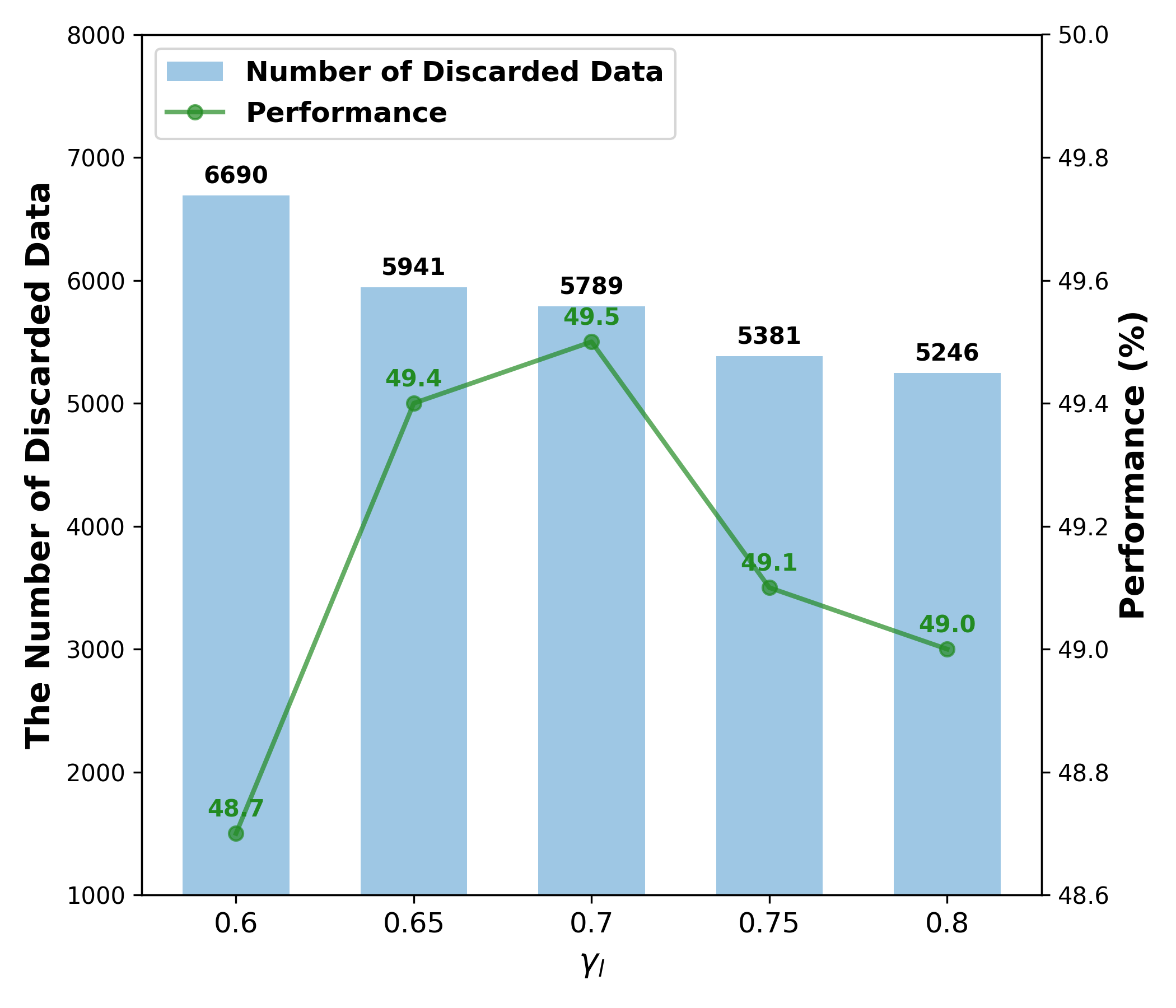}
        \caption{\small Ablation on threshold $\gamma_l$}
        \label{fig:gamma_l_ablation}
    \end{subfigure}
    \hfill
    \begin{subfigure}[t]{0.32\textwidth}
        \centering
        \includegraphics[width=\textwidth]{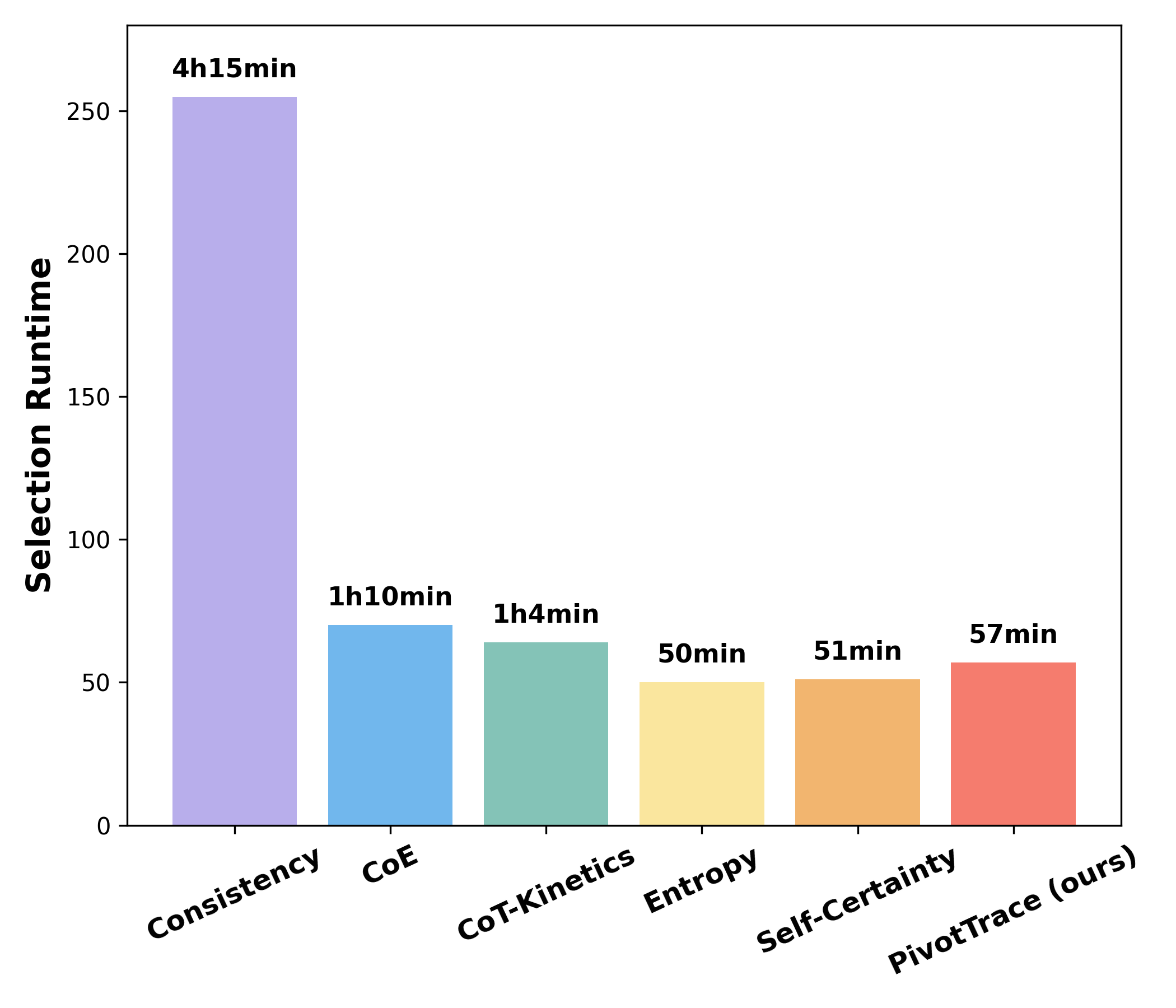}
        \caption{\small Comparison of selection runtime}
        \label{fig:selection_runtime}
    \end{subfigure}
    
    \caption{\small (a) As $\gamma_h$ increases, both the number of annotated samples and model performance rise. 
    (b) As $\gamma_l$ increases, fewer samples are discarded (i.e., more are retained for training), yet performance peaks at $\gamma_l = 0.7$ and slightly declines thereafter. 
    (c) PivotTrace completes selection in just $57$ minutes, which is comparable to lightweight heuristics like entropy.}
\end{figure*}

\paragraph{Peak Detection Parameters ($\zeta$, $\psi$, $\Delta$).}
To evaluate the sensitivity of our peak detection module to its key hyperparameters, we conduct an ablation study by varying one parameter at a time while fixing the other two to their default values ($\zeta$: $95$th percentile, $\psi$: $5\%$ of dynamic range, $\Delta = 10$). 
For each configuration, we compute the Jaccard similarity between the resulting set of selected labeled samples and that obtained under the default setting, as a measure of stability in data selection.

The results in Table~\ref{tab:peak_ablation} show that our method is largely robust to moderate changes in all three parameters. 
Varying $\zeta$ from the $80$th to the $90$th percentile yields Jaccard similarities above $0.936$, indicating consistent pivot selection. 
Similarly, $\psi = 10\%$ achieves the highest similarity to the default configuration ($0.966$), but similarity drops notably at $\psi = 20\%$ ($0.812$), indicating that excessively high amplitude thresholds may discard pivot tokens.
In contrast, $\Delta$ exhibits remarkable insensitivity: even doubling $\Delta$ from $10$ to $20$ only reduces similarity slightly (from $0.951$ to $0.941$).
These results demonstrate that our peak detection mechanism is stable under reasonable hyperparameter perturbations.

\begin{table}[!t]
\centering
\small
\caption{\small Ablation on peak detection hyperparameters. Jaccard similarity measures the overlap between labeled sample sets selected under varied vs. default configurations ($\zeta$: 95th percentile, $\psi$: 5\% dynamic range, $\Delta = 10$). In each case, two hyperparameters remain fixed while the third is varied.}
\label{tab:peak_ablation}
\begin{tabular}{@{}ccc@{}}
\toprule
\multicolumn{3}{c}{\textbf{Jaccard Similarity w.r.t. Default}} \\
\midrule
\textbf{Varying $\zeta$} & \textbf{Varying $\psi$} & \textbf{Varying $\Delta$} \\
(fixed: $\psi=5\%$, $\Delta=10$) & (fixed: $\zeta=95^{\text{th}}$, $\Delta=10$) & (fixed: $\zeta=95^{\text{th}}$, $\psi=5\%$) \\
\midrule
\begin{tabular}{@{}lc@{}}
$\zeta = 90^{\text{th}}$ & 0.947 \\
$\zeta = 85^{\text{th}}$ & 0.940 \\
$\zeta = 80^{\text{th}}$ & 0.936 \\
\end{tabular} &
\begin{tabular}{@{}lc@{}}
$\psi = 10\%$ & 0.966 \\
$\psi = 15\%$ & 0.934 \\
$\psi = 20\%$ & 0.812 \\
\end{tabular} &
\begin{tabular}{@{}lc@{}}
$\Delta = 5$  & 0.951 \\
$\Delta = 15$ & 0.951 \\
$\Delta = 20$ & 0.941 \\
\end{tabular} \\
\bottomrule
\end{tabular}
\end{table}

\begin{figure*}[ht]
    \centering
    \includegraphics[width=0.7\textwidth]{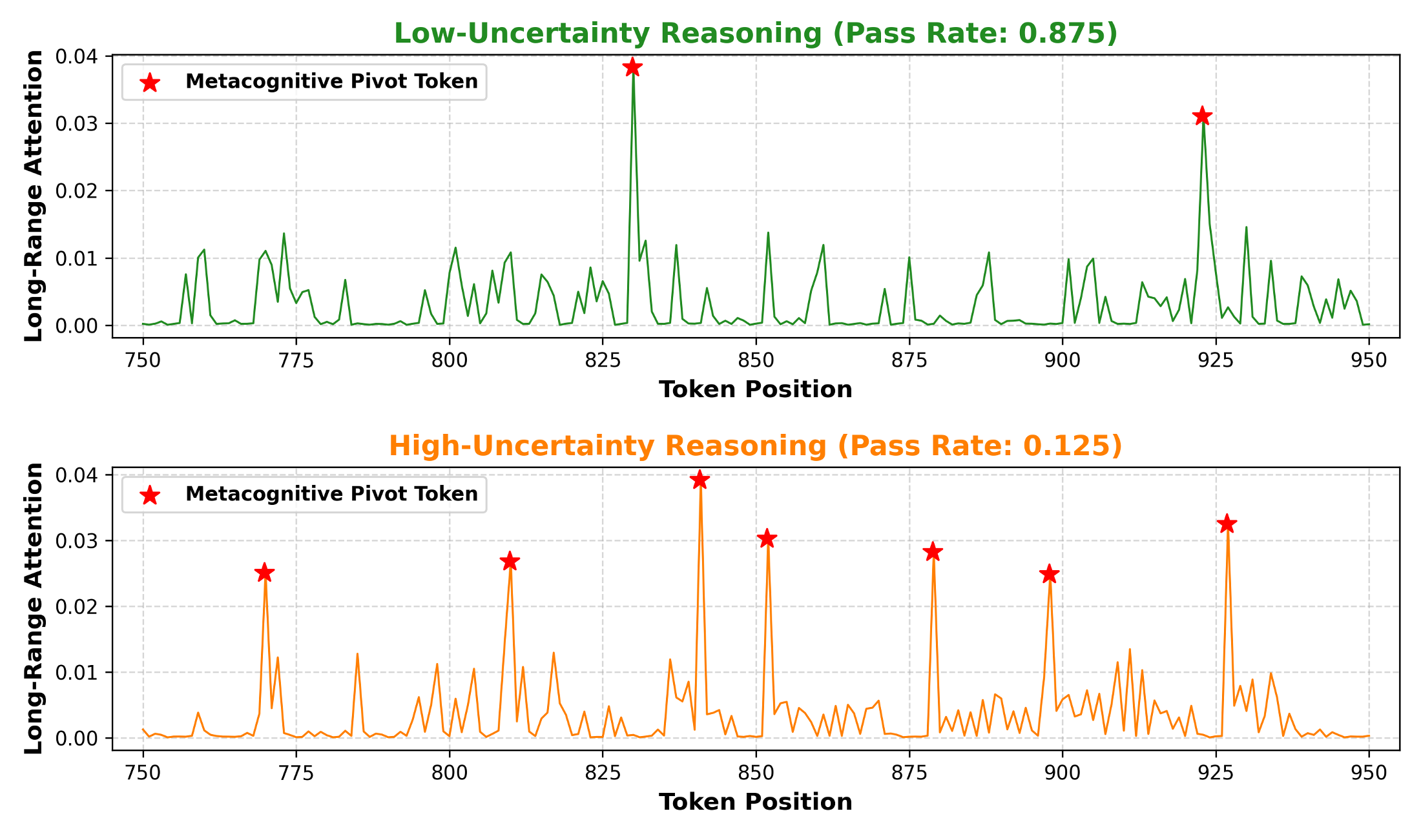}
    \caption{\small Long-range attention received per token in \textcolor[HTML]{228B22}{low-uncertainty} and \textcolor[HTML]{FF7F00}{high-uncertainty} reasoning segments (same length), with peaks (marked by \textcolor{red}{red stars}) indicating \emph{metacognitive pivots}.}
    \label{fig:length}
    \vskip -0.08in
\end{figure*}

\paragraph{Sliding Window Size ($K$).}
To assess the sensitivity of our method to the sliding window size, we evaluate configurations with $K \in \{10, 20, 30, 40, 50, 60\}$. As shown in Figure~\ref{fig:slide_ablation}, the resulting lower threshold $\tau_l$ and upper threshold $\tau_h$ remain stable across all settings: $\tau_l$ varies only between $42.7$ and $46.1$, while $\tau_h$ ranges from $86.0$ to $89.8$. The narrow variation indicates that our dynamic thresholding mechanism is robust to the choice of window size.

\paragraph{Accuracy Thresholds ($\gamma_l$, $\gamma_h$).}
The hyperparameter $\gamma_h$ determines which samples are selected for annotation. As shown in Figure~\ref{fig:gamma_h_ablation}, both the number of annotated samples and model performance increase monotonically with $\gamma_h$. 
Notably, when $\gamma_h < 0.3$, performance improves rapidly with additional annotations, suggesting that many informative, high-uncertainty samples remain unlabeled. 
In contrast, for $\gamma_h > 0.3$, the marginal gain in performance diminishes significantly despite a continued rise in annotation cost, indicating that the additional annotations provide limited utility.

Conversely, $\gamma_l$ controls the filtering of low-quality samples. 
Figure~\ref{fig:gamma_l_ablation} reveals that model performance peaks at $\gamma_l = 0.7$. 
When $\gamma_l < 0.7$, performance degrades because overly lenient filtering discards a significant number of high-utility samples that would otherwise benefit learning. Conversely, when $\gamma_l > 0.7$, performance slightly declines, likely due to the inclusion of low-utility samples which not only harm model accuracy by diluting the effective signal but also reduce training efficiency.

Based on these observations, we adopt $\gamma_h = 0.3$ and $\gamma_l = 0.7$ as our default settings, striking a balance between annotation efficiency, training efficiency, and overall performance.

\paragraph{Probing size $N$.}
To investigate the sensitivity of our method to the scale of the probing data, we conduct an ablation study on the probing set size $N$. As illustrated in Figure~\ref{fig:prob_size}, when $N$ varies from $50$ to $200$ (with an interval of $50$), the thresholds $\tau_l$ and $\tau_h$ generated by our automated calibration remain remarkably stable. Specifically, $\tau_l$ fluctuates within a narrow margin of [$43.2$, $46.2$], while $\tau_h$ stays consistently between $87.3$ and $90.2$. In both cases, the total variation does not exceed $3.0$, demonstrating that our dynamic calibration approach is robust and does not necessitate large probing sets to achieve reliable performance.

\subsection{Comparison of Runtime for Sample Selection}
\label{app:runtime}
The dominant cost in the sample selection process stems from generating model responses. 
Among all methods, \textit{Consistency} incurs the highest computational overhead, as it requires sampling multiple reasoning paths per question to estimate answer consistency, resulting in a selection runtime that is nearly \textbf{4}$\times$ longer than that of other baselines ($255$ vs. $\sim 50$--$70$ minutes). 
In contrast, the remaining approaches require only a single model response per question, from which uncertainty scores can be computed directly, resulting in substantially lower time costs.
Notably, \textit{PivotTrace} achieves a selection runtime of just $57$ minutes, which is faster than representation-based methods (\textit{CoE} and \textit{CoT-Kinetics}) that rely on expensive hidden-state computations, and closely approaches the efficiency of lightweight output-probability-based heuristics (\textit{Entropy} and \textit{Self-Certainty}). 
This demonstrates that \textit{PivotTrace} strikes an effective balance between selection sophistication and computational practicality, enabling scalable active labeling without sacrificing speed.

\begin{table*}[!t]
\centering
\caption{\small Performance comparison with length-based selection under the same annotation budget.}
\label{tab:compare_length}
\setlength{\tabcolsep}{5pt}  
\renewcommand{\arraystretch}{1.2} 
\resizebox{\textwidth}{!}{%
\begin{tabular}{lcccccc|cccc}
\toprule
\multirow{2}{*}{\textbf{Methods}} & \multicolumn{6}{c}{\textbf{In-Domain Performance}} & \multicolumn{4}{c}{\textbf{Out-of-Domain Performance}} \\
\cmidrule(lr){2-7} \cmidrule(lr){8-11}
 & \textbf{AIME 24/25} & \textbf{AMC} & \textbf{MATH-500} & \textbf{Minerva} & \textbf{Olympiad} & \textbf{Avg.} & \textbf{ARC-c} & \textbf{GPQA}$^{*}$ & \textbf{MMLU-Pro} & \textbf{Avg.} \\
\midrule
Length & 24.1/22.4 & 59.4 & 85.8 & 43.1 & 48.7 & 47.2 & 80.7 & 36.0 & 62.3 & 59.7 \\
\midrule
\rowcolor{Gray}\textbf{PivotTrace (ours)} & \textbf{27.2}/\textbf{25.2} & \textbf{62.4} & \textbf{87.3} & \textbf{44.6} & \textbf{50.1} & \textbf{49.5} & \textbf{93.0} & \textbf{38.3} & \textbf{63.5} & \textbf{64.9} \\
\bottomrule
\end{tabular}%
}
\end{table*}

\subsection{Analysis on Potential Length Bias}
\label{app:len_bias}
A potential concern regarding PivotTrace is that it might simply exhibit a bias toward selecting samples with longer response lengths. We clarify that a higher pivot count does not inherently equate to greater verbosity. While a long response may represent a successful, single-pass derivation, a shorter one might involve multiple intensive trial-and-error cycles. As demonstrated in Figure~\ref{fig:length}, when examining reasoning segments of the same length, high-uncertainty reasoning generates more pivot tokens than low-uncertainty reasoning. This confirms that pivot frequency is a reflection of the model's internal uncertainty and cognitive shifts rather than a mere byproduct of response length.

To further address this concern, we conducted a quantitative correlation analysis and a baseline comparison. The Pearson correlation between pivot count and response length is $r = 0.54$, which indicates only a moderate relationship and proves that the two are not strongly coupled. More importantly, as shown in Table~\ref{tab:compare_length}, when using response length itself as a selection metric, the performance is significantly inferior to PivotTrace, trailing by $2.3\%$ in ID Average and $5.2\%$ in OOD Average. These results provide strong evidence that response length alone is an ineffective selection metric and that the effectiveness of PivotTrace stems from its ability to capture meaningful reasoning transitions that length-based heuristics overlook. Consequently, we conclude that PivotTrace identifies high-quality reasoning based on its structural complexity rather than simple length.

\subsection{Comparison with a Lexical Cue Baseline}
\label{app:regex_baseline}

To examine whether explicit lexical cues alone provide an effective proxy for reasoning uncertainty, we implement \textit{Regex-Count}, a baseline that counts occurrences of predefined hedge and retraction expressions, including ``wait,'' ``but,'' ``alternatively,'' and ``let me check,'' in each generated reasoning trace.
We use this count as the uncertainty score in the same data-selection pipeline, with higher counts indicating greater estimated uncertainty.
Following the comparison protocol in Section~\ref{sec:setup}, we keep the annotation and filtering budgets, training procedure, and evaluation settings identical to those of PivotTrace.
This comparison evaluates whether a simple text-based heuristic can achieve comparable selection quality without accessing attention maps.

\begin{table*}[!t]
\centering
\caption{\small Performance comparison with a lexical cue counting baseline under the same annotation budget.}
\label{tab:compare_regex}
\setlength{\tabcolsep}{5pt}
\renewcommand{\arraystretch}{1.2}
\resizebox{\textwidth}{!}{%
\begin{tabular}{lcccccc|cccc}
\toprule
\multirow{2}{*}{\textbf{Methods}} & \multicolumn{6}{c}{\textbf{In-Domain Performance}} & \multicolumn{4}{c}{\textbf{Out-of-Domain Performance}} \\
\cmidrule(lr){2-7} \cmidrule(lr){8-11}
 & \textbf{AIME 24/25} & \textbf{AMC} & \textbf{MATH-500} & \textbf{Minerva} & \textbf{Olympiad} & \textbf{Avg.} & \textbf{ARC-c} & \textbf{GPQA}$^{*}$ & \textbf{MMLU-Pro} & \textbf{Avg.} \\
\midrule
Regex-Count & 24.5/21.9 & 59.9 & 86.3 & 42.8 & 50.0 & 47.6 & 92.2 & 36.0 & 61.9 & 63.4 \\
\midrule
\rowcolor{Gray}\textbf{PivotTrace (ours)} & \textbf{27.2}/\textbf{25.2} & \textbf{62.4} & \textbf{87.3} & \textbf{44.6} & \textbf{50.1} & \textbf{49.5} & \textbf{93.0} & \textbf{38.3} & \textbf{63.5} & \textbf{64.9} \\
\bottomrule
\end{tabular}%
}
\end{table*}

As shown in Table~\ref{tab:compare_regex}, PivotTrace achieves higher scores than Regex-Count on all nine benchmarks, improving the average ID and OOD accuracy by $1.9\%$ and $1.5\%$, respectively.
A limitation of fixed lexical cues is that their presence does not necessarily indicate a reasoning-state transition, while their absence does not rule one out.
For example, ``$x$ is positive but less than $1$'' contains ``but'' without revising an earlier reasoning step, and ``alternatively'' may introduce an equivalent derivation rather than express uncertainty.
Conversely, a revision such as ``This result violates the constraint. Re-derive the solution from the boundary condition'' need not contain any of these example expressions.
Regex-Count matches predefined expressions without considering their surrounding context, whereas PivotTrace scores positions using the sustained attention they receive from subsequent reasoning tokens.
The performance gap supports the utility of this attention-based signal over fixed lexical cue counts for data selection in our experimental setting.

\subsection{Generalization to Data-Scarce Specialized Domains}
\begin{wraptable}{r}{0.58\textwidth}
    \vspace{-12pt}
    \centering
    \small
    \caption{\small Performance comparison on data-scarce domains.}
    \label{tab:open_domain}
    \setlength{\tabcolsep}{3pt}
    \begin{tabular}{lcccc}
    \toprule
    \textbf{Methods} & \textbf{FinQA} & \textbf{HiTab} & \textbf{MultiHiertt} & \textbf{Avg.} \\
    \midrule
    Random & 19.4 & 58.2 & 23.2 & 33.6 \\
    Consistency & 19.5 & 61.3 & 23.8 & 34.9 \\
    CoE & 20.1 & 60.3 & 23.5 & 34.6 \\
    CoT-Kinetics & 18.7 & 61.9 & 23.2 & 34.6 \\
    Entropy & 19.2 & 61.5 & 24.1 & 34.9 \\
    Self-Certainty & 19.3 & 59.5 & 23.2 & 34.0 \\
    \midrule
    \rowcolor{Gray}\textbf{PivotTrace (ours)} & \textbf{20.4} & \textbf{64.8} & \textbf{25.6} & \textbf{36.9} \\
    \bottomrule
    \end{tabular}
\end{wraptable}

To further evaluate the robustness of PivotTrace and address its performance in scenarios where ground-truth labels are genuinely scarce, we extend our evaluation to complex tabular reasoning tasks within the finance and healthcare sectors. 
Unlike general math benchmarks, datasets such as FinQA~\citep{chen2021finqa}, HiTab~\citep{cheng2022hitab}, and MultiHiertt~\citep{zhao2022multihiertt} represent specialized domains where expert-annotated rationales are difficult and costly to obtain. 
As shown in Table~\ref{tab:open_domain}, our method consistently outperforms all competitive baselines, achieving an average improvement of $+2.0\%$ across these benchmarks. 
These results demonstrate that PivotTrace effectively reduces the reliance on dense ground-truth supervision by successfully leveraging its internal verification mechanism, proving its utility not only in ``gold-standard'' mathematical reasoning but also in open-domain, data-constrained professional environments.

\subsection{Case Study}
To better illustrate the relationship between pivot count and question difficulty, we conducted a qualitative case study (see Figure~\ref{fig:case_study_1} and Figure~\ref{fig:case_study_2}). 
Our analysis reveals a strong correlation: questions with the highest pivot count consistently involve advanced mathematical concepts, multi-step reasoning, and abstract problem-solving strategies, often requiring integration of knowledge across domains. 
These problems are particularly challenging for the model, yet mastering them yields the greatest gains in its reasoning capabilities, making them high-value targets for training. 
In contrast, questions with the lowest pivot count are typically straightforward, relying on basic operations, direct application of formulas, or simple logical inference with minimal cognitive overhead.
Such problems are already well within the model’s competence, and applying RLVR training to them offers little to no benefit.
This qualitative observation supports our hypothesis that pivot count serves as a meaningful proxy for the intrinsic difficulty of a problem from the model’s perspective.

\begin{figure*}[t]
    \centering
    \begin{subfigure}[t]{0.32\textwidth}
        \centering
        \includegraphics[width=\textwidth]{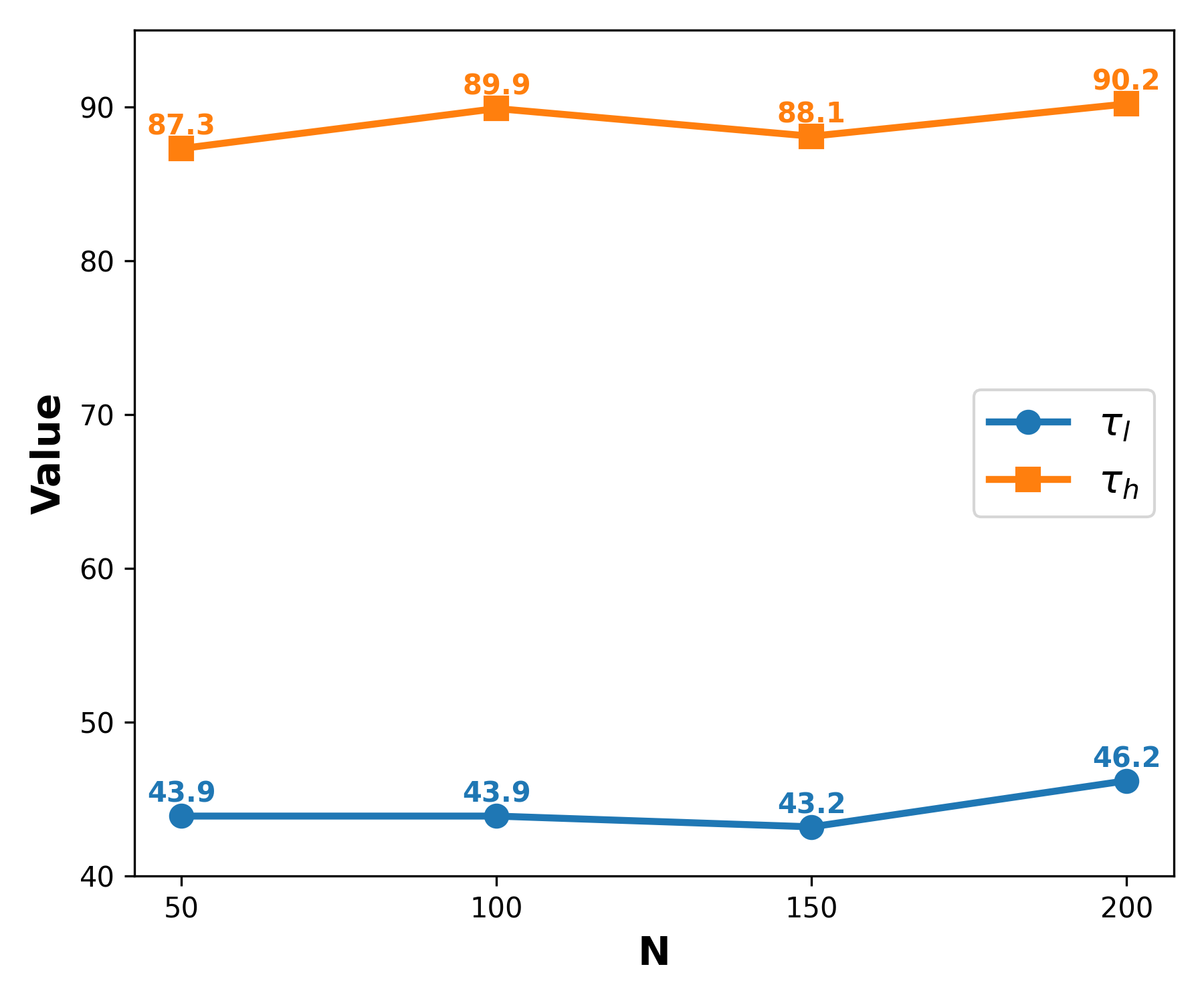}
        \caption{\small Ablation on probing size $N$}
        \label{fig:prob_size}
    \end{subfigure}
    \hfill
    \begin{subfigure}[t]{0.356\textwidth}
        \centering
        \includegraphics[width=\textwidth]{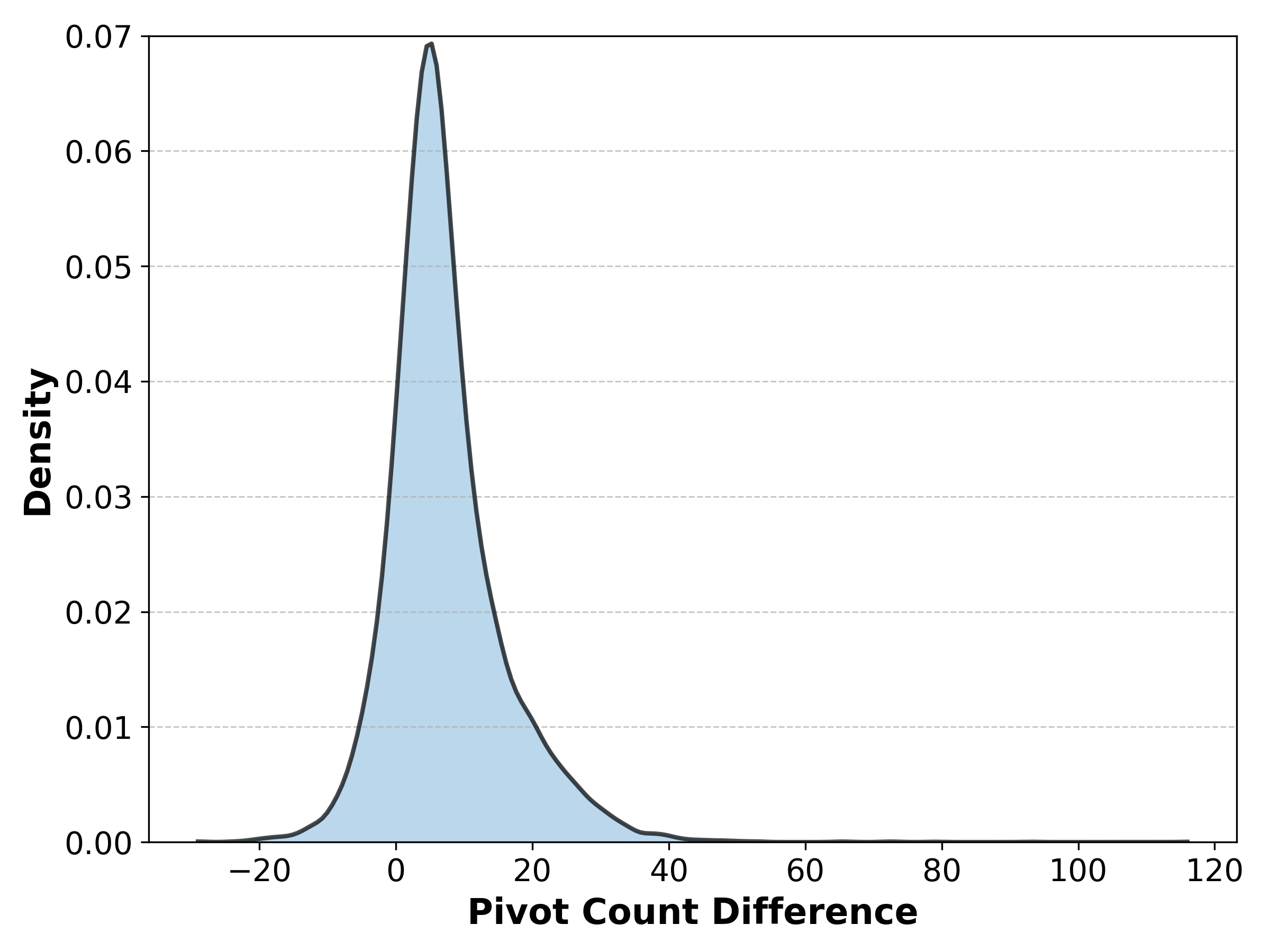}
        \caption{\small Distribution of pivot count difference}
        \label{fig:pivot_diff}
    \end{subfigure}
    \hfill
    \begin{subfigure}[t]{0.27\textwidth}
        \centering
        \includegraphics[width=\textwidth]{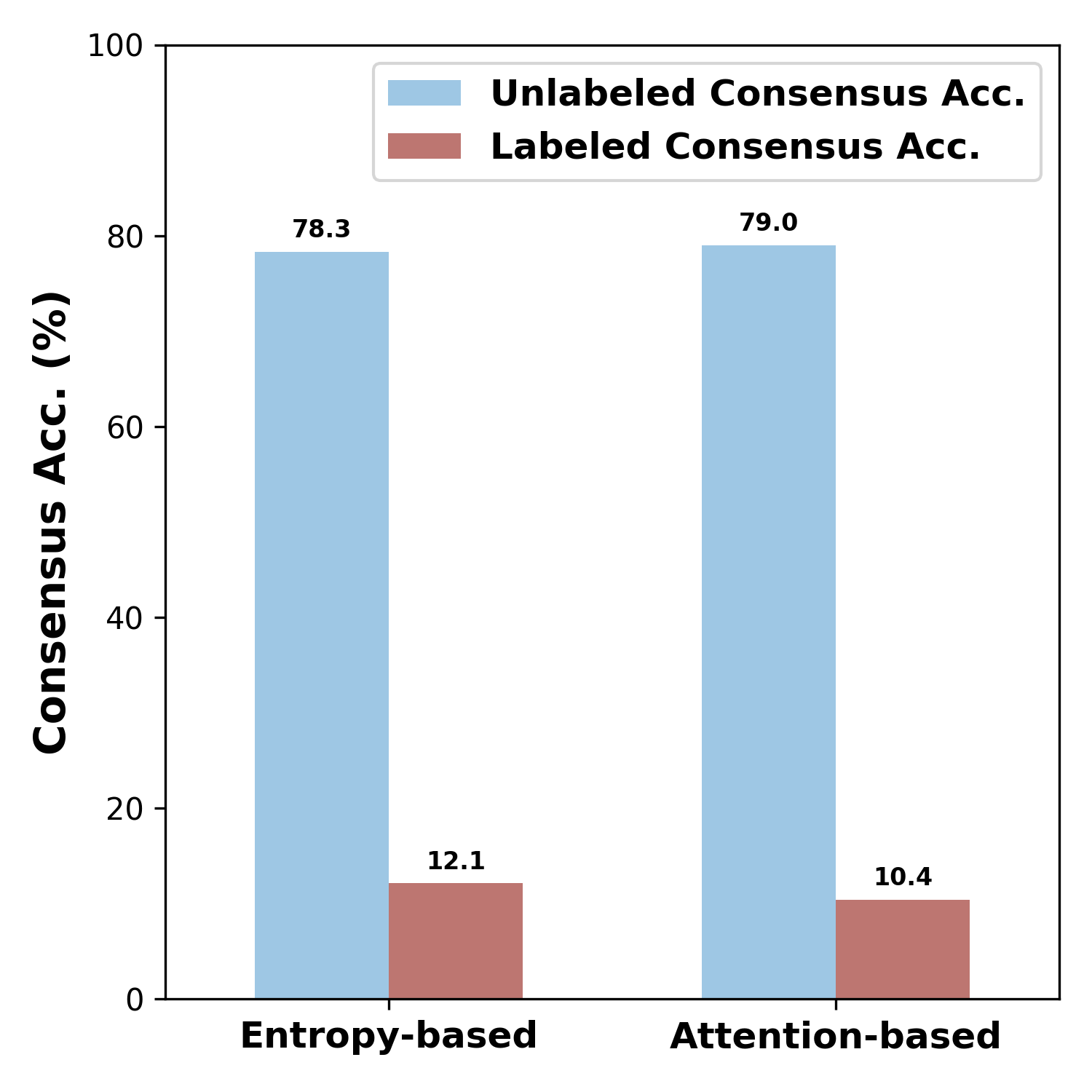}
        \caption{\small Key token identification}
        \label{fig:compare_entropy}
    \end{subfigure}
    
    \caption{\small (a) The dynamic thresholds $\tau_l$ and $\tau_h$ remain stable across varying probing size $N$.
    (b) Distribution of the difference in pivot token counts obtained by attention-based and entropy-based key token identification methods.
    (c) Comparison of data selection effectiveness when using attention-based versus entropy-based key token identification to estimate model uncertainty.}
\end{figure*}

\section{Additional Analysis and Discussion}
\subsection{Tension between Annotation and Training Efficiency}
\label{app:train_label}
There exists an inherent tension between annotation efficiency and training efficiency.
Maximizing annotation efficiency entails focusing human labeling efforts on the model’s knowledge frontier—that is, samples for which the model exhibits high uncertainty or is most likely to err. 
Human annotation on such samples provides the highest marginal utility, as they lie beyond the model’s current capability and cannot be reliably supervised via self-generated signals. 
In contrast, samples that the model can handle confidently contribute little additional value when labeled, making their annotation inefficient.

However, training exclusively on this highly informative yet limited set of labeled data leads to suboptimal performance. 
As shown in Figure~\ref{fig:trade_off}, when we train the RLVR model using only $29.3\%$ of the data that has been selectively annotated (i.e., those near the knowledge frontier), performance drops significantly to $47.6\%$. 
This degradation stems from two key factors: 
(1) the drastically reduced training sample size impairs the model’s generalization capacity; and 
(2) the intrinsic difficulty of these frontier samples results in sparse reward signals, since the model rarely generates correct rollouts, meaningful reinforcement feedback becomes scarce.

This observation necessitates incorporating unlabeled data to stabilize and improve training~\citep{xiao2024targeted,tang2024learning,yang2025nodereg,zhu2025large}. 
Indeed, when we perform the semi-supervised RLVR with the labeled set and the remaining unlabeled data, performance recovers to $48.9\%$, nearly matching the fully supervised baseline ($49.2\%$). 
Yet this comes at a steep cost in training efficiency: convergence requires $320$ training steps. 
The slowdown arises because the unlabeled pool, while containing useful self-supervision, also includes a large proportion of low-utility samples, i.e., instances that are too easy, providing negligible learning signal.

To reconcile annotation efficiency with training efficiency, we propose a three-way data triage strategy that dynamically partitions the dataset into: 
(i) high-value labeled samples (on the knowledge frontier), 
(ii) high-utility unlabeled samples (amenable to reliable self-supervision), and 
(iii) low-utility unlabeled samples (discarded to avoid inefficiency). 
This approach achieves $49.5\%$ performance in just $160$ training steps, surpassing even the fully supervised baseline while more than halving the training time. 
Therefore, our triage mechanism enables the simultaneous maximization of both annotation and training efficiency, resolving the fundamental trade-off.

\subsection{Key Token Identification for Uncertainty-Aware Data Selection}
A growing body of work~\citep{wang2025beyond,li2025attention,cheng2025reasoning,cui2025entropy,yang2026oprd} has focused on identifying \textit{key tokens} in LLM reasoning trajectories, i.e., tokens deemed critical to the model’s decision-making process. 
These studies argue that such tokens should receive higher credit during RLVR training to enhance reasoning capabilities. 
Most existing approaches rely on \textit{entropy-based} analysis: they posit that tokens with high entropy are pivotal, as they reflect moments of high uncertainty where the model’s reasoning direction is being determined.

In contrast, we interpret the number of key tokens as an indicator of overall reasoning difficulty for data selection. 
More key tokens suggest greater internal uncertainty and less confident reasoning. 
To capture this effectively, we propose identifying \textit{metacognitive pivots} through attention patterns, which reveal how the model self-monitors its reasoning process.

For comparison, we also implement an entropy-based baseline that detects peaks in the entropy trajectory across the generated reasoning chain and uses the count of such high-entropy tokens to guide data selection. 
Interestingly, as shown in Figure~\ref{fig:pivot_diff}, the distribution of the difference in pivot token counts between the attention-based and entropy-based methods is tightly concentrated near zero—most differences fall within the range of $0$ to $20$, indicating a strong alignment between the two strategies in identifying critical reasoning steps. 

As illustrated in Figure~\ref{fig:compare_entropy}, the entropy-based strategy also performs remarkably well: it yields unlabeled data with a consensus accuracy of $78.3\%$ and labeled data with $12.1\%$, resulting in a ratio of $6.47$. Nevertheless, our attention-based approach achieves even stronger performance: $79.0\%$ vs. $10.4\%$, with a higher ratio of $7.60$, demonstrating that attention provides a more discriminative signal for quantifying model uncertainty in the context of active data selection. 
Consequently, we adopt the attention-based pivot token identification as the core mechanism in PivotTrace.

\begin{table*}[!t]
\centering
\caption{\small Performance comparison with CONST under the same annotation budget.}
\label{tab:compare_const}
\setlength{\tabcolsep}{5pt}  
\renewcommand{\arraystretch}{1.2} 
\resizebox{\textwidth}{!}{%
\begin{tabular}{lcccccc|cccc}
\toprule
\multirow{2}{*}{\textbf{Methods}} & \multicolumn{6}{c}{\textbf{In-Domain Performance}} & \multicolumn{4}{c}{\textbf{Out-of-Domain Performance}} \\
\cmidrule(lr){2-7} \cmidrule(lr){8-11}
 & \textbf{AIME 24/25} & \textbf{AMC} & \textbf{MATH-500} & \textbf{Minerva} & \textbf{Olympiad} & \textbf{Avg.} & \textbf{ARC-c} & \textbf{GPQA}$^{*}$ & \textbf{MMLU-Pro} & \textbf{Avg.} \\
\midrule
CONST & 24.4/23.5 & 58.5 & 85.6 & 43.9 & 48.2 & 47.4 & 88.0 & 29.8 & 62.4 & 60.1 \\
\midrule
\rowcolor{Gray}\textbf{PivotTrace (ours)} & \textbf{27.2}/\textbf{25.2} & \textbf{62.4} & \textbf{87.3} & \textbf{44.6} & \textbf{50.1} & \textbf{49.5} & \textbf{93.0} & \textbf{38.3} & \textbf{63.5} & \textbf{64.9} \\
\bottomrule
\end{tabular}%
}
\end{table*}

\subsection{Discussion on Related Attention-based Interpretability Works}
The exploration of attention patterns in LRMs has recently gained traction, particularly in the context of deciphering the internal mechanics of reasoning. 
Recent studies~\citep{bogdan2025thought,li2025attention,liu2025attention} have provided pioneering insights into how attention scores signify critical cognitive steps during reasoning. 
Specifically, \citet{liu2025attention} identifies ``massive attention values'' as indicators of pivotal reasoning behaviors (e.g., verification or planning) and utilizes these signals to guide tree-based branch exploration in reinforcement learning. 
Similarly, \citet{li2025attention} characterizes a rhythmic ``two-beat'' attention pattern where models alternate between local pre-planning and global anchoring, subsequently using these insights for fine-grained policy optimization. 
On a more macro scale, \citet{bogdan2025thought} elevates the analysis to the sentence level, identifying "receiver heads" that consistently focus on causal "anchors" such as plan generation and uncertainty management to validate the faithfulness of long-CoT reasoning.

While these works provide contributions to reasoning interpretability and on-policy optimization, their methods and derived metrics are fundamentally unsuited for the task of offline data selection. 
First, the primary objective of these works is to optimize the model’s behavior during the generation or RL process (e.g., deciding where to branch in a search tree or how to redistribute rewards). 
They assume the input data is already given and focus on maximizing the utility of the reasoning trace itself. 
In contrast, our work addresses the data efficiency bottleneck at a prior stage: identifying which unlabeled samples from a massive pool possess the highest learning value for subsequent training. 
Consequently, the signals used in interpretability—such as sequence-level causal links or simplistic attention differentials—lack the necessary robustness to serve as a reliable proxy for sample-level uncertainty across diverse, unseen datasets.

Furthermore, the metrics proposed in these studies are primarily optimized for specific architectural behaviors or granularities that do not naturally translate to stable data-ranking criteria. 
For instance, sentence-level causal analysis or fixed ``pre-plan'' rhythm detection relies on structural assumptions about the reasoning chain that may not hold across diverse scenarios. 
Unlike these interpretability-focused metrics, our proposed PivotTrace framework treats attention peaks as a statistically robust proxy for model uncertainty. 
To the best of our knowledge, we are the first to bridge the gap between intrinsic attention patterns and the data efficiency bottleneck, shifting the paradigm from ``identifying important reasoning steps'' to ``quantifying sample difficulty'' for strategic data selection. 

\subsection{Comparison with Concurrent Work: CONST}
Concurrently with our work, an independent study CONST~\citep{xiao2026sample} also addresses data efficiency in RLVR under the same “pick in the dark” setting, where no ground-truth answers are available during the sample selection phase.
However, there are two key differences:

First, CONST uses only the annotated subset for training and discards the rest of the unlabeled data. In contrast, PivotTrace is inherently semi-supervised: it not only selects informative samples for labeling but also leverages the remaining high-utility unlabeled data with reliable self-supervision signals, making far more effective use of the available data~\citep{yang2026venus,zhu2026geomin}.

Second, the sample selection of CONST requires generating $40$ model responses per question ($20$ for procedural volatility and $20$ for outcome volatility), which dominates the overall cost as we discuss in Appendix~\ref{app:runtime}. 
This makes its selection process $\sim40\times$ slower than ours, which needs just one response per sample. In practice, this overhead can rival or even exceed the cost of RLVR training itself, severely limiting its scalability to large unlabeled pools.

We reproduced CONST under the same annotation budget ($\sim4$k samples) for a fair comparison. 
As illustrated in Table~\ref{tab:compare_const}, PivotTrace outperforms CONST across all benchmarks, with ID and OOD average gains of $2.1\%$ and $4.8\%$, respectively. 
This not only clearly demonstrates the superiority of our method, but also provides empirical support once again for our argument in Appendix~\ref{app:train_label}: merely selecting “critical” labeled samples is insufficient, effectively leveraging unlabeled data is essential to maximize performance under limited annotation budgets.

\begin{figure}[t]
    \centering
    \includegraphics[width=\textwidth]{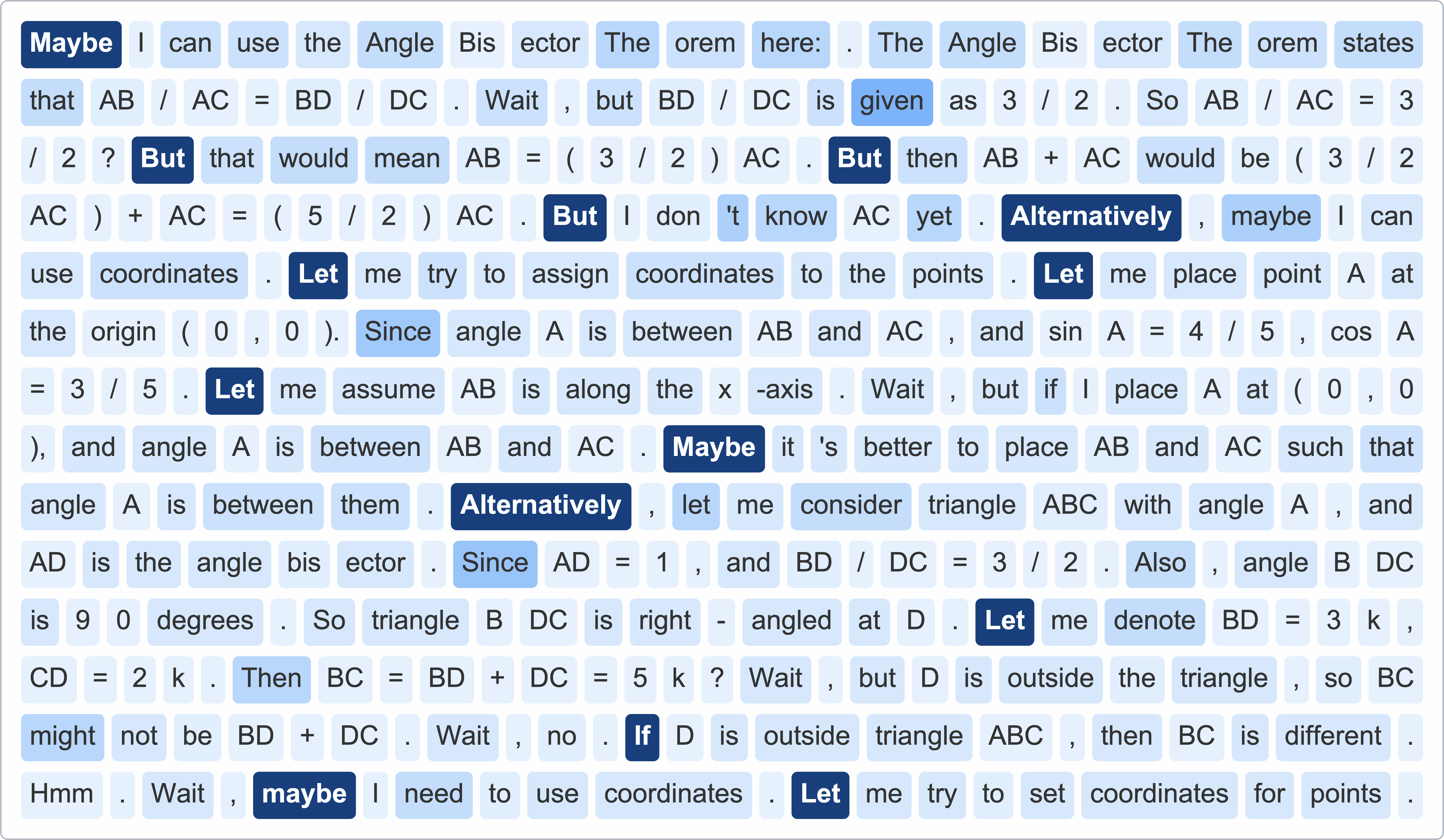}
    \caption{\small Visualization of metacognitive pivots in a reasoning segment. Token shading intensity represents the magnitude of received long-range attention. The most deeply shaded tokens indicate attention peaks, which correspond to the identified metacognitive pivots—critical junctures where the model performs significant informational transitions or self-correction during reasoning.}
    \label{fig:attn_vis}
\end{figure}

\subsection{Limitations}
\label{sec:limit}
Despite its effectiveness, our approach has several limitations. First, the reliance on full attention analysis may encounter memory bottlenecks when scaling to ultra-large models or extreme context lengths. Second, the optimal thresholds for uncertainty signals are task-dependent; while robust in general settings, they may require domain-specific recalibration for highly subjective or factually stringent applications~\citep{farquhar2024detecting,tang2025personality,tang2026group,yuan2026kardia}. Finally, while our current framework operates as a discrete pre-training selection phase, future work could evolve this into an integrated in-training screening or multi-round annotation paradigm to dynamically refine data quality throughout the model's evolution.

\subsection{Broader Impacts}
\label{sec:impact}
Our work primarily focuses on enhancing data efficiency, which carries minimal direct ethical risk while offering significant social benefits. By reducing the reliance on massive human-labeled datasets, we lower the financial and technical barriers to AI development, contributing to the democratization of technology. Furthermore, the reduction in training time directly translates to a smaller carbon footprint. While any automated selection process could potentially inherit or amplify sampling biases from the source pool, this is a technical challenge rather than a fundamental ethical flaw. We advocate for integrating diversity-preserving constraints alongside our method to ensure that the resulting efficiency gains do not come at the expense of model fairness.

\begin{figure}[t]
    \centering
    \includegraphics[width=\textwidth]{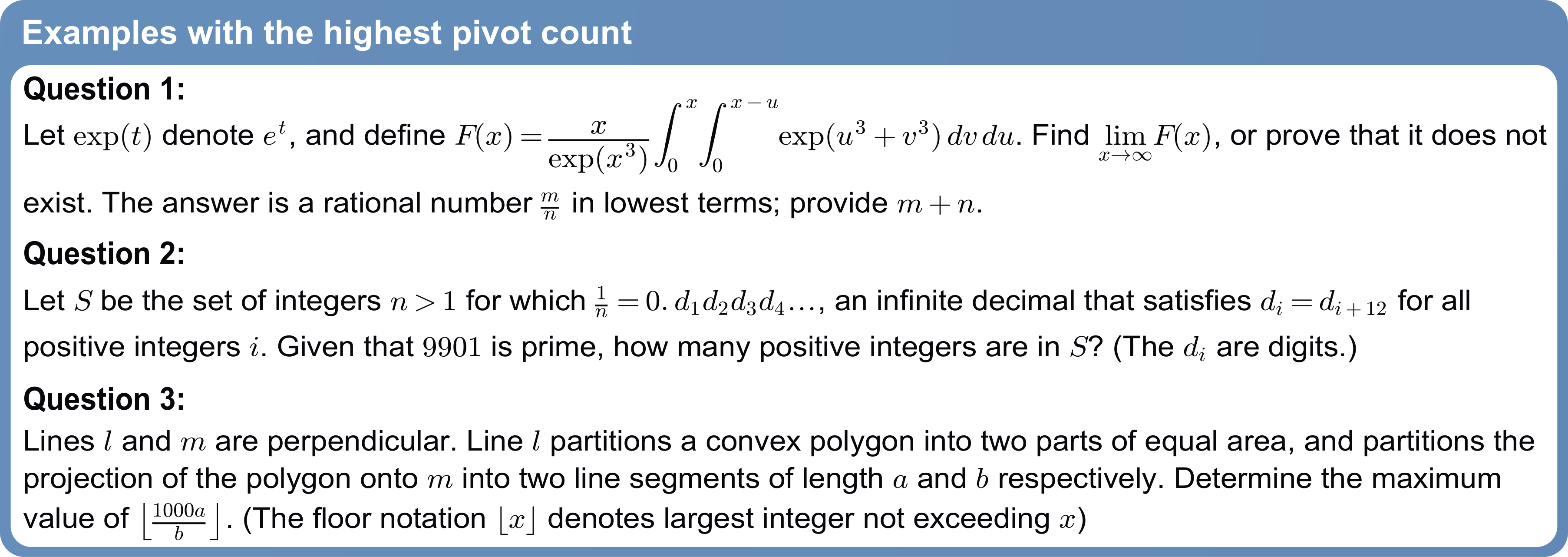}
    \caption{\small Example questions with the highest pivot count}
    \label{fig:case_study_1}
\end{figure}

\begin{figure}[t]
    \centering
    \includegraphics[width=\textwidth]{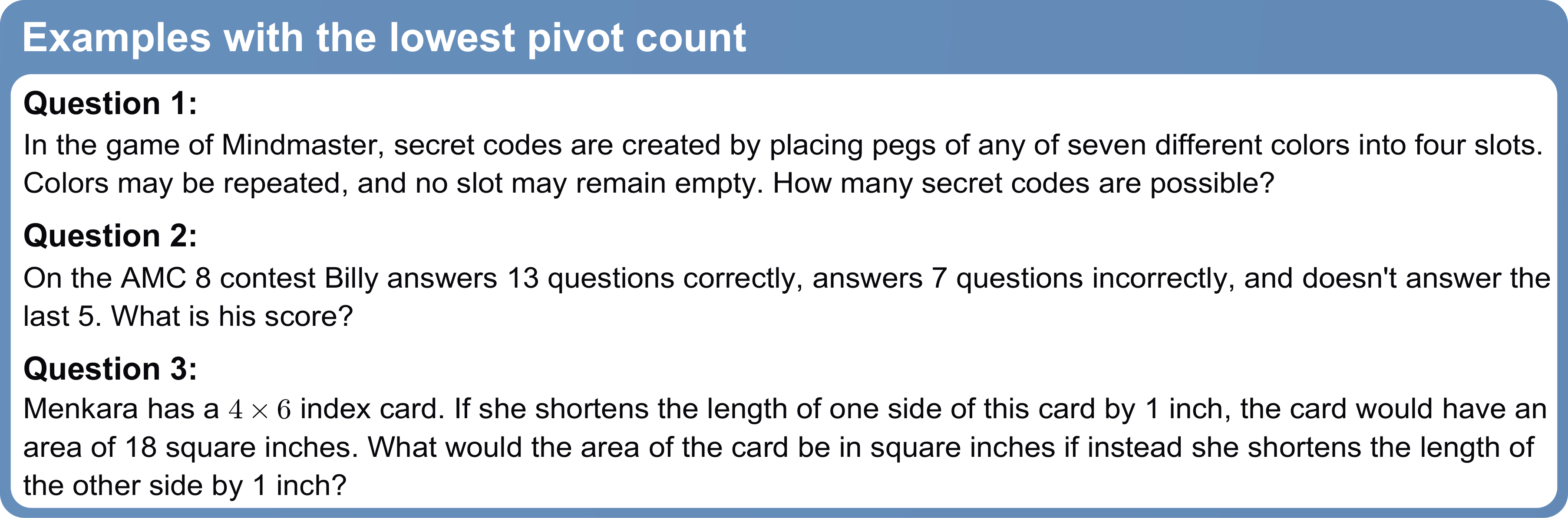}
    \caption{\small Example questions with the lowest pivot count}
    \label{fig:case_study_2}
\end{figure}

\FloatBarrier

\section{Pseudo Code}
\begin{algorithm}[H]
\caption{Three-Way Data Triage via PivotTrace}
\label{alg:pivot_rlvr}
\begin{algorithmic}[1]
\REQUIRE Question pool  $ \mathcal{Q} $ , base model  $ \theta_0 $ , future window bounds  $ d_{\min}, d_{\max} $ , long-range head size  $ k $ , peak thresholds  $ \zeta, \psi, \Delta $ , probing set size  $ N $ , sliding window size  $ K $ , accuracy levels  $ \gamma_l, \gamma_h $ 
\ENSURE Annotated set  $ \mathcal{D}_a $ , unlabeled set  $ \mathcal{D}_u $ , discard set $\mathcal{D}_d$

\FOR{each question  $ q $  in  $ \mathcal{Q} $ }
    \STATE Generate CoT response $y = (y_1, \dots, y_T)$ using base model $\theta_0$
    \STATE Perform forward pass on $(q, y)$ and extract attention maps $\{\mathbf{A}^{(h)}\}_{h=1}^H$ from all $H$ heads
    \FOR{ $ t = 1 $  to  $ T $ }
        \STATE Set future window  $ \mathcal{W}_t = \{ s \mid t + d_{\min} \leq s \leq \min(t + d_{\max}, T) \} $ 
        \FOR{ $ h = 1 $  to  $ H $ }
            \IF{ $ |\mathcal{W}_t| > 0 $ }
                \STATE Compute long-range attention  $ \alpha^{(h)}_t = \frac{1}{|\mathcal{W}_t|} \sum_{s \in \mathcal{W}_t} \mathbf{A}^{(h)}_{s,t} $ 
            \ELSE
                \STATE Set  $ \alpha^{(h)}_t = 0 $ 
            \ENDIF
        \ENDFOR
    \ENDFOR
\ENDFOR

\STATE Estimate head-level long-range score $\bar{\alpha}^{(h)}$ as the average of $\alpha^{(h)}_t$ over tokens in a small set of CoT responses
\STATE Select top-$ k $  heads  $ \mathcal{T}_k $  with highest  $ \bar{\alpha}^{(h)} $ 
\FOR{each  $ q $  in  $ \mathcal{Q} $ }
    \FOR{ $ t = 1 $  to  $ T $ }
        \STATE Aggregate long-range attention:  $ \alpha_t = \frac{1}{k} \sum_{h \in \mathcal{T}_k} \alpha^{(h)}_t $ 
    \ENDFOR
    \STATE Detect pivot positions  $ \bm{p} $  as local maxima in  $ (\alpha_1, \dots, \alpha_T) $  satisfying:
          (i)  $ \alpha_{p_i} \geq \zeta $ ,
          (ii) prominence of  $ p_i $  is at least  $ \psi $ ,
          (iii) distance between any two pivots is at least  $ \Delta $ 
    \STATE Record pivot count  $ c_q = |\bm{p}| $ 
\ENDFOR

\STATE Sort  $ \mathcal{Q} $  by increasing  $ c_q $  (higher count implies lower confidence)
\STATE Sample  $ N $  questions uniformly to form probing set  $ \mathcal{P} $ 
\FOR{each  $ q_j $  in  $ \mathcal{P} $ }
    \STATE Annotate ground-truth answer $a_j$ and generate  $ G $  model responses
    \STATE Compute empirical accuracy  $ \hat{\mu}_j $ 
\ENDFOR
\FOR{ $ i = 1 $  to  $ N - K + 1 $ }
    \STATE Let  $ \mathcal{S}_i $  be the  $ i $-th sliding window of size  $ K $  over  $ \mathcal{P} $ 
    \STATE Compute average accuracy  $ \bar{\mu}_i = \frac{1}{K} \sum_{q_j \in \mathcal{S}_i} \hat{\mu}_j $ 
\ENDFOR
\STATE Let $i_l = \min \{ i \mid \bar{\mu}_i < \gamma_l \}$ and $i_h = \min \{ i \mid \bar{\mu}_i < \gamma_h \}$
\STATE Set $\tau_l = \frac{1}{K} \sum_{q_j \in \mathcal{S}_{i_l}} c_{q_j}$ and $\tau_h = \frac{1}{K} \sum_{q_j \in \mathcal{S}_{i_h}} c_{q_j}$

\STATE Form annotation set  $ \mathcal{D}_a = \{ q \in \mathcal{Q} \mid c_q \geq \tau_h \} $ , unlabeled set  $ \mathcal{D}_u = \{ q \in \mathcal{Q} \mid \tau_l < c_q < \tau_h \} $ , and discard set  $ \mathcal{D}_d = \{ q \in \mathcal{Q} \mid c_q \leq \tau_l \} $ 
\STATE Annotate all questions in  $ \mathcal{D}_a $ 
\STATE \textbf{return} $\mathcal{D}_a, \mathcal{D}_u, \mathcal{D}_d$
\end{algorithmic}
\end{algorithm}

\FloatBarrier


\newpage
\input{checklist.tex}

\end{document}

%% file: checklist.tex
\section*{NeurIPS Paper Checklist}

\begin{enumerate}

\item {\bf Claims}
    \item[] Question: Do the main claims made in the abstract and introduction accurately reflect the paper's contributions and scope?
    \item[] Answer: \answerYes{} 
    \item[] Justification: The abstract and introduction clearly state that PivotTrace achieves dual efficiency in RLVR by tracing metacognitive pivots. The empirical results and theoretical analysis of uncertainty estimation fully support these claims.
    \item[] Guidelines:
    \begin{itemize}
        \item The answer \answerNA{} means that the abstract and introduction do not include the claims made in the paper.
        \item The abstract and/or introduction should clearly state the claims made, including the contributions made in the paper and important assumptions and limitations. A \answerNo{} or \answerNA{} answer to this question will not be perceived well by the reviewers. 
        \item The claims made should match theoretical and experimental results, and reflect how much the results can be expected to generalize to other settings. 
        \item It is fine to include aspirational goals as motivation as long as it is clear that these goals are not attained by the paper. 
    \end{itemize}

\item {\bf Limitations}
    \item[] Question: Does the paper discuss the limitations of the work performed by the authors?
    \item[] Answer: \answerYes{} 
    \item[] Justification:  We include a discussion of the limitations in the Appendix~\ref{sec:limit}.
    \item[] Guidelines:
    \begin{itemize}
        \item The answer \answerNA{} means that the paper has no limitation while the answer \answerNo{} means that the paper has limitations, but those are not discussed in the paper. 
        \item The authors are encouraged to create a separate ``Limitations'' section in their paper.
        \item The paper should point out any strong assumptions and how robust the results are to violations of these assumptions (e.g., independence assumptions, noiseless settings, model well-specification, asymptotic approximations only holding locally). The authors should reflect on how these assumptions might be violated in practice and what the implications would be.
        \item The authors should reflect on the scope of the claims made, e.g., if the approach was only tested on a few datasets or with a few runs. In general, empirical results often depend on implicit assumptions, which should be articulated.
        \item The authors should reflect on the factors that influence the performance of the approach. For example, a facial recognition algorithm may perform poorly when image resolution is low or images are taken in low lighting. Or a speech-to-text system might not be used reliably to provide closed captions for online lectures because it fails to handle technical jargon.
        \item The authors should discuss the computational efficiency of the proposed algorithms and how they scale with dataset size.
        \item If applicable, the authors should discuss possible limitations of their approach to address problems of privacy and fairness.
        \item While the authors might fear that complete honesty about limitations might be used by reviewers as grounds for rejection, a worse outcome might be that reviewers discover limitations that aren't acknowledged in the paper. The authors should use their best judgment and recognize that individual actions in favor of transparency play an important role in developing norms that preserve the integrity of the community. Reviewers will be specifically instructed to not penalize honesty concerning limitations.
    \end{itemize}

\item {\bf Theory assumptions and proofs}
    \item[] Question: For each theoretical result, does the paper provide the full set of assumptions and a complete (and correct) proof?
    \item[] Answer: \answerYes{} 
    \item[] Justification: We provide the complete assumptions and proofs in Appendix~\ref{sec:theory_proof}.
    \item[] Guidelines:
    \begin{itemize}
        \item The answer \answerNA{} means that the paper does not include theoretical results. 
        \item All the theorems, formulas, and proofs in the paper should be numbered and cross-referenced.
        \item All assumptions should be clearly stated or referenced in the statement of any theorems.
        \item The proofs can either appear in the main paper or the supplemental material, but if they appear in the supplemental material, the authors are encouraged to provide a short proof sketch to provide intuition. 
        \item Inversely, any informal proof provided in the core of the paper should be complemented by formal proofs provided in appendix or supplemental material.
        \item Theorems and Lemmas that the proof relies upon should be properly referenced. 
    \end{itemize}

    \item {\bf Experimental result reproducibility}
    \item[] Question: Does the paper fully disclose all the information needed to reproduce the main experimental results of the paper to the extent that it affects the main claims and/or conclusions of the paper (regardless of whether the code and data are provided or not)?
    \item[] Answer: \answerYes{} 
    \item[] Justification: The paper details the experimental setup (Section~\ref{sec:setup}, Appendix~\ref{app:exp_setup}), including datasets, baselines, hyperparameters, and hardware.
    \item[] Guidelines:
    \begin{itemize}
        \item The answer \answerNA{} means that the paper does not include experiments.
        \item If the paper includes experiments, a \answerNo{} answer to this question will not be perceived well by the reviewers: Making the paper reproducible is important, regardless of whether the code and data are provided or not.
        \item If the contribution is a dataset and\slash or model, the authors should describe the steps taken to make their results reproducible or verifiable. 
        \item Depending on the contribution, reproducibility can be accomplished in various ways. For example, if the contribution is a novel architecture, describing the architecture fully might suffice, or if the contribution is a specific model and empirical evaluation, it may be necessary to either make it possible for others to replicate the model with the same dataset, or provide access to the model. In general. releasing code and data is often one good way to accomplish this, but reproducibility can also be provided via detailed instructions for how to replicate the results, access to a hosted model (e.g., in the case of a large language model), releasing of a model checkpoint, or other means that are appropriate to the research performed.
        \item While NeurIPS does not require releasing code, the conference does require all submissions to provide some reasonable avenue for reproducibility, which may depend on the nature of the contribution. For example
        \begin{enumerate}
            \item If the contribution is primarily a new algorithm, the paper should make it clear how to reproduce that algorithm.
            \item If the contribution is primarily a new model architecture, the paper should describe the architecture clearly and fully.
            \item If the contribution is a new model (e.g., a large language model), then there should either be a way to access this model for reproducing the results or a way to reproduce the model (e.g., with an open-source dataset or instructions for how to construct the dataset).
            \item We recognize that reproducibility may be tricky in some cases, in which case authors are welcome to describe the particular way they provide for reproducibility. In the case of closed-source models, it may be that access to the model is limited in some way (e.g., to registered users), but it should be possible for other researchers to have some path to reproducing or verifying the results.
        \end{enumerate}
    \end{itemize}

\item {\bf Open access to data and code}
    \item[] Question: Does the paper provide open access to the data and code, with sufficient instructions to faithfully reproduce the main experimental results, as described in supplemental material?
    \item[] Answer: \answerYes{} 
    \item[] Justification: We provide the code in the supplementary material, and datasets used in this study are publicly available.
    \item[] Guidelines:
    \begin{itemize}
        \item The answer \answerNA{} means that paper does not include experiments requiring code.
        \item Please see the NeurIPS code and data submission guidelines (\url{https://neurips.cc/public/guides/CodeSubmissionPolicy}) for more details.
        \item While we encourage the release of code and data, we understand that this might not be possible, so \answerNo{} is an acceptable answer. Papers cannot be rejected simply for not including code, unless this is central to the contribution (e.g., for a new open-source benchmark).
        \item The instructions should contain the exact command and environment needed to run to reproduce the results. See the NeurIPS code and data submission guidelines (\url{https://neurips.cc/public/guides/CodeSubmissionPolicy}) for more details.
        \item The authors should provide instructions on data access and preparation, including how to access the raw data, preprocessed data, intermediate data, and generated data, etc.
        \item The authors should provide scripts to reproduce all experimental results for the new proposed method and baselines. If only a subset of experiments are reproducible, they should state which ones are omitted from the script and why.
        \item At submission time, to preserve anonymity, the authors should release anonymized versions (if applicable).
        \item Providing as much information as possible in supplemental material (appended to the paper) is recommended, but including URLs to data and code is permitted.
    \end{itemize}

\item {\bf Experimental setting/details}
    \item[] Question: Does the paper specify all the training and test details (e.g., data splits, hyperparameters, how they were chosen, type of optimizer) necessary to understand the results?
    \item[] Answer: \answerYes{} 
    \item[] Justification: The paper details the experimental setup (Section~\ref{sec:setup}, Appendix~\ref{app:exp_setup}), including datasets, baselines, hyperparameters, and hardware.
    \item[] Guidelines:
    \begin{itemize}
        \item The answer \answerNA{} means that the paper does not include experiments.
        \item The experimental setting should be presented in the core of the paper to a level of detail that is necessary to appreciate the results and make sense of them.
        \item The full details can be provided either with the code, in appendix, or as supplemental material.
    \end{itemize}

\item {\bf Experiment statistical significance}
    \item[] Question: Does the paper report error bars suitably and correctly defined or other appropriate information about the statistical significance of the experiments?
    \item[] Answer: \answerNo{} 
    \item[] Justification: Due to the substantial computational cost, error bars are rarely reported in LLM RLVR studies. Our work follows this common practice within the community.
    \item[] Guidelines:
    \begin{itemize}
        \item The answer \answerNA{} means that the paper does not include experiments.
        \item The authors should answer \answerYes{} if the results are accompanied by error bars, confidence intervals, or statistical significance tests, at least for the experiments that support the main claims of the paper.
        \item The factors of variability that the error bars are capturing should be clearly stated (for example, train/test split, initialization, random drawing of some parameter, or overall run with given experimental conditions).
        \item The method for calculating the error bars should be explained (closed form formula, call to a library function, bootstrap, etc.)
        \item The assumptions made should be given (e.g., Normally distributed errors).
        \item It should be clear whether the error bar is the standard deviation or the standard error of the mean.
        \item It is OK to report 1-sigma error bars, but one should state it. The authors should preferably report a 2-sigma error bar than state that they have a 96\% CI, if the hypothesis of Normality of errors is not verified.
        \item For asymmetric distributions, the authors should be careful not to show in tables or figures symmetric error bars that would yield results that are out of range (e.g., negative error rates).
        \item If error bars are reported in tables or plots, the authors should explain in the text how they were calculated and reference the corresponding figures or tables in the text.
    \end{itemize}

\item {\bf Experiments compute resources}
    \item[] Question: For each experiment, does the paper provide sufficient information on the computer resources (type of compute workers, memory, time of execution) needed to reproduce the experiments?
    \item[] Answer: \answerYes{} 
    \item[] Justification: We provide the hardware setup in Section~\ref{sec:setup} and compare the time efficiency in Appendix~\ref{app:runtime}.
    \item[] Guidelines:
    \begin{itemize}
        \item The answer \answerNA{} means that the paper does not include experiments.
        \item The paper should indicate the type of compute workers CPU or GPU, internal cluster, or cloud provider, including relevant memory and storage.
        \item The paper should provide the amount of compute required for each of the individual experimental runs as well as estimate the total compute. 
        \item The paper should disclose whether the full research project required more compute than the experiments reported in the paper (e.g., preliminary or failed experiments that didn't make it into the paper). 
    \end{itemize}
    
\item {\bf Code of ethics}
    \item[] Question: Does the research conducted in the paper conform, in every respect, with the NeurIPS Code of Ethics \url{https://neurips.cc/public/EthicsGuidelines}?
    \item[] Answer: \answerYes{} 
    \item[] Justification:  The research presented in this paper fully complies with the NeurIPS Code of Ethics. We have carefully considered issues such as reproducibility, fairness, transparency, potential societal impact, and the responsible use of data. All experiments were conducted ethically, and any datasets used were publicly available and appropriately cited. Code is provided in the supplementary material to support reproducibility.
    \item[] Guidelines:
    \begin{itemize}
        \item The answer \answerNA{} means that the authors have not reviewed the NeurIPS Code of Ethics.
        \item If the authors answer \answerNo, they should explain the special circumstances that require a deviation from the Code of Ethics.
        \item The authors should make sure to preserve anonymity (e.g., if there is a special consideration due to laws or regulations in their jurisdiction).
    \end{itemize}

\item {\bf Broader impacts}
    \item[] Question: Does the paper discuss both potential positive societal impacts and negative societal impacts of the work performed?
    \item[] Answer: \answerYes{} 
    \item[] Justification: We discuss broader impacts in Appendix~\ref{sec:impact}.
    \item[] Guidelines:
    \begin{itemize}
        \item The answer \answerNA{} means that there is no societal impact of the work performed.
        \item If the authors answer \answerNA{} or \answerNo, they should explain why their work has no societal impact or why the paper does not address societal impact.
        \item Examples of negative societal impacts include potential malicious or unintended uses (e.g., disinformation, generating fake profiles, surveillance), fairness considerations (e.g., deployment of technologies that could make decisions that unfairly impact specific groups), privacy considerations, and security considerations.
        \item The conference expects that many papers will be foundational research and not tied to particular applications, let alone deployments. However, if there is a direct path to any negative applications, the authors should point it out. For example, it is legitimate to point out that an improvement in the quality of generative models could be used to generate Deepfakes for disinformation. On the other hand, it is not needed to point out that a generic algorithm for optimizing neural networks could enable people to train models that generate Deepfakes faster.
        \item The authors should consider possible harms that could arise when the technology is being used as intended and functioning correctly, harms that could arise when the technology is being used as intended but gives incorrect results, and harms following from (intentional or unintentional) misuse of the technology.
        \item If there are negative societal impacts, the authors could also discuss possible mitigation strategies (e.g., gated release of models, providing defenses in addition to attacks, mechanisms for monitoring misuse, mechanisms to monitor how a system learns from feedback over time, improving the efficiency and accessibility of ML).
    \end{itemize}
    
\item {\bf Safeguards}
    \item[] Question: Does the paper describe safeguards that have been put in place for responsible release of data or models that have a high risk for misuse (e.g., pre-trained language models, image generators, or scraped datasets)?
    \item[] Answer: \answerNA{} 
    \item[] Justification: The paper does not describe explicit safeguards for responsible release, because the released resources (selected datasets) are low-risk compared to pre-trained language models or image generators. The datasets are intended for reproducible RLVR research and are documented with their construction process.
    \item[] Guidelines:
    \begin{itemize}
        \item The answer \answerNA{} means that the paper poses no such risks.
        \item Released models that have a high risk for misuse or dual-use should be released with necessary safeguards to allow for controlled use of the model, for example by requiring that users adhere to usage guidelines or restrictions to access the model or implementing safety filters. 
        \item Datasets that have been scraped from the Internet could pose safety risks. The authors should describe how they avoided releasing unsafe images.
        \item We recognize that providing effective safeguards is challenging, and many papers do not require this, but we encourage authors to take this into account and make a best faith effort.
    \end{itemize}

\item {\bf Licenses for existing assets}
    \item[] Question: Are the creators or original owners of assets (e.g., code, data, models), used in the paper, properly credited and are the license and terms of use explicitly mentioned and properly respected?
    \item[] Answer: \answerYes{} 
    \item[] Justification: All external datasets, codebases, and models used in our experiments are properly cited in the references. Licenses and terms of use have been respected.
    \item[] Guidelines:
    \begin{itemize}
        \item The answer \answerNA{} means that the paper does not use existing assets.
        \item The authors should cite the original paper that produced the code package or dataset.
        \item The authors should state which version of the asset is used and, if possible, include a URL.
        \item The name of the license (e.g., CC-BY 4.0) should be included for each asset.
        \item For scraped data from a particular source (e.g., website), the copyright and terms of service of that source should be provided.
        \item If assets are released, the license, copyright information, and terms of use in the package should be provided. For popular datasets, \url{paperswithcode.com/datasets} has curated licenses for some datasets. Their licensing guide can help determine the license of a dataset.
        \item For existing datasets that are re-packaged, both the original license and the license of the derived asset (if it has changed) should be provided.
        \item If this information is not available online, the authors are encouraged to reach out to the asset's creators.
    \end{itemize}

\item {\bf New assets}
    \item[] Question: Are new assets introduced in the paper well documented and is the documentation provided alongside the assets?
    \item[] Answer: \answerYes{} 
    \item[] Justification: We provide the source code and implementation details in the supplementary material for reproducibility. The code includes a README file with instructions for installation and environment setup.
    \item[] Guidelines:
    \begin{itemize}
        \item The answer \answerNA{} means that the paper does not release new assets.
        \item Researchers should communicate the details of the dataset\slash code\slash model as part of their submissions via structured templates. This includes details about training, license, limitations, etc. 
        \item The paper should discuss whether and how consent was obtained from people whose asset is used.
        \item At submission time, remember to anonymize your assets (if applicable). You can either create an anonymized URL or include an anonymized zip file.
    \end{itemize}

\item {\bf Crowdsourcing and research with human subjects}
    \item[] Question: For crowdsourcing experiments and research with human subjects, does the paper include the full text of instructions given to participants and screenshots, if applicable, as well as details about compensation (if any)? 
    \item[] Answer: \answerNA{} 
    \item[] Justification: No human subjects or crowdsourcing were involved.
    \item[] Guidelines:
    \begin{itemize}
        \item The answer \answerNA{} means that the paper does not involve crowdsourcing nor research with human subjects.
        \item Including this information in the supplemental material is fine, but if the main contribution of the paper involves human subjects, then as much detail as possible should be included in the main paper. 
        \item According to the NeurIPS Code of Ethics, workers involved in data collection, curation, or other labor should be paid at least the minimum wage in the country of the data collector. 
    \end{itemize}

\item {\bf Institutional review board (IRB) approvals or equivalent for research with human subjects}
    \item[] Question: Does the paper describe potential risks incurred by study participants, whether such risks were disclosed to the subjects, and whether Institutional Review Board (IRB) approvals (or an equivalent approval/review based on the requirements of your country or institution) were obtained?
    \item[] Answer: \answerNA{} 
    \item[] Justification: Not applicable, as no human subjects research was conducted.
    \item[] Guidelines:
    \begin{itemize}
        \item The answer \answerNA{} means that the paper does not involve crowdsourcing nor research with human subjects.
        \item Depending on the country in which research is conducted, IRB approval (or equivalent) may be required for any human subjects research. If you obtained IRB approval, you should clearly state this in the paper. 
        \item We recognize that the procedures for this may vary significantly between institutions and locations, and we expect authors to adhere to the NeurIPS Code of Ethics and the guidelines for their institution. 
        \item For initial submissions, do not include any information that would break anonymity (if applicable), such as the institution conducting the review.
    \end{itemize}

\item {\bf Declaration of LLM usage}
    \item[] Question: Does the paper describe the usage of LLMs if it is an important, original, or non-standard component of the core methods in this research? Note that if the LLM is used only for writing, editing, or formatting purposes and does \emph{not} impact the core methodology, scientific rigor, or originality of the research, declaration is not required.
    \item[] Answer: \answerNA{} 
    \item[] Justification: The core method development in this research does not involve LLMs as any important, original, or non-standard components.
    \item[] Guidelines:
    \begin{itemize}
        \item The answer \answerNA{} means that the core method development in this research does not involve LLMs as any important, original, or non-standard components.
        \item Please refer to our LLM policy in the NeurIPS handbook for what should or should not be described.
    \end{itemize}

\end{enumerate}